\documentclass[11pt,letterpaper,notitlepage]{report}

\usepackage[margin=1.2in]{geometry}
\usepackage{amsmath, amsfonts, amssymb}
\usepackage{amsthm}
\usepackage{hyperref}
\usepackage{graphicx}
\usepackage{cite}
\usepackage{longtable}
\usepackage{booktabs}
\usepackage{colortbl}
\usepackage{comment}
\usepackage{dsfont}
\usepackage[table]{xcolor}
\usepackage{thmtools}
\usepackage{thm-restate}
\usepackage{lscape}
\usepackage{cleveref}

\makeatletter
\renewenvironment{abstract}{
  \par\bigskip
  \noindent\textbf{\abstractname}\par
  \smallskip
}{\par\bigskip}
\makeatother

\newtheorem{theorem}{Theorem}[chapter]
\newtheorem{lemma}[theorem]{Lemma}
\newtheorem{proposition}[theorem]{Proposition}

\newtheorem{definition}[theorem]{Definition}
\newtheorem{example}[theorem]{Example}

\DeclareMathOperator*{\argmin}{arg\,min}
\DeclareMathOperator*{\argmax}{arg\,max}
\newcommand{\R}[1]{\mathbb{R}^{#1}}
\newcommand{\C}[1]{\mathbb{C}^{#1}}
\newcommand{\Cov}[1]{\text{Cov}\left\{#1\right\}}
\newcommand{\E}[2]{\mathbb{E}_{#1}\left\{#2\right\}}
\newcommand{\V}[2]{\mathbb{V}_{#1}\left\{#2\right\}}

\newcommand{\Es}[1]{\mathbb{E}_{#1}}
\newcommand{\x}{\boldsymbol{x}}
\newcommand{\z}{\boldsymbol{z}}
\newcommand{\ubold}{\boldsymbol{u}}
\newcommand{\A}{\boldsymbol{A}}

\newcommand{\T}{\boldsymbol{T}}
\newcommand{\m}{\boldsymbol{b}}

\newcommand{\bdim}[1]{\operatorname{boxdim}\left(#1\right)}
\newcommand{\zero}[1]{\boldsymbol{0}}
\newcommand{\G}[2]{\mathcal{N}\left(#1,#2\right)}
\newcommand{\Id}{\boldsymbol{I}}
\newcommand{\bSigma}{\boldsymbol{\Sigma}}
\newcommand{\bPsi}{\boldsymbol{\Psi}}

\newcommand{\btheta}{\boldsymbol{\theta}}

\newcommand{\score}{\nabla \log p_{\y} (\y)}
\newcommand{\F}{\boldsymbol{F}}
\newcommand{\yset}{\mathcal{Y}}
\newcommand{\Q}{\boldsymbol{Q}}

\newcommand{\trace}[1]{\text{trace}\left(#1\right)}
\newcommand{\rank}[1]{\text{rank}\left(#1\right)}
\newcommand{\diag}[1]{\text{diag}\left(#1\right)}
\newcommand{\y}{\boldsymbol{y}}
\newcommand{\boldeta}{\boldsymbol{\eta}}
\newcommand{\bomega}{\boldsymbol{\omega}}
\newcommand{\bepsilon}{\boldsymbol{\epsilon}}

\newcommand{\der}[2]{\frac{\partial {#1}}{\partial {#2}}}
\newcommand{\nder}[3]{\frac{\partial^{#3} {#1}}{\partial {#2}^{#3}}}
\newcommand{\loss}[1]{\mathcal{L}_{\textrm{#1}}(\y,f)}
\newcommand{\lossarg}[2]{\mathcal{L}_{\textrm{#1}}\left(#2, f\right)}
\newcommand{\signalset}{\mathcal{X}}
\newcommand{\const}{\text{const.}}

\newcommand{\JT}[1]{{\color{black}{#1}}}
\newcommand{\rev}[1]{{\color{black}{#1}}}

\crefname{equation}{}{}
\Crefname{equation}{}{}

\title{Self-Supervised Learning from Noisy and Incomplete Data} 

\author{
  Juli\'an Tachella \\
  CNRS, ENS Lyon \\
  \and
  Mike Davies \\
  University of Edinburgh \\
}

\begin{document}

\maketitle

\begin{abstract}
\noindent Many important problems in science and engineering involve inferring a signal from noisy and/or incomplete observations, where the observation process is known. Historically, this problem has been tackled using hand-crafted regularization (e.g., sparsity, total-variation) to obtain meaningful estimates. Recent data-driven methods often offer better solutions by directly learning a solver from examples of ground-truth signals and associated observations. However, in many real-world applications, obtaining ground-truth references for training is expensive or impossible. Self-supervised learning methods offer a promising alternative by learning a solver from measurement data alone, bypassing the need for ground-truth references. This manuscript provides a comprehensive summary of different self-supervised methods for inverse problems, with a special emphasis on their theoretical underpinnings, and presents practical applications in imaging inverse problems.
\end{abstract}

\setcounter{tocdepth}{1}
\tableofcontents
\clearpage

\chapter{Introduction to self-supervised learning for inverse problems}\label{chap: introduction}

Many important problems in science and engineering boil down to inferring a signal or image from noisy and/or incomplete observations, where the measurement process, often a physical system, is a priori known. For example, this includes the large range of applications in sensing and imaging inverse problems, from learning the structure of molecules using computational microscopy to astronomical imaging. In healthcare, medical imaging via computational tomography (CT), Magnetic resonance imaging (MRI), and ultrasound provides a crucial component of early diagnosis of disease. While applications in time series and audio include source separation, acoustic tomography and blind deconvolution.


While, historically, such inverse problems were solved through model-based approaches \cite{fessler_model_based_2010}, the powerful representation learning properties of deep neural networks 
have allowed researchers to develop new state-of-the-art data-driven reconstructions. Such solutions, trained on large quantities of ground truth data, are able to exploit the sophisticated statistical dependencies that previous hand-crafted models, such as sparse representations or total variation (TV) regularization~\cite{rudin_nonlinear_1992}, do not capture, and have substantially raised the bar on the achievable image reconstruction performance, e.g., in accelerated MRI image reconstruction \cite{zbontar_fastmri_2019}, showing a significant 6~dB gain in peak signal-to-noise ratio (PSNR) over TV regularization. 


Despite the phenomenal success of such solutions,
their reliance on large amounts of ground truth training data is a key limitation of the technology, restricting its application to problems where access to ground truth data is readily available - ones that have therefore essentially already been "solved" beforehand. This is particularly problematic in important scientific, medical and engineering settings, as well as for sensing systems working in complex environments, where ground truth data is scarce and where prediction accuracy is of overriding importance. This, in turn has led to a growing interest in the development of new \emph{self-supervised learning} solutions that aim to learn reconstructions without direct access to ground truth data. 

The goal of this monograph is to provide a self-contained presentation of such self-supervised learning techniques that have emerged within recent years and  highlighting the links to the underpinning statistical and geometric theory for such methods.

\section{Inverse problems}
The main focus of all such methods is the solution of a mathematical \emph{inverse problem} to estimate or reconstruct a signal or image of interest. While often these may in reality be defined as continuous functions, in order to compute a solution it is necessary to represent it in a discrete form, e.g., through an appropriate basis function expansion \cite{fessler_model_based_2010}. At the risk of committing an inverse crime \cite{Shimron_data_crimes_2022} we will focus in this manuscript on discrete signals, represented/approximated as a finite dimensional vector,
$\x \in \R{n}$ (or $\x \in \C{n}$) that can be estimated from measurements, $\y\in\R{m}$, through the stable \emph{inversion} of an acquisition process, also called the \emph{forward operator}, $\A:\R{n}\mapsto \R{m}$, that we assume to have already accommodated the discretization process:
\begin{equation}\label{eqs:forward_model}
    \y = \A(\x) + \bepsilon.
\end{equation}
Here $\bepsilon$ captures any noise or modelling errors and should be assumed to be possibly signal dependent, like the case of Poisson noise~\cite{luisier_pure_2010}. 

\subsection*{Examples}
We can illustrate the forward model in \Cref{eqs:forward_model} with a few idealized examples that we will use throughout this manuscript:
\begin{itemize}
    \item \textbf{Denoising} is the simplest inverse problem, where the forward operator is the identity mapping, that is $\A(\x)=\x$, and the goal is to remove the noise from the observed measurements.
    \item \textbf{Image inpainting} consists of recovering a set of missing pixels in an image, that is $\A(\x) = \diag{\m}\x$ with mask $\m \in \{0,1\}^{n}$.
    \item \textbf{Super resolution} is generally modelled as an inverse problem~\cite{hussein_correction_2020} with $\A(\x) = \boldsymbol{S}\, \text{circ}(\boldsymbol{k})\x$ where $\text{circ}(\boldsymbol{k}) \in \R{n\times n}$ is a convolution with a kernel $\boldsymbol{k} \in \R{n}$ and $\boldsymbol{S} \in \R{m \times n}$ is a subsampling operation.
    \item \textbf{Accelerated magnetic resonance imaging}  can be written as a linear inverse problem~\cite{zbontar_fastmri_2019}. In the single-coil setting, the acquisitions can be modelled as $\A(\x) =  \diag{\m}\F\x$ where $\F \in \C{n\times n}$ is the 2D discrete Fourier transform and $\m \in \{0,1\}^{n}$ is the acceleration mask.
    \item \textbf{Phase retrieval} is a non-linear inverse problem, which can be written as $\A(\x) = |\boldsymbol{B}\x|^2$ where $\boldsymbol{B}\in\C{m\times n}$ is a linear operator, which can be either random or structured according to the application~\cite{dong_phase_2023}.
    \item \textbf{Inverse scattering} is a complex non-linear inverse problem related to the Helmholtz equation, which can be written~\cite{soubies_efficient_2017} as $$\A(\x) = \text{circ}(\boldsymbol{g}) \diag{\x} (\Id - \text{circ}(\boldsymbol{g}) \diag{\x})^{-1} \boldsymbol{v}$$ where $\text{circ}(\boldsymbol{g})$ denotes a convolution with Green's kernel $\boldsymbol{g}\in\C{n}$ and $\boldsymbol{v}\in\C{n}$ is the incident field.
\end{itemize}

\section{From analytic reconstruction to machine learning}

\JT{Solving an inverse problem consists in devising a reconstruction function $f:\R{m}\mapsto\R{n}$ which removes the noise and inverts the effect of the forward operator, such that $f(\y) \approx \x$ approximately recovers the underlying signal.}
In many imaging and sensing scenarios, the forward model \eqref{eqs:forward_model} is linear, and early systems were explicitly designed to ensure that sufficient measurements were acquired in order that reconstruction through iterative or direct inversion $f(\y)=\A^{-1}\y$ could be used. However, as sensing and imaging problems became more challenging there was a need for more sophisticated reconstruction techniques.

\subsection{Why it is hard to invert?}
The key challenges in solving any ill-conditioned inverse problem are two-fold. First, the measurements acquired are generally not noise free. For example, low flux imaging results in observing only a limited number of photons at each measurement - something that is typically modeled statistically as Poisson noise. Part of the role of the inversion process is therefore to be able to infer clean signals from noisy observations.

The second major challenge is due to an inability to acquire a `complete' set of measurements. Sometimes this is a result of explicit undersampling, for example, in the accelerated MRI example above. In other scenarios the level of incompleteness is more subtle, such as in deconvolution problems where the forward operator might be full rank but severely ill-posed, or in super-resolution example, where the notion of undersampling is to a certain extent user defined. 

In either case, incomplete measurements means that there are insufficient measurements to simply directly invert the problem. \JT{For example, in the case of linear problems,} the forward operator may be rank deficient with a non-trivial null space resulting in an infinity of possible measurement consistent solutions, or the forward operator may be full rank, but severely ill-conditioned, meaning that there will be no general \emph{stable} inverse.

Geometrically, solving the problem of incomplete measurements requires, at least implicitly, the restriction of the signal model to a low dimensional set, as the image of a stable (i.e. Lipschitz) mapping, $f:\R{m} \mapsto \R{n}$, can at most have dimension $m$. This, for example, was the underpinning idea behind the compressed sensing revolution~\cite{candes_introduction_2008}, that popularized the notion of sparse signal models.

\subsection{Model based reconstruction}
Tackling inverse problems with noise and incomplete measurements historically used statistical techniques that combined a model-based consistency loss with the addition of some statistical constraint to capture the desired properties of the signals of interest. For example, this is often achieved by solving a regularized variational optimization problem composed of an $\ell_2$ consistency loss\footnote{In settings with non-Gaussian noise, the $\ell_2$ consistency is often replaced by the negative log-likelihood or other robust alternatives.},  $\|\y-\A(\x)\|^2$, and a regularization term, $\rho(\x)$, that captures the prior knowledge of the set of signals of interest:
\begin{equation}
    \JT{f(\y)} = \argmin_{\x} \|\y-\A(\x)\|^2 + \rho(\x).
\end{equation}

Many different regularizers have been used depending on the precise application, ranging from classical Tikhonov regularization to those that encourage sparse or low rank solutions, such as TV or nuclear norm regularization. However, such hand-crafted regularization can rarely capture all the sophisticated statistical dependencies within the problem leading researchers to explore the possibility of developing superior data-driven solutions. 

\section{Supervised learning}
The standard (supervised) way of learning inverse problem solvers from data consists of using a neural network, $f$, as the reconstruction function, $\hat{\x} = f_{\btheta}(\y)$ with weights, $\btheta \in \R{p}$, learned directly from training data that consist of pairs of ground truth signals and their associated measurements: $\{\x_i,\y_i\}_{i=1}^N$.
This is typically achieved by minimizing some supervised loss, 
\begin{equation} \label{eq: sup training}
    f^* = \argmin_{f} \sum_{i=1}^{N} \lossarg{SUP}{\x_i,\y_i}
\end{equation}
such as the $\ell_2$ loss
\begin{equation}\label{eq:suploss}
 \lossarg{SUP}{\x,\y} =  \frac{1}{n}\| f(\y) - \x \|^2 
\end{equation}
to give the learned solution, $f^*$, where we have dropped the explicit dependence on the weight vector, $\btheta$, and instead consider the optimization in the space of admissible functions. 

In principle, if the class of admissible neural network functions is sufficiently flexible \emph{and} we have sufficient training data, such an approach should allow us to approximate the optimal reconstruction function, which in the case of the $\ell_2$ loss is the conditional mean estimator: 
\begin{equation}
    f^*(\y) \approx \E{\x|\y}{\x}
\end{equation}
where the expectation\footnote{Throughout the manuscript we will use the notation $\E{\x|\y}{\phi(\x)}$ to denote the expectation of $\phi(\x)$ under $p(\x|\y)$.
} 
here is taken with respect to the posterior distribution $p(\x|\y)$.

\JT{The} learning approach transforms the problem into one of regression and potentially enables us to fully exploit the structure available within the training data. In practice, as we will briefly discuss in \Cref{sec:inductive bias}, the choice of the neural network architecture will also play an important role in the performance of the learned inverse mapping. 

\subsection{Commonly used network models}
Various different neural networks configurations for the inverse mapping, $f$, have been proposed for inverse problem solvers. Here, we focus on the two main classes of solutions that have been considered, noting that this will inevitably be incomplete in such a rapidly evolving field.

Most imaging solutions leverage an efficient low-level vision subnetwork structure that provides a image-to-image mapping, $g: \R{n} \mapsto \R{n}$, and is typically realized through either ResNet \cite{he_deep_2015} or UNet \cite{jin_deep_2017,ronneberger_unet_2015} style architectures with more recent incarnations incorporating attention mechanisms, e.g.~\cite{liang_swinir_2021}. This subnetwork is then used in various ways that differ primarily in how the acquisition model, $\A$, is incorporated into the overall solution. There are two broad approaches.
\begin{description}
   \item[Back-projection networks] A popular, simple and yet effective solution  
   is to first map the measurements back into the image/signal domain using a  back-projection operator. In the linear case, this can be done using the linear pseudo-inverse $\A^\dagger$ or some easily computable surrogate to this, such as $\A^\top$. The subnetwork, $\rev{\phi}(\ubold)$, is then used to map the backprojected signal to the clear reconstructed one. The full reconstruction function then takes the form $f(\y) = \rev{\phi}(\A^\dagger \y)$ and is trained in an end-to-end manner.
   \item[Unrolled architectures] These networks are motivated by attempting to mimic the structure of an iterative optimization algorithm unrolled for a small number of iterations, with the image-to-image mapping playing the role of a \emph{proximal} type operator. Probably the simplest such algorithm is the proximal gradient descent variant that takes the form:
   \begin{equation} \label{eq: unrolled}
       \x^{(k+1)} = \rev{\phi_k}\Big(\x^{(k)} - \tau  \JT{\der{\A}{\x}^{\top}}(\A(\x^{(k)}) - \y)\Big)
   \end{equation}
   where the weights in the subnetwork at each iteration, $k$, can be tied or trained independently. A range of different optimization algorithms have been unrolled in this manner, including primal-dual methods~\cite{adler_PDnet_2018} and gradient solvers for variational losses~\cite{Hammernik_varnet_2018}. Such networks tend to perform better than simple back-projection networks.
\end{description}
In each case, for best results, training tends to be performed in an end-to-end manner using~\Cref{eq: sup training}.


\subsection{Limitations of Supervised Learning}
While a supervised learning approach seems to offer the possibility of learning an approximation to the statistically optimal estimator, this is based on access to large quantities of ground truth data on which to train the model. This restricts its application to problems that have essentially been already solved previously (in order to generate the ground truth) and is particularly problematic in important scientific, medical and engineering settings, such as astronomical imaging or microscopy and for systems working in complex environments, where ground truth data is scarce and where prediction accuracy is of overriding importance.

One solution that is often adopted in the machine learning community to counter a lack of ground truth training data is to generate data from simulation. Although this provides access to potentially infinite quantities of data, such data is limited to the model from which the simulations are generated and even advanced simulations cannot fully capture the subtle complexities and dependencies that exist in the real setting. 

A related issue is the problem of distribution shift, where there is a change between the distribution of the training data and the measurements acquired at test time. For example, ground truth data may be available for a different but related set of signals or images that do not exactly represent the signals being targeted at test time.
Even when ground truth data is apparently available, such data is often generated through extended or repeated acquisitions, e.g., in MRI, or increased levels of illumination/radiation, such as in x-ray imaging or microscopy. This can significantly affect the nature of the imaging process and also result in a distribution shift between the acquired training data and the measurements acquired at test time. Unfortunately, supervised learning is notoriously poor at generalizing to such distribution shift~\cite{Recht_imagenet_classifier_generalization_2019}.

These challenges have led researchers to seek to develop new self-supervised learning methods that rely solely on the measurement data and knowledge of the acquisition process. Whether trained on just measurements from scratch or used to fine-tune existing models trained on simulations or related data~\cite{terris_reconstruct_2025}, such methods offer the potential in scientific imaging to learn to image structures and patterns for which no ground truth images yet exist~\cite{belthangady_applications_2019}.

\begin{figure}[t]
    \centering
    \includegraphics[width=1\textwidth,alt={An schematic illustrating supervised training on the left and self-supervised training on the right. In the supervised case, a dataset of paired noisy and clean cryo electron microscopy images are displayed, whereas in the self-supervised setting, the dataset consists of noisy images only.}]{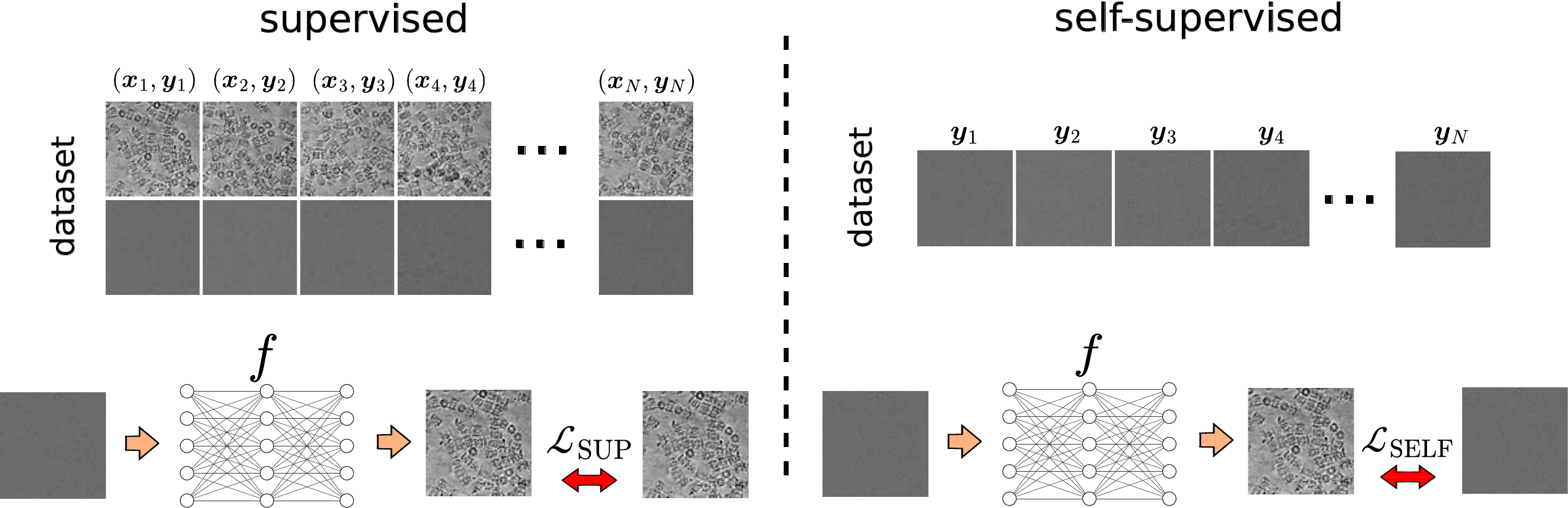}
    \caption{\JT{\textbf{Supervised and self-supervised learning.} Supervised learning requires a dataset of paired data $\{(\x_i,\y_i)\}_{i=1}$, whereas self-supervised learning, the main focus of this manuscript, relies on measurement data alone $\{\y_i\}_{i=1}$, and consists of constructing losses that do not require ground truth data, and can approximate the supervised loss.} }
    \label{fig: self-sup schematic}
\end{figure}

\section{Self-supervised learning}
The essential goal of self-supervised learning methods in imaging and sensing inverse problems is to replace a desired supervised loss function in \Cref{eq:suploss}, $\lossarg{SUP}{\x,\y}$, with an self-supervised loss, $\loss{}$, that is only a function of the measurement data, \JT{as illustrated in~\Cref{fig: self-sup schematic}.}

The general strategy is to develop a proxy that can be used in a self-supervised manner to either replicate or approximate the supervised loss. Here, we will see that the acquisition physics and noise model play an essential role in enabling one to formulate such an appropriate self-supervised loss. 
Examples of the growing body of work include applications to audio restoration~\cite{kashyap_speech_2021}, point cloud~\cite{casajus_total_2019} and image denoising~\cite{krull_noise2void_2019, krull_probabilistic_2020, laine_high-quality_2019, dalsasso_sar2sar_2021}, and image reconstruction~\cite{hendriksen_noise2inverse_2020, chen_equivariant_2021, yaman_self-supervised_2020}.

In many applications, we can expect to have many $n\gg 1$ samples/pixels, while the dominant statistical dependencies tend to be local. Hence the law of large numbers tells us that even when considering the loss for a single training sample $\lossarg{}{\y} \approx \E{\y}{\loss{}}$ and we can achieve relatively stable estimates of the expected losses with a modest number of samples $N$ (and can sometimes even get away with a single sample).
Thus, in most of the analyzes in this monograph,  we will assume that we have access to a sufficiently large dataset of measurements $\{\y_i\}_{i=1}^N$ such that we can replace sums over the dataset by expectations over the measurement distribution $p_{\y}$. The effects of having a finite training dataset are discussed in~\Cref{chap: sample complexity}.

Depending on the knowledge of the noise distribution and the range of forward operators giving rise to the measurements, we can obtain different levels of approximation of the supervised loss.

\paragraph{Unbiased losses}
The best we can hope for is to build a self-supervised loss that is an unbiased estimate of the supervised loss, i.e.,
\begin{equation}
   \E{\y}{\loss{}} = \E{\x,\y}{ \lossarg{SUP}{\x,\y}} + \const
\end{equation}
where the constant is generally a function of the variance of the noise.
Here we can expect to learn a reconstruction function that is as good as the one learned with supervised learning, as
long as we have enough measurement data. 

In some specific cases, we can have an even stronger result, where the loss is 
an unbiased estimate of the reconstruction error of a single instance, $\x$, that is
\begin{equation}
   \E{\y|\x}{\loss{}} = \E{\y|\x}{ \lossarg{SUP}{\x,\y}} + \const
\end{equation}
where the constant is known. This means that we can also use the self-supervised loss to quantify the 
reconstruction error of $\x$ at test time. This is the case, for example, with Stein's Unbiased Risk Estimator and its variants which will be discussed in~\Cref{chap: noisy}.

\paragraph{Constrained losses}
In some cases, we will not be able to build a loss that is unbiased over the whole space of  possible reconstruction functions, but can obtain unbiased estimates over a constrained set of reconstruction functions, $\mathcal{F}$:
\begin{equation}
   \E{\y}{\loss{}} = \E{\x,\y}{ \lossarg{SUP}{\x,\y}} + \const \; \text{ for } f\in \mathcal{F}
\end{equation}
The choice of the constraint set can be motivated by either a restricted function class designed to make the problem learnable, e.g.,~\cite{krull_noise2void_2019,tachella_unsure_2025}, or to incorporate additional prior information within the model, such as imposing an equivariance constraint~\cite{chen_equivariant_2021}. 
If the optimal supervised reconstruction function does not belong to the constrained set $\mathcal{F}$, the learned reconstruction function will inevitably not be optimal and will not match the performance achievable through supervised learning. However, in some cases we will be able to quantify the bias introduced by the constraint and therefore control the performance gap.

\paragraph{Losses sharing global minimum}
A final case is when the self-supervised loss is not an unbiased estimate of the supervised loss,
but the two losses share a common global minimum, i.e.,
\begin{equation}\label{eq: same minimum}
   \argmin_{f} \, \E{\y}{\loss{}} = \argmin_{f} \, \E{\x,\y}{\lossarg{SUP}{\x,\y}}
\end{equation}
Thus, we can expect to learn a reconstruction function that is close to the one learned with supervised learning, but may not be able to quantify the associated reconstruction error.

Throughout this survey we will mainly focus on proxies for the $\ell_2$ supervised loss \eqref{eq:suploss} as this is where the theory is most well developed. However, along the way we will highlight where the theory extends beyond $\ell_2$ and/or where practitioners have applied similar techniques using other loss functions in a more heuristic manner. 

\subsection{Learning a generative model}



Going beyond self-supervised proxies for the supervised loss, we can ask whether it is possible to learn a generative model for the full signal distribution $p_{\x}(\x)$ or the posterior distribution $p(\x|\y)$ from measurement data alone. 
If we were able to learn such a generative model, we would be able to compute not only the conditional mean $f^{*}(\x)=\E{\x|\y}{\x}$ but also any other posterior statistic.
We will see that this is indeed possible in some scenarios, and it often also relies on training on a self-supervised loss that approximates the $\ell_2$ supervised loss (e.g., diffusion models require learning a conditional mean estimator).
However, it is worth noting that learning a generative model is typically a harder task, and in some cases, learning such a model can be impossible even when constructing a self-supervised loss that approximates the supervised one is still possible~\cite{raphan_least_2011} (see also Appendix~\ref{app: moments} for some examples). The question of whether it is possible (or not) to identify the signal distribution from measurement data alone is discussed in~\Cref{sec: model identification}.

\section{What this manuscript is not about}
\label{sec:inductive bias}

\paragraph{Self-supervised representation learning}
\rev{
It is important to draw a distinction at this point between the notion of self-supervised learning in imaging and sensing covered in this manuscript, and self-supervised representation learning (SSRL) techniques~\cite{ericsson_self-supervised_2022}, such as SimCLR~\cite{chen_simple_2020}, BYOL~\cite{grill_bootstrap_2020}, DINO~\cite{oquab2023dinov2}, or masked autoencoders~\cite{he_masked_2022}, that learn powerful representations by training on a set of pretext tasks.

SSRL methods typically require a dataset of clean data
$\{\x_i\}_{i=1}^N$, and aim to learn powerful high-level representations for downstream tasks such as classification or segmentation.
On the contrary, the self-supervised methods presented here rely on noisy and/or incomplete data alone $\{\y_i\}_{i=1}^N$ and aim to recover the underlying clean images associated to these measurements.
Moreover, self-supervised losses in this manuscript serve as proxies for the gold standard supervised reconstruction loss in~\Cref{eq: sup training}, whereas pretext tasks used in SSRL do not aim at approximating a supervised classification or segmentation loss. 

Despite these differences, some of the fundamental principles behind the design of pretext tasks, such as invariance to transformations or masking, are also pillars of the self-supervised losses used for imaging inverse problems, and a better understanding of the connections between these two fields remains an open research problem.
}



\paragraph{Deep image prior and inductive bias}
While we will generally focus on the behaviour of the expected loss, in practice, there will only be a finite amount of training data and thus the inductive bias of the learning system will also play an important role on actual observed performance.

There are various sources of inductive bias in neural network systems, from the choice of model architecture and weight initialization, to the inclusion of regularization terms in the loss function, e.g., weight decay, and even the optimization procedure, e.g., Adam versus stochastic gradient descent, or the use of early stopping.

For example, Ulyanov et al.~\cite{ulyanov_deep_2018} showed that various convolutional neural network architectures could be trained to solve inverse problems from a single set of measurements (the one being restored). Something they called the deep image prior (DIP). This has motivated many researchers \cite{heckel_deep_2019,darestani_accelerated_2021,ren_neural_2020} to try to exploit this concept for unsupervised image reconstruction. However, the nature of the inductive bias is poorly understood~\cite{tachella_neural_2020}, and while, as demonstrated in the original DIP paper, the performance is highly dependent on the specific network architecture
and is generally well below that obtained by self-supervised methods covered in this manuscript, e.g., see practical comparisons in~\cite{lehtinen_noise2noise_2018}.
Thus, although the DIP is certainly an intriguing phenomenon, we do not consider it further here.

\paragraph{Pretrained diffusion and plug-and-play models} 
\rev{Denoising diffusion models~\cite{daras_survey_2024} and plug-and-play (PnP) solutions~\cite{kamilov_plug-and-play_2023} have become popular for solving inverse problems. Such methods typically rely on \emph{pre-trained} denoising neural networks,
where the denoisers are used to define an implicit signal prior through the score function and Tweedie's formula. While these solutions are often termed unsupervised, this is not wholly accurate as the creation of the pre-trained denoisers requires access to ground truth data. 
Nonetheless, there have been recent efforts to learn the denoiser in a self-supervised way~\cite{rozet_learning_2024,daras_consistent_2024}, which we will also cover in this manuscript.
}



\section{Outline}
\Cref{notation} sets out the notation that is commonly used throughout the survey. 
The outline of the rest of the survey is set out below. Most of the self-supervised methods discussed in this monograph are implemented in the DeepInverse open-source library~\cite{tachella2025deepinverse}, which contains \href{https://deepinv.github.io/deepinv/auto_examples/self-supervised-learning/index.html}{various jupyter notebook examples}, covering many of the topics in this manuscript.

\paragraph{}\textbf{Chapter 2} focuses on the problem of self-supervised learning for denoising with the forward operator, $\A = \Id$. We consider various self-supervised losses that act as proxies for the supervised loss under a range of different noise models, from presumed knowledge of various well known noise distributions (Gaussian, Poisson, etc.) to partially specified noise models. Throughout, links between newly proposed self-supervised learning strategies and theoretical results from classical statistics are highlighted. We end by considering the case of more general but invertible forward operators.

\paragraph{}\textbf{Chapter 3} goes on to consider what can be done when we have incomplete measurements, i.e., the forward operator is not invertible. Here, we focus on the case of linear forward operators 
where there exists a non-trivial null space. We describe two approaches to solving this problem. The first relies on access to a set of multiple forward operators $\{\A_g \}_{g=1}^G$, such as the case of being able to select different sampling patterns in accelerated MRI~\cite{yaman_self-supervised_2020}, where the different operators typically have distinct nullspaces. The second approach tackles the more challenging problem of a single rank-deficient forward operator, $\A$, and instead leverages the assumption that the distribution of signals of interest is invariant to a group of transformations, e.g., a shifted version of an image is still a viable image. 


\paragraph{}\textbf{Chapter 4} While the previous two chapters concentrate of the roles of the expected loss functions in enabling self-supervised learning solutions, this chapter considers 
how accurately these expectations can be approximated when there is only access to finite number of training samples. We consider the simple Noise2Noise algorithm to explore how sample complexity for self-supervised learning behaves in relation to the supervised learning case. We also show how the standard holdout method used in most supervised learning to avoid overfitting can be extended to the self-supervised learning setting and how pretrained models can be used to reduce the number of measurement samples required for good performance.

\paragraph{}\textbf{Chapter 5} sets out some open problems within the field and possible future research directions.


\subsection{Notation} \label{notation}

Following standard mathematical notation, vectors will be represented by bold lowercase letters and matrices will be represented in bold uppercase. The $i$th component of a vector $\x$ is written as $x_i$.
The identity matrix is written as $\Id$, the transpose of a matrix, $\A$, is denoted by $\A^\top$ and its pseudo-inverse is written as $\A^\dagger$. 
Other notation that is regularly used throughout the manuscript can be found in the table below. 

\begin{longtable}{c @{ } l}


\caption{Notation}\\
\endfirsthead

\multicolumn{1}{l}{$\quad$\textbf{Symbol}} & \multicolumn{1}{l}{\textbf{Description}} \\
\hline
\endfirsthead

\multicolumn{2}{c}%
{{\tablename\ \thetable{} -- continued from previous page}} \\
\multicolumn{1}{c}{\textbf{Symbol}} & \multicolumn{1}{l}{\textbf{Description}} \\
\hline
\endhead

\hline 
\multicolumn{2}{r}{{Continued on next page}} \\
\endfoot

\hline
\endlastfoot

$\mathbb{R}$        & Set of real numbers. \\
$\mathbb{C}$        & Set of complex numbers. \\
$p_{\x}(\x)$        & Signal distribution. \\
$p_{\y}(\y)$        & Measurement distribution. \\
$\signalset$        & Support of the signal distribution. \\
$\mathcal{Y}$        & Support of the measurement distribution. \\
$\x$                & Vector representing the ground truth signal or image. \\
$\y$                & Vector representing the observed measurements. \\
$\m$ & Binary vector representing a mask. \\
$n$                 & Dimension of the signal vector, $\x \in \R{n}$.\\
$m$                 & Dimension of the measurement vector, $\y \in \R{m}$.\\
$k$                 & Dimension of the signal set, $\signalset$.\\
$N$                 & Number of training samples.\\
$\A$                & Forward operator, $\A: \R{n} \rightarrow \R{m}$ where $\y = \A(\x)$. \\
$f$                 & Reconstruction network mapping $\y$ to an estimate of $\x$.\\
$\btheta$           & Weights of a neural network, $f$.\\
$\|\cdot\|$         & $\ell_2$ norm. \\
$\|\cdot\|_F$       & Frobenius norm of a matrix.\\
$\E{\boldsymbol{u}}{g(\boldsymbol{u})}$  & Expectation of $g(\boldsymbol{u})$ under the distribution $p(\boldsymbol{u})$.\\
$\E{\boldsymbol{u}|\boldsymbol{v}}{g(\boldsymbol{u},\boldsymbol{v})}$         & Expectation of $g(\boldsymbol{u},\boldsymbol{v})$ under the distribution $p(\boldsymbol{u}|\boldsymbol{v})$. \\
$\V{\boldsymbol{u}|\boldsymbol{v}}{\boldsymbol{u}}$         & Variance of $\boldsymbol{u}$ under the distribution $p(\boldsymbol{u}|\boldsymbol{v})$.\\
$\mathcal{L}_{\text{SUP}}(\x,\y,f)$       & Supervised loss.\\
$\mathcal{L}_{\text{X}}(\y,f)$       & Self-supervised loss associated with technique $\text{X}$.\\
$\nabla$            & Gradient of a scalar field.\\ 
$\const$            & Constant term that is not further quantified.\\
$\G{\x}{\bSigma}$ & Multivariate Gaussian with mean $\x$ and covariance $\bSigma$. \\
$\mathcal{P}(\x)$ & Poisson distribution with rate $\x$.\\
$\text{Ber}(\x)$ & Bernouilli distribution with probability $\x \in [0,1]^{n}$.
\end{longtable}
\chapter{Learning from noisy measurements}\label{chap: noisy}


We start by focusing on denoising problems, where the forward operator is simply the identity mapping $\A(\x)=\x$, and thus both images and measurements lie in the same space.
We present various self-supervised losses that only require measurement data and aim at approximating the supervised loss in expectation. We show that the design of the loss is dependent on the knowledge about the noise distribution:
if we fully know the noise distribution, we are generally able to build unbiased estimators of the supervised loss, whereas when the noise distribution is not fully known, we can still build self-supervised losses, but they do not achieve the same performance as supervised learning.



The chapter is divided in three parts: in the first part we assume that we observe two independent noisy realizations $(\y_1,\y_2)$ per image $\x$. In the second part, we will relax this assumption, only relying on a single noisy realization $\y$ per image, but instead assume full knowledge about the noise distribution. In the third part, we will tackle the case where we observe a single noisy realization $\y$ per image and the noise distribution is partially unknown.
Most of the results presented here are independent of the architecture or parameterization of the reconstruction network $f$.



\section{Learning from independent noisy pairs} \label{sec: noise2noise}

In some applications, it is possible to observe two (or more) independent noisy realizations $(\y_1,\y_2)$ of the same underlying signal $\x$. These can then be used to learn an estimator in a self-supervised way even without explicit knowledge of the noise distribution~\cite{mallows_comments_1973}. Noise2Noise~\cite{lehtinen_noise2noise_2018} proposed such an approach\footnote{\JT{The idea of using independent observations of the same underlying parameter for model selection can be traced back to Mallows work in the 1970s~\cite{mallows_comments_1973}. This idea has been rediscovered in the computer vision field by Noise2Noise~\cite{lehtinen_noise2noise_2018}.}} using one of the noisy measurements as input to the reconstruction network, and the other as target, building the following loss:
\begin{equation} \label{eq: noise2noise} \tag{Noise2Noise}
    \lossarg{N2N}{\y_1,\y_2} = \frac{1}{n}\|f(\y_1)-\y_2\|^2.
\end{equation}
Since the input noise is independent of the output noise, we can show
that \Cref{eq: noise2noise} is an unbiased estimator of the supervised loss up to a constant:
\begin{proposition} \label{prop: noise2noise}
Let $\y_1$ and $\y_2$ be two random variables independent conditional on $\x$, and assume that $\E{\y_2|\x}{\y_2}=\x$, then
\begin{equation}
    \E{\y_1,\y_2|\x}{\frac{1}{n}\|f(\y_1)-\y_2\|^2} = \E{\y_1|\x}{\frac{1}{n}\|f(\y_1) - \x\|^2} + \const
\end{equation}
where the constant is independent of $f$.
\end{proposition} 
\begin{proof}
\begin{align*} 
&\E{\y_1,\y_2|\x}{\|f(\y_1)-\y_2\|^2} \\ 
&=\E{\y_1,\y_2|\x}{\|(f(\y_1)-\x)-(\y_2-\x)\|^2} \\ &= \E{\y_1|\x}{\|f(\y_1) - \x\|^2} - 2 \, \E{\y_1,\y_2|\x} {(f(\y_1)-\x)^{\top} (\y_2 -\x)} +  \const\\
 &= \E{\y_1|\x}{\|f(\y_1) - \x\|^2} - 2\left(\E{\y_1|\x}{f(\y_1)-\x}\right)^{\top}  (\E{\y_2|\x}{\y_2 -\x}) + \const \\
  &= \E{\y_1|\x}{\|f(\y_1) - \x\|^2} + \const
\end{align*}
where the fourth line uses the fact that $\y_1$ and $\y_2$ are conditionally independent and the last line relies uses $\E{\y_2|\x}{\y_2 -\x}=\boldsymbol{0}$.
\end{proof}
This result can be extended to any Bregman divergence beyond the $\ell_2$ norm~\cite{efron_estimation_2004}, but it does not hold for some other popular losses such as the $\ell_1$ norm. 
Intuitively, the estimator $f$ cannot overfit the noise in $\y_2$ as it observes an independent noise realization $\y_1$.
The result in~\Cref{prop: noise2noise} requires minimal assumptions on the noise (only that the target has zero-mean noise, i.e., $\E{\y_2|\x}{\y_2}=\x$), making it very appealing for real-world problems where the noise distribution is not known and is possible to obtain two independent observations of the same object.
We illustrate this with some imaging examples: 
\begin{itemize}
    \item In cryo-electron microscopy, we observe a series of very noisy images (micrographs) of the same underlying object. Bepler et al.~\cite{bepler_topaz-denoise_2020} show that we can drastically boost the SNR using a Noise2Noise approach. 
    \item In synthetic aperture radar (SAR), we observe complex images following a circularly-symmetric complex normal distribution, where the real and imaginary parts have independent noise. Dalsasso et al.~\cite{dalsasso_as_2022} show that it is possible to train a denoiser with real part as input and imaginary as target.
    \item In video denoising, the similarity between consecutive frames almost meets the Noise2Noise criterion. Ehret et al.~\cite{ehret_frame_to_frame_2019} show that a pretrained video denoising network can be fine-tuned using the Noise2Noise approach in combination with optical flow estimates to warp one frame onto another. 
\end{itemize}
The assumption of observing two independent measurements is not met in many applications. Nonetheless, we will see in the following section that (perhaps surprisingly!), if the noise distribution is known, we  can often obtain two independent noise realizations $(\y_1,\y_2)$ from a single measurement $\y$ without knowledge of the underlying image $\x$, and apply the same Noise2Noise loss using these independent pairs. 

\section{Known noise distribution}

In many applications, the noise distribution is approximately known, or it can be approximated using some calibration data.
There are two main approaches for building self-supervised losses that incorporate this knowledge: 
the first approach was pioneered by Noisier2Noise~\cite{moran_noisier2noise_2020} \JT{and Recorrupted2Recorrupted~\cite{pang_recorrupted--recorrupted_2021}, who showed} that it is possible to add synthetic noise to the observation $\y$ to generate two independent realizations $(\y_1,\y_2)$. A second approach is based on a classical result in statistics known as Stein's Unbiased Risk Estimate (SURE)~\cite{stein_estimation_1981}, which penalizes the divergence of the network $f$ to avoid overfitting the noise. 
In both cases, we require exact knowledge of the noise distribution in order to correctly approximate the supervised case. We will further see that, despite at first sight looking quite different, \JT{Recorrupted2Recorrupted} and SURE are closely related.

\subsection{Bootstrapping noisy measurements}

\noindent While the Noisier2Noise framework~\cite{moran_noisier2noise_2020} set out the original approach to bootstrapping noisy measurements, we will follow the equivalent\footnote{Noisier2Noise~\cite{moran_noisier2noise_2020} introduced the idea of adding noise to the inputs previous to Recorrupted2Recorrupted~\cite{pang_recorrupted--recorrupted_2021}, but the latter presented a simplified loss, showing conditional independence of the simulated pairs $(\y_1,\y_2)$. See Appendix~\ref{app: noisier2noise} for more details regarding the close links between these approaches.}  Recorrupted2Recorrupted~\cite{pang_recorrupted--recorrupted_2021} work as this sets the scene for further generalizations. 

Assuming a Gaussian noise model, $\y = \x + \bepsilon$ with $\bepsilon\sim \G{\boldsymbol{0}}{\bSigma}$, or equivalently that $\y \sim \G{\x}{\bSigma}$, we can resample two independent noisy realizations from the original measurement, $\y$, as
\begin{equation} \label{eq: gaussian r2r}
\left\{
\begin{array}{l}
    \y_1 = \y +  \, \sqrt{\frac{\alpha}{1-\alpha}} \bomega \\ 
    \y_2 = \y - \sqrt{\frac{1-\alpha}{\alpha}} \, \bomega
\end{array}
\right.
\end{equation}
where $\bomega\sim\G{\boldsymbol{0}}{\bSigma}$ follows the same distribution as the noise $\bepsilon$ and
$\alpha \in (0,1)$ is a positive scalar parameter. 

\JT{

\begin{proposition}[Pang et al.~\cite{pang_recorrupted--recorrupted_2021}]
The random variables $\y_1$ and $\y_2$ defined by~\Cref{eq: gaussian r2r} are independent conditional on $\x$ for any $\alpha\in(0,1)$.
\end{proposition}
\begin{proof}
Let $\tau = \sqrt{\frac{\alpha}{1-\alpha}}$. Since $\y_1$ and $\y_2$ follow a Gaussian distribution conditional on $\x$, we can prove their independence by simply showing that they are not linearly correlated:
\begin{align*}
&\E{\y_1,\y_2|\x}{(\y_1 - \x)(\y_2 - \x)^{\top}} \\ &= \E{\bepsilon,\bomega}{(\bepsilon + \tau\bomega)(\bepsilon - \frac{1}{\tau}\bomega)^{\top}} \\
&=  \E{\bepsilon}{\bepsilon\bepsilon^{\top}} -\frac{1}{\tau}  \E{\bepsilon}{\bepsilon} \E{\bomega}{\bomega^{\top}} +  \tau\E{\bomega}{\bomega}\E{\bepsilon}{\bepsilon^{\top}} -   \E{\bomega}{\bomega\bomega^{\top}}\\
&= \E{\bepsilon}{\bepsilon\bepsilon^{\top}} -   \E{\bomega}{\bomega\bomega^{\top}} \\ &= \boldsymbol{0}
\end{align*}
The last line relies on the assumption that the added noise $\bomega$ has the same covariance as the measurement noise to achieve independence.
\end{proof}

}

Following the Noise2Noise approach, we can define the Recorrupted2Recorrupted loss
as
\begin{equation} \label{eq: R2R} \tag{R2R}
    \loss{R2R} = \E{\y_1,\y_2|\y}{\frac{1}{n}\|f(\y_1)-\y_2\|^2}
\end{equation}
which, due to the conditional independence of $(\y_1,\y_2)$ and $\E{\y_2|\x}{\y_2}=\x$, is an unbiased estimate of the supervised $\ell_2$ loss with $\y_1$ at the input of the network:
\JT{\begin{equation} \label{eq: unbiased r2r}
    \E{\y|\x}{\loss{R2R}} = \E{\y_1|\x}{\frac{1}{n}\|f(\y_1)-\x\|^2} + \const
\end{equation}
Note} that \Cref{eq: R2R} is an idealized loss, as it involves the expectation over the resampled realizations. However, in practice we can use a single resampled pair, $(\y_1,\y_2)$, per gradient step and the resulting stochastic gradient estimates of the loss will remain unbiased. 

The independence of $\y_1$ and $\y_2$ also comes at a price: the input to the network, $\y_1$, has lower signal-to-noise ratio (SNR) than the original measurement, $\y$, due to the additional synthetic noise. 
The parameter $\alpha \in (0,1)$ acts as a trade-off between the amount of additional noise injected into $\y_1$ and $\y_2$.  By defining the SNR of a random variable $\boldsymbol{z}$ as
$
\text{SNR}(\boldsymbol{z}) := \frac{\| \E{\boldsymbol{z}}{\boldsymbol{z}} \|^2}{\E{\boldsymbol{z}}{\|\boldsymbol{z}-\E{\boldsymbol{z}}{\boldsymbol{z}}\|^2}} 
$
we have that the SNR of the new variables is
\begin{align} \label{eq: SNR y1y2}
\text{SNR}(\y_1) &= (1-\alpha) \, \text{SNR}(\y), \\
\text{SNR}(\y_2) &= \alpha \, \text{SNR}(\y). \notag
\end{align}
Thus, a good strategy is to choose $\alpha$ small, to have as much SNR on the input $\y_1$ as possible, but not too small, to avoid targets $\y_2$ with very low SNR, which would result in noisier loss estimates. In practice we only use a finite number of resamplings of pairs $(\y_1,\y_2)$ for a fixed $\y$. A good choice in practice generally lies in the interval $\alpha\in(0.05, 0.2)$.
For more information about the optimal choice of $\alpha$, see the analysis by Oliveira et al.~\cite{oliveira_unbiased_2022}. 

At test time, we can improve the estimation by averaging over $J>1$ additions of synthetic noise, that is  
\begin{equation}
    f^{\text{test}}(\y) = \frac{1}{J}\sum_{j=1}^J f(\y_1^{(j)})
\end{equation}
where $\y_1^{(j)}\sim \G{\y}{\frac{\alpha}{1-\alpha}\, \bSigma}$ for $j=1,\dots,J$. 


\JT{This loss can be extended to non-Gaussian additive noise: if some first order moments of the noise distribution are known,} we can still use \Cref{eq: gaussian r2r} to generate pairs $(\y_1,\y_2)$ by adding synthetic noise $\bomega$ that matches these moments~\cite{monroy_generalized_2025}.

\paragraph{Extensions beyond additive noise}
In many applications, the noise affecting the measurements is not additive. For example, Poisson noise arises in photon-counting devices such as single-photon lidar~\cite{rapp_advances_2020}, and the Gamma distribution is often used to model speckle noise associated with synthetic aperture radar images~\cite{dalsasso_as_2022}. 

Gaussian, Poisson and Gamma distributions belong to a natural exponential family (NEF)~\cite{efron2022exponential}, and can be written as
\begin{equation} \label{eq: nef} \tag{NEF}
    p(\y|\x)= h(\y) \exp( \y^{\top} \eta(\x) - \phi(\x)),
\end{equation}
for some $h$, $\eta$ and $\phi$ functions which are specific to each distribution (note this forms a NEF with respect to the \emph{natural parameter}, $\eta(\x)$ and not with respect to the image, $\x$, itself).
We can generalize the \Cref{eq: R2R} loss to the NEF using the following theorem:
\begin{theorem}[Monroy et al.~\cite{monroy_generalized_2025}] 
\label{theo: NEF}
    Let $p(\y|\x)$ belong to the NEF with $\E{\y|\x}{\y}=\x$ and $\alpha\in(0,1)$. We can sample $\y_1$ and $\y_2$ as
    \begin{equation} \label{eq: sampling gr2r}
    \left\{
    \begin{array}{l}
        \y_1 \sim  \; p(\y_1|\y,\alpha), \\ 
        \y_2 =  \frac{1}{\alpha} \y -  \frac{(1-\alpha)}{\alpha}\y_1,
    \end{array}
    \right.
    \end{equation}
     such that
   $\y_1$ and $\y_2$ are independent random variables conditional on $\x$, with means $\E{\y_1|\x }{\y_1}= \E{\y_2|\x}{\y_2} = \x$ and variances $\V{\y_1 | \x}{\y_1}= \frac{1}{1-\alpha}\V{\y | \x}{\y} $ and $\V{\y_2 | \x}{\y_2}= \frac{1}{\alpha}\V{\y | \x}{\y}$, and their distributions $p_1(\y_1|\x)$ and $p_2(\y_2|\x)$ also belong to the NEF.
\end{theorem}
The generalized R2R loss is thus defined as
\begin{equation}\label{eq: GR2R} \tag{GR2R}
    \loss{GR2R} = \E{\y_1,\y_2|\y}{\frac{1}{n}\|f(\y_1)-\y_2\|^2} 
\end{equation}
where $\y_1$ and $\y_2$ are generated via~\Cref{eq: sampling gr2r}.
Since the synthetic pairs are independent conditional on $\x$, we can use again~\Cref{prop: noise2noise} to show that
\begin{equation*}
\E{\y|\x}{\loss{GR2R}} = \E{\y_1|\x}{\| f(\y_1) -\x \|^2} + \const
\end{equation*}
The definition\footnote{ In general, we have that
   $p(\y_1|\y,\alpha)=h_1(\y_1)h_2(\y-\y_1)/h(\y)$ where 
    \begin{align*}
        h_1(\y_1)&=\int e^{-\boldsymbol{s}^{\top}\y_1+(1-\alpha)\phi(\eta^{-1}(\frac{\boldsymbol{s}}{1-\alpha}))} d\boldsymbol{s} \\
        h_2(\y_2)&=\int e^{-\boldsymbol{s}^{\top}\y_2+\alpha\phi(\eta^{-1}(\frac{\boldsymbol{s}}{\alpha})) } d\boldsymbol{s}.
    \end{align*}} of $p(\y_1| \y, \alpha)$ for the Gaussian, Poisson and Gamma noise distributions is included in~\Cref{tab:NEF}. \JT{As in the Gaussian noise case, the SNR of $\y_1$ and $\y_2$ is given by \Cref{eq: SNR y1y2}, and $\alpha$ should be chosen to approximately lie in the $(0.05, 0.2)$ interval~\cite{monroy_generalized_2025}.}

The idea of generating synthetic pairs of independent data has also been recently explored in other statistical inference applications where it is known as \emph{data fission}~\cite{leiner2025data} and includes extensions beyond the NEF.

\begin{table*}[!t]
\footnotesize 
\centering
\resizebox{\linewidth}{!} {
\setlength\tabcolsep{0.009cm}
\begin{tabular}{c|c|c|c}
\toprule
\cellcolor{green!10}\textbf{Model}  & \cellcolor{green!10} \textbf{Gaussian} & \cellcolor{green!10} \textbf{Poisson} & \cellcolor{green!10} \textbf{Gamma} \\ 
\cellcolor{green!10}&\cellcolor{green!10} $\y \sim \mathcal{N}(\x,\bSigma)$ 
& \cellcolor{green!10} $ \boldsymbol{z}\sim \mathcal{P}(\x/\gamma) , \y = \gamma \boldsymbol{z}$ & \cellcolor{green!10} $\y\sim \mathcal{G}(\ell, \ell / \x) $
\\ \toprule
$\y_1$ & $ \y_1=\y + \sqrt{\frac{\alpha}{1-\alpha}} \boldsymbol{\omega}$, & $\y_1=\frac{\boldsymbol{y}- \gamma \boldsymbol{\omega}}{1-\alpha}  $, &  $\y_1=\y\circ(\boldsymbol{1}-\boldsymbol{\omega}) /(1-\alpha)$ \\ 
&  $\boldsymbol{\omega} \sim  \mathcal{N}(0,\bSigma)$ & $\boldsymbol{\omega} \sim \text{Bin}(\boldsymbol{z}, \alpha)$  & $\boldsymbol{\omega} \sim  \text{Beta} (\ell \alpha, \ell (1-\alpha))$
\\ \hline
 $\loss{GR2R-NLL}$ & 
 $\|\sqrt{\bSigma^{-1}}(f(\y_1)-\y_2)\|^2$ & $-\gamma \y_2^{\top}\log f(\y_1)+  \boldsymbol{1}^{\top}f(\y_1)$ &  
 $ \log f(\y_1) +  \y_2/f(\y_1)  $ 
 \\ \hline
 \begin{tabular}{c}
$\loss{SURE}$    =  \\
    $ \lim_{\alpha \to 0}  \lossarg{GR2R}{\y,\alpha}$
\end{tabular} &  \begin{tabular}{c}
    $\|f(\y)-\y\|^2+ $  \\
     $ 2\, \trace{\bSigma \der{f}{\y}(\y)} $
\end{tabular}  & 
\begin{tabular}{c}
    $\|f(\y)-\y\|^2 + $  \\
    $2 \sum_{i=1}^{n} y_i (f_i(\y) - f_i(\y-\gamma \boldsymbol{e}_i)) $
\end{tabular} &  
\begin{tabular}{c}
    $\|f(\y)-\y\|^2 + $  \\
    $2 \sum_{i=1}^{n} \sum_{k\geq 1} b(\ell, k)(-y_i)^{k+1}  \nder{f_i}{y_i}{k} (\y) $
\end{tabular}
\\ \hline
\end{tabular} } 
\caption{\textbf{Summary of generalized R2R losses.} Popular noise distributions belonging to the natural exponential family and the associated bootstrapping functions with $\alpha\in (0,1)$ and negative-log likelihood losses. 
} \label{tab:NEF}
\end{table*}

\paragraph{Beyond the $\ell_2$ loss}
We can also incorporate the knowledge about the distribution $p(\y_2|\x)$ in the choice of the metric of the loss.
The $\ell_2$ metric is well adapted for Gaussian loss, since it is proportional to the negative log-likelihood under the Gaussian model. Thus, if the noise model is not Gaussian but belongs to the NEF, we can use the negative log-likelihood of the appropriate noise model:
\begin{align*}
    \loss{GR2R-NLL} &= \E{\y_1,\y_2|\y}{ \frac{1}{n}\phi\left(f(\y_1)\right)  - \frac{1}{n}\y_2^{\top} \eta\left(f(\y_1)\right)} \\
&= \E{\y_1,\y_2|\y} {-  \frac{1}{n}\log p_2\Big(\y_2|\hat{\x}=f(\y_1)\Big) } + \const 
\end{align*}
\JT{
Due to the independence of $\y_1$ and $\y_2$ conditional on $\x$, the loss is an unbiased estimator of a supervised negative log-likelihood loss:
\begin{equation*}
    \E{\y|\x}{\loss{GR2R-NLL}} = \E{\y|\x}{-\frac{1}{n}\log p_2\Big(\x|\hat{\x}=f(\y_1)\Big)} + \const 
\end{equation*}
}
These losses (including $\ell_2$) correspond to Bregman divergences, and share the same global minima in expectation over the dataset, i.e., the posterior mean $f(\y)=\E{\x|\y}{\x}$~\cite{monroy_generalized_2025}. 
However, when dealing with finite datasets, using the $\ell_2$ or the negative log-likelihood will lead to different networks $f$.
The resulting negative log-likelihood losses of (anisotropic) Gaussian, Gamma and Poisson noise distributions are included in~\Cref{tab:NEF}.

\subsection{Stein's Unbiased Risk Estimate} \label{subsec: sure}

We now turn to a self-supervised loss that is not based on generating two independent noisy pairs. 
Let us return again to the Gaussian noise model, $\y|\x\sim \G{\x}{\bSigma}$. The following seminal result by Stein~\cite{stein_estimation_1981} shows that we can estimate the correlation between the prediction and the noise without knowledge of the ground truth: 
\begin{lemma}[Stein~\cite{stein_estimation_1981}] \label{lemma: stein}
Let $\y | \x \sim \G{\x}{\bSigma}$ and $f$ be weakly differentiable. Then, we have
\begin{equation}\label{eq: stein lemma}
\E{\y|\x} {f(\y)^{\top} (\y -\x)} = \E{\y|\x}{\trace{\der{f}{\y} \bSigma}}
\end{equation}
\JT{where $\der{f}{\y} \in \R{n\times n}$ is the Jacobian of $f$ at $\y$.}
\end{lemma}
The proof relies on integration by parts, and we leave it to the reader as an exercise. When the noise is isotropic, i.e., $\bSigma = \sigma^2 \Id$ \JT{with $\trace{\der{f}{\y}} = \sum_{i=1}^n \der{f_i}{y_i}$}, this shows that the correlation between the noise and prediction is simply proportional to the divergence of the estimator, $f$.

Building on this lemma, the Stein's Unbiased Risk Estimate (SURE) is given by:
\begin{equation}\label{eq:sure} \tag{SURE}
 \loss{SURE} = \frac{1}{n}\| f(\y) - \y \|^2 + \frac{2}{n}\, \trace{\der{f}{\y} \bSigma} - \frac{1}{n}\trace{\bSigma}
\end{equation}
and is an unbiased estimate of the supervised loss as
\begin{align*}
&\E{\y|\x}{\frac{1}{n}\| f(\y) - \y \|^2 + \frac{2}{n}\, \trace{\der{f}{\y} \bSigma} - \frac{1}{n}\trace{\bSigma}} \\ 
&= \E{\y|\x}{\frac{1}{n}\| \left(f(\y)-\x\right) - (\y-\x) \|^2 + \frac{2}{n}\, \trace{\der{f}{\y} \bSigma}  - \frac{1}{n}\trace{\bSigma}}  \\ 
&= \E{\y|\x}{\frac{1}{n}\|f(\y) - \x\|^2 - \frac{2}{n}\, f(\y)^{\top} (\y -\x) + \frac{2}{n}\, \trace{\der{f}{\y} \bSigma} }\\
&= \E{\y|\x}{\frac{1}{n}\|f(\y) - \x\|^2} 
\end{align*}
where the third line uses Stein's lemma \JT{and  $\E{\y|\x}{\|\y-\x\|^2}=\trace{\bSigma}$}.
SURE has been widely popular in the signal processing community well before the advent of neural networks, see e.g.,~\cite{donoho1995adapting}.
To the best of our knowledge, Metzler et al.~\cite{metzler_unsupervised_2016} were the first to use SURE for training deep networks,
and it has been widely used for learning ever since, e.g.,~\cite{zhussip_extending_2019,chen_robust_2022}.

\paragraph{Extensions beyond Gaussian noise}

The SURE loss has been extended beyond Gaussian noise. 
In the case of measurements corrupted by Poisson noise, $\y \sim \gamma \mathcal{P}(\frac{\x}{\gamma})$ with gain $\gamma>0$, we can use the following extension of Stein's lemma, introduced by Hudson~\cite{hudson_natural_1978}:
\begin{equation}\label{eq: hudson poisson}
    \E{\y|\x}{(\y-\x)^{\top}f(\y)} = \E{\y|\x}{ \sum_{i=1}^{n} y_i \left(f_i(\y) - f_i(\y - \gamma \boldsymbol{e}_i) \right)}
\end{equation}
where $\boldsymbol{e}_i$ is a canonical vector containing a one in the $i$th entry and zeros in the rest. In a similar fashion to \Cref{eq:sure}, this lemma can be used to derive the following self-supervised loss~\cite{luisier_pure_2010}:
\begin{equation}\label{eq: pure} \tag{PURE}
  \loss{PURE} = \frac{1}{n}\|f(\y)-\y\|^2  + \frac{2}{n} \sum_{i=1}^{n} y_i \Big(f_i(\y) - f_i(\y + \gamma \boldsymbol{e}_i) \Big)
\end{equation}
which is an unbiased estimator of the supervised loss when the measurements are corrupted by Poisson noise.
Le Montagner et al.~\cite{le_montagner_unbiased_2014} combined the Stein~\cref{eq: stein lemma} and Hudson~\cref{eq: hudson poisson} identities to handle mixed Poisson-Gaussian noise, i.e. $\y\sim \gamma\mathcal{P}(\frac{\x}{\gamma}) + \bepsilon$ with $\bepsilon\sim\G{\boldsymbol{0}}{\bSigma}$, which appears in various imaging applications, such as fluorescence   microscopy~\cite{zhang_poisson-gaussian_2019} or low-dose computed tomography~\cite{ding_statistical_2018}. 

Other similar extensions of SURE to continuous variables belonging to the NEF can be obtained using the following lemma:
\begin{lemma}[Hudson~\cite{hudson_natural_1978}]\label{lemma: hudson}
Let $\y \sim p(\y|\x)$ be a continuous random variable whose distribution belongs to the \Cref{eq: nef}, and let $f$ be weakly differentiable. Then
    \begin{equation}
        \E{\y|\x}{\left(\nabla\log h(\y)+\eta(\x)\right)^{\top}f(\y)} = \sum_{i=1}^{n} \E{\y|\x}{\der{f_i}{y_i}(\y)}.
    \end{equation}
\end{lemma}
\noindent Stein's lemma is recovered for the special case of (isotropic) Gaussian noise with $\eta(\x)=-\frac{\x}{\sigma^2}$ and $h(\y)=\exp(\frac{\|\y\|^2}{2\sigma^2})$. 

Eldar~\cite{eldar_generalized_2009} used this result to build the generalized SURE (GSURE) loss. In the notation of \Cref{eq: nef}, the loss can be written as:
\begin{equation} \label{eq: gsure} \tag{GSURE}
\loss{GSURE} = \frac{1}{n}\| f(\y) + \nabla \log h(\y) \|^2 + \frac{2}{n} \sum_{i=1}^n \der{f_i}{y_i}(\y)
\end{equation}
which, similar to the Gaussian case, we can use Lemma~\ref{lemma: hudson} to show that 
\begin{align}
  \E{\y|\x}{\loss{GSURE}}  =  \E{\y|\x}{ \frac{1}{n}\| f(\y) - \eta(\x)\|^2} + \const
\end{align} 
Note, however, this is not suitable for deriving a loss equivalent to those presented in this chapter, as GSURE 
targets the natural parameter, $\eta(\x)$, instead of $\x$. The optimal $f$ is thus the estimator $f(\y)=\E{\x|\y}{\eta(\x)}$ which is not equal to the posterior mean, except for the Gaussian noise case where $\eta(\x) \propto \x$.


\paragraph{Approximating the divergence term}
For complex models such as neural networks, there are generally no analytic solutions for the divergence term in the SURE loss and it is common practice to resort to Monte Carlo estimates. A first option is to compute the divergence term in the SURE loss using automatic differentiation~\cite{soltanayev2020divergence}, 
together with Hutchinson's unbiased trace estimator~\cite{hutchinson1989stochastic}:
\begin{equation} \label{eq:hutchinson}
      \trace{ \bSigma \der{f}{\y}} \approx \bomega^{\top}\der{f}{\y}\bomega
\end{equation}
where $\bomega\sim\G{\boldsymbol{0}}{\bSigma}$ is a random standard Gaussian, and $\der{f}{\y}\bomega$ is computed using as a Jacobian-vector product via automatic differentiation.

Alternatively, one can use a finite difference approximation to the same estimator via a simple Monte Carlo approximation~\cite{ramani_monte-carlo_2008}
\begin{equation} \label{eq:ramani}
      \trace{ \bSigma \der{f}{\y}} \approx (\bSigma \frac{ \bomega }{\tau})^{\top} \left(f(\y + \tau\bomega) -  f(\y)\right)
\end{equation}
where $\bomega\sim\G{\boldsymbol{0}}{\Id}$ and small $\tau>0$. The approximation becomes exact if we take the expectation over $\bomega$ and let $\tau \to 0$. This approximation avoids the need for automatic differentiation but at the cost of requiring two evaluations of $f$ per calculation. In practice, $\tau$ can be chosen to be a small fraction, for example $1\%$, of the standard deviation of the noise.

Although unbiased, both these estimators can potentially have high variance which can be reduced by averaging over multiple samples of $\bomega$. However, for imaging applications, as argued in~\cite{ramani_monte-carlo_2008}, it is usually reasonable to assume that a single Monte Carlo sample will provide a sufficiently low variance estimate. Intuitively, this is because denoising functions tend to act locally and we are therefore spatially averaging over a large number, $n \ll 1$ of almost independent estimates per pixel.



In the case of Poisson noise, Luisier et al.~\cite{luisier_pure_2010} proposed the following approximation of the Poisson variant \Cref{eq: pure}:
\begin{align}  \label{eq: pure approx} 
      \sum_{i=1}^{n} y_i \left(f_i(\y) - f_i(\y - \gamma \boldsymbol{e}_i) \right)&\approx   \trace{\gamma\diag{\y} \der{f}{\y}}
      \\ &\approx (\gamma\diag{\y}\frac{\bomega}{\tau})^{\top} (f(\y + \tau \bomega) - f(\y)) \notag
\end{align}
The approximation relies on the assumption that the Poisson noise is approximately Gaussian with a signal dependent covariance, that is  $\bSigma\approx \diag{\y}$ when $\gamma$ is large, 
and, it is not well suited for low $\gamma$ settings.

\paragraph{Equivalence with Recorrupted2Recorrupted}
The attentive reader might have noticed that the synthetic noise in R2R~\cref{eq: gaussian r2r} and the Monte Carlo approximation of SURE~\cref{eq:ramani} are relatively similar.
It turns out that \cref{eq: R2R} can be seen as another Monte Carlo approximation of the analytic \Cref{eq:sure} as $\alpha\to0$. This observation was first made by Oliveira et al.~\cite{oliveira_unbiased_2022} for the Gaussian noise case, then extended to the (discrete) Poisson case~\cite{oliveira_unbiased_2023}, and finally extended to the continuous natural exponential family in the following proposition:
\begin{proposition}[Monroy et al.~\cite{monroy_generalized_2025}] \label{prop: r2r sure}
Assume that $f$ is analytic, $p(\y|\x)$ is a continuous distribution belonging to the NEF, and that
$a_k:\mathbb{R}\mapsto\mathbb{R}$ as 
$$
a_k(y_i) = \lim_{\alpha\to 0} \E{y_{2,i}|y_i,\alpha}{(y_{2,i}-y_i)(\alpha y_{2,i})^{k}}
$$ 
for all $i=1,\dots,n$ verifies $|a_k(y_i)|<\infty$ for all positive integers $k\geq 1$.  Then,
 \begin{multline*}
  \lim_{\alpha\to 0} \mathcal{L}_{\text{GR2R}}\left(\y, f, \alpha\right) =  \\
\frac{1}{n}\| f(\y) - \y\|^2  +  \frac{2}{n} \sum_{i=1}^n \sum_{k\geq 1} (-1)^{k+1} a_k(y_i) \frac{1}{k!}\frac{\partial^k f_i}{\partial y_i^k}(\y)  + \const
 \end{multline*}
 where the resulting SURE-like loss is an unbiased estimator of the supervised loss with input $\y$ instead of $\y_1$, that is
 $$
   \E{\y|\x}{\lim_{\alpha\to 0} \mathcal{L}_{\text{GR2R}}\left(\y, f, \alpha\right)} = \E{\y|\x}{\frac{1}{n}\|f(\y)-\x\|^2}.
 $$
\end{proposition}
Interestingly, in the isotropic Gaussian case we have $a_1(y_i)=\sigma^2$ and $a_k(y_i)=0$ for $k\geq 2$, recovering the standard SURE formula. \JT{In the Poisson noise case~\cite{oliveira_unbiased_2023}, the R2R loss recovers~\Cref{eq: pure} as $\alpha\to0$, without relying on the continuous approximation in~\Cref{eq: pure approx}. Unlike \Cref{eq: gsure} that requires a single divergence term but learns the conditional estimator $\E{\x|\y}{\eta(\x)}$, \Cref{prop: r2r sure} shows that SURE-like formulas exist for learning the conditional mean $\E{\x|\y}{\x}$, albeit they often require computing higher-order derivatives of $f$.} 


\paragraph{Connection with Tweedie's formula}

The second term in~\cref{eq:sure} verifies the following equality 
$$
\E{\y}{\trace{\der{f}{\y} \bSigma}} = - \E{\y}{f(\y)^{\top} \bSigma \score }
$$
which can be shown again using integration by parts.
Using this result, we can find the optimal denoiser under a SURE loss by solving
\begin{align*}
  f^*  &= \argmin_{f} \;   \E{\y}{\loss{SURE}} \\ 
   &= \argmin_{f}   \; \E{\y}{\frac{1}{n}\| f(\y) - \y \|^2 - \frac{2}{n}\, f(\y)^{\top} \bSigma \score }  \\
   &=\argmin_{f} \;    \E{\y}{\frac{1}{n}\| f(\y) - \y - \bSigma \score  \|^2} 
\end{align*}
where the last step completes squares, and $\score$ is known as \emph{the score} of the measurement distribution. The optimal solution is thus the well-known Tweedie's formula:
\begin{equation}\label{eq: tweedie} \tag{Tweedie}
    f^*(\y)  = \y + \bSigma \, \score
\end{equation}
\JT{Since SURE is an unbiased estimator of the supervised $\ell_2$ loss, whose global minimizer is the conditional mean estimator $ f^*(\y)=\E{\x|\y}{\x}$, we have that the optimal least-squares estimator in the Gaussian noise case is given by $\E{\x|\y}{\x}=\y + \bSigma \, \score$.}

We can further combine \Cref{eq:sure} and \Cref{eq: tweedie} to compute the minimum mean squared error (MMSE) of the denoising problem, 
as a function of the score and the noise covariance:
\begin{align} \label{eq: mmse tweedie}
     \text{MMSE} &=  \E{\y}{\mathcal{L}_{\text{SURE}}(\y, f^*)} \\
     &=  \frac{1}{n}\trace{\bSigma}- \E{\y}{\frac{1}{n}\|\bSigma\score\|^2}
\end{align}

Inspired by the \JT{close} link between SURE and Tweedie's formula, the Noise2Score approach~\cite{kim_noise2score_2021,kim_noise_2022} proposes to first train a network $s(\y)$ that approximates the score, i.e., $s(\y) \approx \score$, and then applies~\cref{eq: tweedie} at test time $f(\y) = \y + \bSigma s(\y)$. 
 
Tweedie's formula also plays a significant role in diffusion models, as it allows one to evaluate the score function indirectly via a denoiser network $f(\y)$. \Cref{sec: learning from noisy py} discusses self-supervised diffusion methods that leverage this connection.

\section{Partially unknown noise distribution}
In many real-world settings, we do not have independent pairs $(\y_1,\y_2)$, and thus cannot use \Cref{eq: noise2noise}, nor do we know exactly the noise distribution and thus cannot use \Cref{eq: R2R} or \Cref{eq:sure} losses. Under certain circumstances we will see that we can still build self-supervised losses, by paying a price on the flexibility of the learned denoiser: we can only expect to minimize a \emph{constrained} supervised loss~\cite{tachella_unsure_2025}, that is
\begin{equation}\label{eq:unknown noise}
    \argmin_{f} \E{\x,\y}{\| f(\y) - \x \|^2} \; \text{ s.t. } f\in\mathcal{F}
\end{equation}
where $\mathcal{F}$ is a constrained set of functions.
Thus, we generally are not able to approximate the oracle estimator, i.e., the conditional mean $f^{*}(\y)=\E{\x|\y}{\x}$ if it does not belong to $\mathcal{F}$. However, as we will see, in some cases the constraints can be relatively mild, letting us find an $f$ performing close to the oracle. As illustrated in~\Cref{fig: constraints f},
the less we know about the distribution, the stronger the constraints needed, and we get further away from the supervised performance. For example, we might know the noise is isotropic and Gaussian, but not the noise level $\sigma^2$, or more extreme, we might not know the distribution at all, except for an assumption of independence across pixels.

\begin{figure}[t]
    \centering
    \includegraphics[width=1\textwidth,alt={A straight line illustrating the noise assumptions (below the line) and the constraints on the learned estimator (above the line, in red) of different self-supervised losses, from SURE and R2R on the left which rely on a fully known noise model, to Noise2Self and BlindSpot Nets on the right which only rely on separable noise distributions.}]{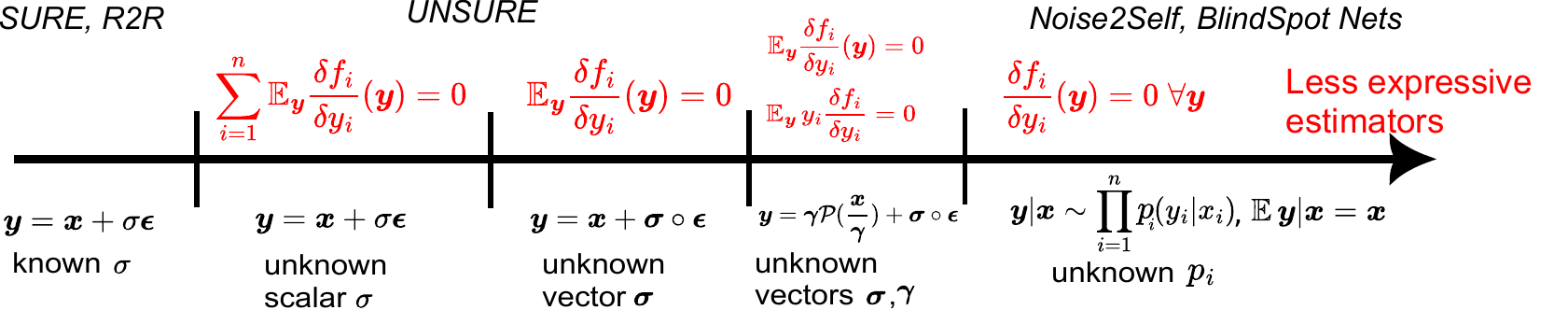}
    \caption{\textbf{The expressivity-robustness trade-off in self-supervised denoising~\cite{tachella_unsure_2025}.} As the assumptions about the noise distribution are relaxed, the learned estimator needs to be less expressive to avoid over-fitting the noise.}
    \label{fig: constraints f}
\end{figure}

\subsection{Unknown noise level Stein's Unbiased Risk Estimate}

In some applications, the noise model can be assumed to be Gaussian 
$ \y = \x + \sigma\bepsilon$ 
with $\bepsilon\sim \G{\boldsymbol{0}}{\Id}$ but the noise level $\sigma$ is unknown. 
A naive approach could be to estimate $\sigma$ from the noisy data first, and then train using \Cref{eq:sure} or \Cref{eq: R2R}. As illustrated in 
\Cref{fig: MNIST denoising}, both losses are very sensitive to a misspecified $\sigma$, as errors of more than 10\%  can result in a significant decrease of performance.

\begin{figure}[h]
    \centering
    \includegraphics[width=.8\textwidth,alt={A plot showing PSNR (y-axis) vs noise level $\sigma$ (x-axis) for different training losses, including supervised (dotted blue line), R2R assuming $\sigma=0.2$ (pink line), SURE assuming $\sigma=0.2$ (gray line), and UNSURE (orange). The plot shows how UNSURE is slightly below supervised for all noise levels, whereas SURE and R2R perform well at $\sigma=0.2$ and completely fail for other noise levels.}]{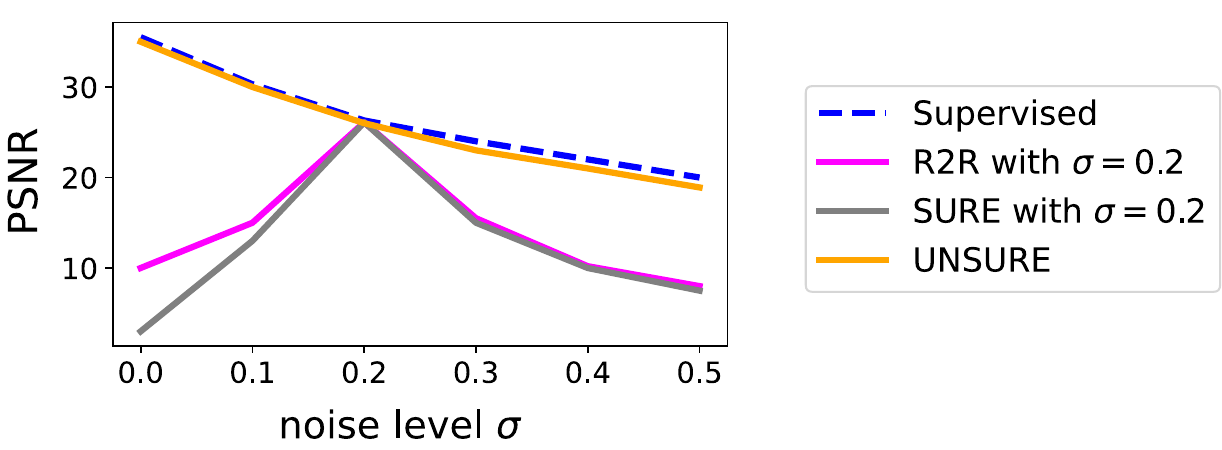}
    \caption{\textbf{Self-supervised denoising across noise levels~\cite{tachella_unsure_2025}.}  Comparison of supervised, \Cref{eq:sure}, \Cref{eq: R2R} and \Cref{eq: unsure} losses on 
    an MNIST Gaussian denoising task with a U-Net denoiser. If the noise level $\sigma$ is correctly specified in SURE and R2R, 
    the performance is close to the supervised case. However, if the noise level is misspecified, the performance drops significantly.
    The UNSURE loss is robust to noise level misspecification, and performs close to the supervised case.}
    \label{fig: MNIST denoising}
\end{figure}

We can instead build a self-supervised loss that ensures that the
divergence term in \Cref{eq:sure} is zero, and thus drop the dependency of the loss on the unknown noise level $\sigma$:
\begin{align} \label{eq: unsure constraints}
\argmin_f  \; \E{\y}{\frac{1}{n}\|f(\y)-\y\|^2} \;  \text{ s.t. } \E{\y}{\sum_{i=1}^n \der{f_i}{y_i}(\y)} = 0
\end{align}
Applying Lemma~\ref{lemma: stein}, we can show that the minimization problem is equivalent to a supervised setting with constraints:
\begin{align*} 
\argmin_{f}   \; \E{\x,\y}{\frac{1}{n}\|f(\y)-\x\|^2} \; \text{ s.t. } f\in \mathcal{F}
\end{align*}
where we only look for denoisers that have zero-expected divergence (ZED), i.e., which belong to the constrained set 
$$\mathcal{F} = \left\{f : \E{\y}{\sum_{i=1}^n \der{f_i}{y_i}(\y)} = 0 \right\}.$$
Using a Lagrange multiplier $\eta \in \R{}$, we can rewrite the constrained learning problem in~\cref{eq: unsure} as  
\begin{align} \label{eq: unsure} \tag{UNSURE}
   \min_f \max_{\eta} \, & \E{\y}{\frac{1}{n}\|f(\y)-\y\|^2  +\frac{2\eta}{n} \, \sum_{i=1}^n \der{f_i}{y_i}(\y)}.
\end{align}
Interestingly, the resulting loss is very similar to~\Cref{eq:sure}, but 
instead of having a fixed noise level $\sigma$, we replace it by a multiplier $\eta$ and maximize over it.

\paragraph{Analyzing the constrained estimator}
What is the cost of adding the zero-expected divergence constraint on the learned denoiser?
It is easy to show that the optimal denoiser has a positive divergence~\cite{gribonval_should_2011}, except for the trivial case
where the image distribution consists of a single image. However, we will see that the constraint is quite mild, and the gap with supervised learning
can be small.

Following a similar derivation as in~\cref{eq: tweedie}, we can obtain the optimal solution for the constrained learning problem again in terms of the score function, 
which can be written as $$f^{\text{ZED}}(\y)=\y + \frac{n\score}{\Es{\y}\|\score\|^2}.$$
We can also compute the expected error if the zero-expected divergence estimator by simply plugging-in its definition inside \cref{eq:sure}:
\begin{align} \label{eq: UNSURE gap}
     \E{\y,\x}{\frac{1}{n}\|f^{\text{ZED}}(\y)-\x\|^2} &= \E{\y}{\mathcal{L}_{\text{SURE}}\Big(\y, f^{\text{ZED}}(\y))}  \\ \notag
     &= \frac{n}{\Es{\y}\|\score\|^2} - \sigma^2 \\\notag
     &= \text{MMSE} \left(1- \frac{\text{MMSE}}{\sigma^2}\right)^{-1}\\\notag
     &= \text{MMSE} + \sigma^2 \sum_{j\geq 2} \left(\frac{\text{MMSE}}{\sigma^2}\right)^j
\end{align}
where the third line uses the expression of the MMSE in~\cref{eq: mmse tweedie} (i.e., the error of the optimal estimator $f(\y)=\E{\x|\y}{\x}$) and the last line relies on the geometrical series formula.
Since $\frac{\text{MMSE}}{\sigma^2}$ is the improvement in signal-to-noise ratio of the optimal estimator, we always have that $\frac{\text{MMSE}}{\sigma^2}<1$ and generally $\frac{\text{MMSE}}{\sigma^2}\ll 1$.
Thus, the additional error of $f^{\text{ZED}}$ compared to the oracle can be quite small.

\paragraph{Extensions beyond isotropic Gaussian noise}
The \Cref{eq: unsure} approach can be further extended to settings where the noise covariance is unknown, 
by considering an $s$-dimensional set of \emph{plausible} covariance matrices
$$
\mathcal{R} = \left\{ \bSigma_{\boldeta} \in \R{n\times n} : \bSigma_{\boldeta} = \sum_{j=1}^s \eta_j \bPsi_{j} , \boldeta \in \mathbb{R}^{s} \right\}
$$ 
for some \JT{fixed} basis matrices $\{\bPsi_j \in \mathbb{R}^{n\times n}\}_{j=1}^{s}$, with the hope that the true covariance belongs to this set, that is $\bSigma \in \mathcal{R} $.
In this case, we consider $s\geq 1$ constraints, and minimize the following objective
\begin{align} \label{eq: divfree general problem}
\argmin_f  & \; \E{\y}{\frac{1}{n}\|f(\y)-\y\|^2}  \\ \text{ s.t. }& \E{\y}{\trace{\bPsi_j \der{f}{\y}(\y)}}= 0, \; j=1,\dots,s 
\end{align}
Note that \Cref{eq: unsure} is recovered as a special case with $s=1$ and $\bPsi_1 = \Id$. 
The less we know about the covariance, the larger the set $\mathcal{R}$, resulting in more constraints 
on the learned estimator. Thus, the dimension $s\geq 1$ offers a trade-off between optimality of the resulting estimator and robustness to a misspecified covariance.
As with the isotropic case, we can find the closed-form expression of the optimal constrained denoiser: 

\begin{theorem}[Tachella et al. \cite{tachella_unsure_2025}] \label{theo: general solution}
  Let $p_{\y}=p_{\x} * \mathcal{N}(\boldsymbol{0}, \bSigma)$ and assume that $\{\bPsi_j \in \mathbb{R}^{n\times n}\}_{j=1}^{s}$ are linearly independent. The optimal solution of problem \cref{eq: divfree general problem} is given by 
\begin{equation}
   f(\y) =\y+ {\bSigma}_{\hat{\boldeta}} \, \nabla \log p_{\y}(\y) 
\end{equation}
where the optimal multipliers are given by
$\hat{\boldeta} = \Q^{-1} \boldsymbol{v}$,
with $$Q_{i,j}=\trace{\bPsi_i \, \E{\y}{\score \score^{\top}}\bPsi_j^{\top} }$$ and $v_i = \trace{\bPsi_i}$ for $i,j=1,\dots,s$.
\end{theorem}

 
We can apply this generalization to problems with correlated noise and unknown correlations, which is generally modeled as
$$
\y = \x + \boldsymbol{\sigma} * \bepsilon
$$
where $*$ denotes the convolution operator, $\boldsymbol{\sigma}\in \mathbb{R}^{p}$ is vector-valued and  $\bepsilon\sim\mathcal{N}(\boldsymbol{0},\Id)$. If we do not know the exact noise correlation, we can consider the set of covariances with correlations up to $\pm r$ taps/pixels\footnote{Here we consider 1-dimensional signals for simplicity but the result extends trivially to the 2-dimensional case.}, 
we can choose $\bSigma_{\boldeta}$ to be a positive definite circulant matrix parameterized by a filter $\boldeta$. In this case, the solution is
$$
    f(\y) = \y + \hat{\boldeta} * \score 
$$
with optimal multipliers \JT{given by $2r+1$ autocorrelation coefficients of the score~\cite{tachella_unsure_2025}.}

We can also extend the UNSURE approach to non-Gaussian noise distributions using Lemma~\ref{lemma: hudson}, such as Poisson-Gaussian noise with unknown parameters~\cite{tachella_unsure_2025}.

\subsection{Cross-validation methods}\label{sec: cross-validation}

In many applications, the noise is separable across pixels/measurements $p(\y|\x)=\prod_{i=1}^n p_i(y_i|x_i)$, but the noise distribution $p_i$ at each pixel is unknown except for the assumption that the mean is $\E{y_i|x_i}{x_i}=x_i$. 
In such settings, \JT{none of the losses we presented so far are applicable, but 
we can still find a loss that is an unbiased estimator of a constrained supervised loss. Since we do not know the noise distribution, we need to impose \emph{stronger constraints} than the zero expected divergence one in~\Cref{eq: unsure} which relies on a Gaussian noise assumption.

Many recent self-supervised methods, including Noise2Void~\cite{krull_noise2void_2019}, Noise2Self~\cite{batson_noise2self_2019}, blind spot networks~\cite{laine_high-quality_2019}, Neighbor2Neighbor~\cite{huang_neighbor2neighbor_2021}, can be broadly classified as \emph{cross-validation} approaches~\cite{efron_estimation_2004}, that minimize the following objective~\cite{tachella_unsure_2025}:}
\begin{align}\label{eq:cross} \tag{CV}
    \argmin_{f} &\quad  \E{\y}{\frac{1}{n}\| f(\y) - \y \|^2} \\ \notag &\text{ s.t. }    \der{f_i}{y_i}(\y)  = 0,  \; \forall \y \in \R{n}, \; i=1,\dots,n
\end{align}
where the derivative constraints are equivalent to asking that the $i$th output $f_i$ doesn't depend on the $i$th input $y_i$.

\JT{
\begin{proposition}[Adapted from Batson and Royer~\cite{batson_noise2self_2019}] \label{prop: cross val}
Let $f:\R{n}\to\R{n}$ be a function whose $i$th output does not depend on the $i$th entry, or equivalently, that $\der{f_i}{y_i}(\y)  = 0$  for all $\y \in \R{n}$ and $i=1,\dots,n$, and assume further that $p(\y|\x)=\prod_{i=1}^n p_i(y_i|x_i)$ and $\E{y_i|x_i}{y_i} = x_i$ for all $i=1,\dots,n$.
Then,
\begin{equation}
    \E{\y|\x}{\| f(\y) - \y \|^2} = \E{\y|\x}{\| f(\y) - \x \|^2} + \const
\end{equation}
\end{proposition}
\begin{proof}
\begin{align*}
&\E{\y|\x}{\| f(\y) - \y \|^2} \\ 
&= \E{\y|\x}{\|f(\y) - \x\|^2} + 2 \, \sum_{i=1}^n\E{\y|\x} {f_i(\y)^{\top} (y_i - x_i)}+\const  \\
&= \E{\y|\x}{\|f(\y) - \x\|^2} + 2 \, \sum_{i=1}^n\E{\y_{-i}|\x}{f_i(\y)}^{\top} \E{y_i|x_i}{y_i - x_i} + \const \\
&= \E{\y|\x}{\|f(\y) - \x\|^2} + \const
\end{align*}
where $\y_{-i}$ denotes the a vector with all entries of $\y$ except for the $i$th entry, the third line uses the fact that $f_i(\y)$ and $y_i$ are independent conditional on $\x$, and the last line uses  $\E{y_i|x_i}{y_i}=x_i$.
\end{proof}
}

Due to \Cref{prop: cross val}, the minimization problem in \Cref{eq:cross} is equivalent to the following constrained supervised problem
\begin{align*} 
\argmin_{f}   \; \E{\x,\y}{\frac{1}{n}\|f(\y)-\x\|^2}  \; \text{ s.t. } f \in \mathcal{F}
\end{align*}
where we only look for denoisers which do not use the $i$th input value for predicting the $i$th output value, that is $$\mathcal{F} = \left\{f :  \der{f_i}{y_i}(\y)  = 0,  \; \forall \y \in \R{n}, \; i=1,\dots,n \right\}.$$

\JT{As with \cref{eq: unsure}, the conditional mean $f^{*}(\y)=\E{\x|\y}{\x}$ does not belong to the constrained set $\mathcal{F}$, and we cannot expect to achieve the same performance as supervised learning. In this case, the optimal solution of~\Cref{eq:cross} is $f_i^{\text{CV}}(\y)=\E{x_i|\y_{-i}}{x_i}$ for $i=1,\dots,n$. Here, unlike the UNSURE case where the suboptimality gap is available closed form in \Cref{eq: UNSURE gap}, the gap between the $f^{*}$ and $f^{\text{CV}}$ does not admit a simple closed-form expression and will be highly dependent on the signal distribution: we expect the gap to be smaller in signals exhibiting strong spatial correlations, and larger in sparse signals~\cite{batson_noise2self_2019}.
}

Two main approches have been proposed for enforcing the zero derivatives constraints: (1) blind spot networks, which use a specific architecture that enforce the constraint, and (2) splitting losses, which enforce it during training.

\paragraph{Blind-spot networks}
Laine et al.~\cite{laine_high-quality_2019} proposed an image-to-image network architecture $f$ which only relies on the neighbors of a pixel (and not the pixel itself) to predict its denoised value, which is equivalent to imposing $\der{f_i}{y_i}(\y)  = 0$ for all pixels $i=1,\dots,n$. This blind-spot network relies on a fully convolutional architecture which combines shifted (upwards, downwards, leftwards and rightwards) receptive fields.
This approach is very efficient since it allows one to train the denoiser $f$ directly on the measurement consistency loss $\| f(\y) - \y \|^2$, but imposes strong architectural constraints on $f$ above and beyond the zero gradient constraint. 

\paragraph{Splitting methods}
A second approach to training networks that do not rely on the central pixel for denoising is to randomly mask this pixel out during the training procedure~\cite{krull_noise2void_2019,batson_noise2self_2019}. This approach can be written as the following loss
\begin{equation}\tag{SPLIT} \label{eq: noise2self}
    \loss{SPLIT} = \E{\m}{\frac{1}{n}\| (\boldsymbol{1} - \m)\circ \left(f(\m \circ \y) - \y\right)\|^2}
\end{equation}
where $\circ$ denotes the elementwise (Hadamard) product and $\m\in \{0,1\}^{n}$ are random binary masks. 
There is a large literature regarding the choice of splitting distribution $p(\m)$, which generally depends on the structure of the data (e.g., images, videos, etc.).
Some of the most popular choices are:

\begin{figure}[t]
    \centering
    \includegraphics[width=1\textwidth,alt={Three plots illustrating the choice of input (in blue) and target pixels (in red) for the Noise2Void, Noise2Self and Neighbor2Neighbor losses. Noise2Void inputs all pixels of a patch except for the central pixel, which is used as target. Noise2Self uses a random non-overlapping split of pixels, whereas Neighbor2Neighbor chooses two pixels per $2\times 2$ neighborhood, one as input and the other as target.}]{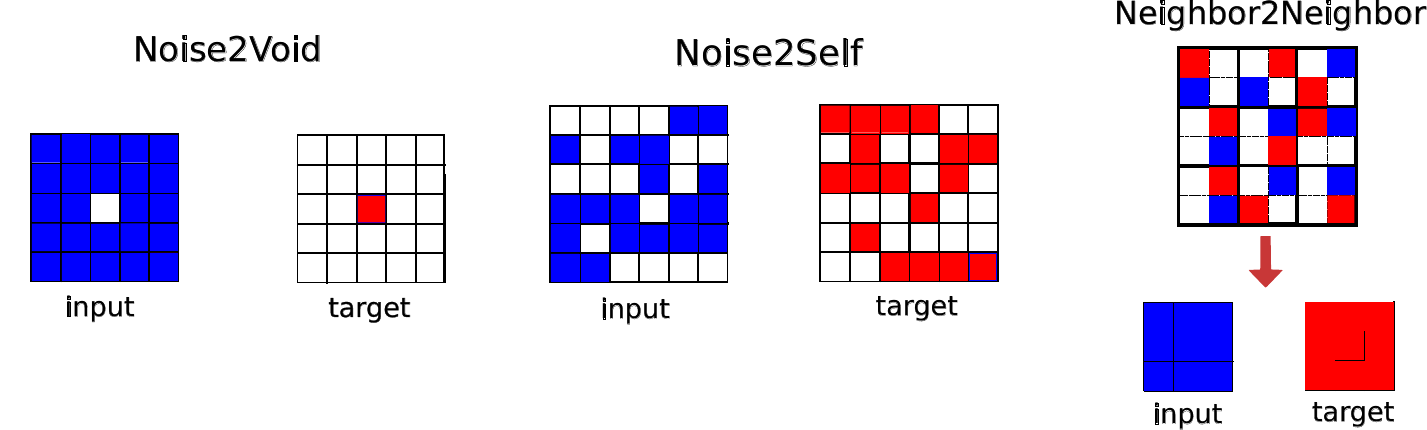}
    \caption{\JT{\textbf{Pixel splitting strategies.} Noise2Void and Noise2Self zero-fill or copy neighboring values to the pixels removed from the input image, whereas Neighbor2Neighbor splits using two random subsamplings of every $2\times 2$ neighborhood, one as input and the other as the target.}}
    \label{fig: splitting strategies}
\end{figure}
\begin{enumerate}   
    \item Noise2Void~\cite{krull_noise2void_2019} replaces a central input pixel with a random neighboring pixel, and computes the loss only on the pixels that have been replaced.
    \item Neighbor2Neighbor~\cite{huang_neighbor2neighbor_2021} constructs two non-overlapping subsampled versions of an image by randomly
    choosing a pixel from every $2 \times 2$ patch for the input and another pixel from the same patch for the target. This loss implicitly assumes that 
    the images are scale-invariant, and thus the denoiser can be trained on undersampled images and tested at full resolution.
    \item Noise2Self~\cite{batson_noise2self_2019} partitions an image using $J$ disjoint masks, $\m^{(j)}$, such that $\frac{1}{J}\sum_{j=1}^J b_{i}^{(j)} = 1$ for $i = 1, \dots, n$, and then trains an estimator, $f$, constructed in the following manner:
 \begin{equation*}   
    f(\y) = \frac{1}{J}\sum_{j=1}^J \m^{(j)} \circ  g\Big((1-\m^{(j)}) \circ \y + \m^{(j)} \circ \ubold\Big)
\end{equation*}
    where $g$ is a based neural network and $\ubold$ is an i.i.d. uniformly distributed random vector\footnote{Although not theoretically justified,~\cite{batson_noise2self_2019} also observed good performance by simply using the trained based neural network, $g$, as the final estimator.}.
\end{enumerate}
\Cref{fig: splitting strategies} illustrates the different masking approaches. It is also possible to use a more general collection of $J$ random masks, $\m^{(j)}$, as long as they cover the entire image (that is $\sum_{s=1}^J (1-b_i^{(s)}) > 0$ for all pixels $i=1,\dots,n$). Then, in a similar fashion to \Cref{eq: R2R}, at test time we can average over multiple splittings~\cite{millard_theoretical_2023}:
\begin{equation}
    f^{\text{test}}(\y) = \sum_{j=1}^J \boldsymbol{w}^{(j)} \circ f(\m^{(j)}\circ \y)
\end{equation}
with weights
$$
w_i^{(j)} = \frac{1-b_i^{(j)}}{\sum_{s=1}^J (1-b_i^{(s)})} \; \text{ for } \; i=1,\dots,n
$$


In the next chapter, we will present an extension of this idea to general inverse problems, where the forward operator is non-trivial $\A(\x)\neq \x$, and splitting strategies are developed in an operator-specific way.

\section{Learning generative models from noisy data} \label{sec: learning from noisy py}
We have seen that approximating the posterior mean is possible with a self-supervised loss, as long as the noise distribution is known. We can also ask whether other posterior statistics beyond the mean can be approximated, or more generally, if we can estimate the signal distribution, $p_{\x}$, from measurement data alone. 
\JT{
Since we have that
\begin{equation} \label{eq: p_x inv prob}
   p_{\y}(\y) = \int_{\x} p(\y|\x) p_{\x}(\x) d\x
\end{equation}
this problem can be seen as a linear inverse problem in the space of measures, which can be written as $p_{\y} = \mathcal{A}(p_{\x})$, where $\mathcal{A}$ is associated to the integral with a kernel $p(\y|\x)$. This formulation can be traced back to Robbins work~\cite{robbins_empirical_1964} in empirical Bayes estimators.

Here} we consider the simplest setup: a denoising problem with additive noise, i.e., $\y = \x + \bepsilon$. Since the noise is additive, we have that the measurement distribution is a convolved version of the signal distribution, i.e.,
\begin{equation}
    p_{\y} = p_{\x} * p_{\bepsilon}
\end{equation}
and thus model identification can be seen as a deconvolution problem. The Fourier analog of this problem is
\begin{equation}\label{eq: char}
    \phi_{\y}(\bomega) = \phi_{\x}(\bomega) \phi_{\bepsilon}(\bomega)
\end{equation}
where $\phi_{\y}(\bomega)=\E{\y}{\exp(\mathrm{i}\y^{\top}\bomega)}$, $\phi_{\x}(\bomega)=\E{\x}{\exp(\mathrm{i}\x^{\top}\bomega)}$ and $\phi_{\bepsilon}(\bomega)=\E{\bepsilon}{\exp(\mathrm{i}\bepsilon^{\top}\bomega)}$ 
are the characteristic functions of $p_{\x}$, $p_{\y}$ and $p_{\bepsilon}$ and $\bomega\in\R{n}$. By simple inspection of \cref{eq: char}, we can deduce that, if the noise model is known and hence $\phi_{\bepsilon}(\bomega)$ is known, it is possible to identify the signal distribution as long as $\phi_{\bepsilon}(\bomega)\neq 0$ for all $\bomega\in\R{n}$:
\begin{proposition}[Tachella et al.~\cite{tachella_sensing_2023}] \label{prop:noise}
If the characteristic function of the noise distribution $\phi_{\bepsilon}$ is nowhere zero, then there is a one-to-one mapping between the spaces of clean measurement distributions and noisy measurement distributions. 
\end{proposition}
For example, the Gaussian noise distribution has a nowhere zero characteristic function, and we can thus uniquely identify the signal distribution from noisy measurements\footnote{Note that this says nothing about the sample complexity of this problem, i.e., how hard this would be from finite data.}. We now present some recent methods aiming to learn a generative model from noisy data alone.

\paragraph{Variational autoencoders} 
Assuming that the noise model $p(\y|\x)$ is known (or estimated in a calibration step), we can model the distribution of measurement data as 
\begin{equation} 
\left\{
\begin{array}{l}
    \z \sim \mathcal{N}(\boldsymbol{0},\Id ), \;
    \hat{\x} = f(\z)\\
    \hat{\y} \sim p(\y|\x=\hat{\x}) \\ 
\end{array}
\right.
\end{equation}
where $f$ is a deep network, and $\z$ are latent variables that follow a standard Gaussian distribution. Prakash et al.~\cite{prakash_fully_2020,prakash_interpretable_2021} propose to learn such a generative model using a variational autoencoder, which requires learning an additional encoder network to approximate the distribution $p(\z|\y)$. Once both encoder, $p(\z|\y)$, and decoder, $f$, are learned, at inference we can either generate clean samples from $p_{\x}$ as 
$
\hat{\x}^{(i)}=f(\z^{(i)})
$
for $i=1,\dots,N$ where $ \z^{(i)} \sim \mathcal{N}(\boldsymbol{0},\Id )$, 
or generate clean samples from the posterior distribution by first sampling the latent variables from the encoder, $\z^{(i)} \sim p(\z|\y)$ and then passing these through the decoder network. In the next chapter, we will present  adversarial methods~\cite{bora_ambientgan_2018} for learning a similar generative model in the case of incomplete measurements.

\paragraph{Diffusion methods} 

\Cref{eq: tweedie} shows that optimal Gaussian denoisers as a function of noise level provide access to the score function of the noisy signal distribution, $\score$. Diffusion models leverage this using neural network approximations to the score to build stochastic samplers that can approximately sample from the clean signal distribution $p(\x)$ or approximately sample from the conditional posterior $p(\x|\y)$ density given a noisy instance, $\y$~\cite{daras_survey_2024}.

As we have seen in this chapter, we can learn the denoisers in a self-supervised manner via \cref{eq:sure} or \Cref{eq: R2R} using noisy data alone. 
However, this only gives us access to an approximation for the score at a noise level greater than or equal to the observed data, $\sigma\geq \sigma_n$. Some recent approaches stop the diffusion at this noise level~\cite{kawar2023gsure}, while others attempt to go below this by imposing a consistency constraint on the learned denoiser~\cite{daras_consistent_2024}.




\section{Towards general inverse problems} \label{sec: general probs}
What happens if we want to extend the losses in this chapter beyond a simple denoising problem, where $\A:\R{n}\mapsto\R{m}$ is a general forward operator?
Defining a denoising function as $\A\circ f$,
we can use any of the losses above in the measurement space. For example, if we observe measurements with Gaussian noise $\y\sim \G{\A(\x)}{\sigma^2\Id}$ where
the noise level $\sigma$ is known, we can adapt \cref{eq:sure} as 
$$
\loss{SURE} = \frac{1}{n}\|\y - \A\circ f(\y)\|^2 + \frac{2}{n}\,\trace{\bSigma\, \der{\A\circ f}{\y}(\y)}
$$
so that we have the equivalent to the following measurement supervised loss
$$
\E{\y}{\loss{SURE}} = \E{\x,\y}{\frac{1}{n}\|\A(\x) - \A\circ f(\y)\|^2} + \const
$$
and importantly the minimizer is $f^{*}(\y)=\E{\x|\y}{\x}$ if $\A$ is a one-to-one mapping\footnote{\JT{In this case, the self-supervised loss will share the same minimizer as the supervised loss, i.e. satisfying~\cref{eq: same minimum}.}}. However, if $\A$ is not one-to-one, for example if there are more pixels $n$ than measurements $m$, then $f^{*}(\y)\neq \E{\x|\y}{\x}$, and we cannot expect to learn the same solution as in the supervised setting, even if we have a dataset with infinitely many measurement vectors. The next chapter will present some solutions to overcome this limitation.

\section{Summary}

In this chapter, we have seen how to build self-supervised losses that can handle noisy data without requiring access to clean targets. The choice of the loss is dependent on how much knowledge we have about the noise distribution: the more we know, the closer we can expect to get to the performance of fully supervised learning, the
less we know, the more constraints we need to impose on the learned denoiser and the further we get from the supervised performance.  Nonetheless, we have seen that in many cases, the gap with supervised learning can be small, and self-supervised losses can be used to train denoisers that perform well in practice. 
\JT{\Cref{tab: summary chap2} shows a summary of the different loss families covered in this chapter, highlighting the different noise assumptions of each loss.}

\begin{table}[h] \label{tab: summary chap2}
\centering
\begin{tabular}{|l|l|l|l|}
\hline
\textbf{Loss family} & \textbf{Noise assumption} & \begin{tabular}[c]{@{}l@{}} \textbf{Learns optimal} \\ $f^{*}(\y)=\E{\x|\y}{\x}$\textbf{?}\end{tabular} & \textbf{Refs.} \\ \hline
\Cref{eq: noise2noise} & 
\begin{tabular}[c]{@{}l@{}} Two independent \\ noise realizations $(\y_1,\y_2)$ \end{tabular}
& Yes & \cite{lehtinen_noise2noise_2018}  \\ \hline
\begin{tabular}[c]{@{}l@{}}\Cref{eq: R2R}\\ \cref{eq: GR2R}\end{tabular} & 
\begin{tabular}[c]{@{}l@{}}
Natural exponential family \\
or additive noise \\known parameters \end{tabular} & Yes  & \begin{tabular}[c]{@{}l@{}} \cite{pang_recorrupted--recorrupted_2021,moran_noisier2noise_2020} \\ \cite{oliveira_unbiased_2022,monroy_generalized_2025} \end{tabular}  \\ \hline
\Cref{eq:sure} & \begin{tabular}[c]{@{}l@{}} Natural exponential family \\ and Poisson-Gaussian \\ known parameters \end{tabular}  & Yes  & \begin{tabular}[c]{@{}l@{}}  \cite{stein_estimation_1981,hudson_natural_1978} \\ \cite{eldar_generalized_2009,metzler_unsupervised_2016} \\ \cite{zhussip_extending_2019,chen_robust_2022} \end{tabular}  \\ \hline
\Cref{eq: unsure} & \begin{tabular}[c]{@{}l@{}} Natural exponential family \\ and Poisson-Gaussian \\  unknown parameters \\  \end{tabular}  &  \begin{tabular}[c]{@{}l@{}}  No, but small \\ gap see~\Cref{eq: UNSURE gap}  \end{tabular}  & \cite{tachella_unsure_2025} \\ \hline 
\Cref{eq:cross} & \begin{tabular}[c]{@{}l@{}} $p(\y|\x)=\prod_{i=1}^{n} p_i(y_i|x_i)$ \end{tabular} &  \begin{tabular}[c]{@{}l@{}} No, gap depends \\ on spatial corr. \end{tabular}  & \begin{tabular}[c]{@{}l@{}} \cite{huang_neighbor2neighbor_2021,batson_noise2self_2019} \\ \cite{krull_noise2void_2019,laine_high-quality_2019} \\ \cite{wang_blind2unblind_2022} \end{tabular} \\ \hline
\end{tabular}
\caption{\textbf{Summary of the self-supervised losses covered in this chapter.} The natural exponential family includes many popular noise distributions, such as Gaussian, Poisson and Gamma. All losses assume that the noise verifies $\E{\y|\x}{\y} = \x$ (or $\E{\y_1|\x}{\y_1} = \x$ for Noise2Noise).}
\end{table}

\chapter{Learning from incomplete measurements} \label{chap: multioperators}



The self-supervised losses presented in the previous chapter can handle various types of noise model and are applicable to any one-to-one forward operator.
However, what happens when the operator is many-to-one, i.e., non-invertible?
It is easy to show that, even in the simple case of noiseless measurements and linear operator $\A(\x)=\A \x$, we do not have any information in the nullspace of $\A$, the linear subspace $\{\x\in\R{n} : \A\x=\boldsymbol{0}\}$, and thus we cannot learn the reconstruction function in this part of the space:
\begin{proposition}[Chen et al.~\cite{chen_equivariant_2021}] \label{prop:simple}
 Any reconstruction function $f:\mathbb{R}^{m}\mapsto\mathbb{R}^{n}$ of the form
\begin{equation}
\label{eq:unlearnable}
    f(\y) = \A^{\dagger}\y + \JT{(\Id - \A^{\dagger}\A) }\, v(\y)
\end{equation}
where $\A^{\dagger}$ is the pseudo-inverse of $\A$ and $v:\R{n}\to\R{n}$ is any function, verifies measurement consistency $\A f(\y)=\y$, and thus is a global minimizer of $\E{\y}{\frac{1}{m}\|\A f(\y)-\A\x\|^2}$.
\end{proposition}

In order to tackle this problem we therefore need additional information. There are two main ways to overcome this limitation: the first one\JT{, covered in~\Cref{sec: multioperators},} is to train with measurements associated to different forward operators $\{\A_g\}_{g=1}^G$, each with possibly a different nullspace. The second option\JT{, covered in~\Cref{sec: equivariance},} is to assume some invariance of the set of images to a group of transformations $\{\T_g\}_{g=1}^G$, which we will show that is equivalent to observing measurements with the set of operators $\{\A \circ \T_g\}_{g=1}^G$.

\section{Leveraging multiple operators}\label{sec: multioperators}


We assume that measurements are obtained via the following model
\begin{equation} \label{eq: multop model}
\left\{
\begin{array}{l}
    \x \sim p(\x), \;
    \A \sim p(\A) \\
    \y \sim p(\y|\A\x) \\ 
\end{array}
\right.
\end{equation}
where we  consider only \emph{linear forward operators} and assume that a different operator $\A$ is sampled for every measurement $\y$. 
In practice, the distribution of operators $p(\A)$ is generally discrete and finite. Some practical examples are (see~\Cref{fig: multiple operators}):
\begin{itemize}
    \item In accelerated MRI applications~\cite{yaman_self-supervised_2020}, the acceleration mask might change from scan to scan, where $\A=\diag{\m}\F$ where $\F$ is the discrete Fourier transform and the mask is randomly sampled as $\m \sim p(\m)$.
    \item In image inpainting problems, $\A=\diag{\m}$, the missing pixels often vary from image to image~\cite{moran_noisier2noise_2020,bora_ambientgan_2018}, resulting in a set of different masks.
\end{itemize}

\begin{figure}[t]
\centering
\includegraphics[width=1\textwidth,alt={An illustration of image inpainting and accelerated magnetic resonance imaging problems with varying forward operators. On the left, three cat images are shown, each partially obscured by a different black mask, labeled $\A_1\x_1$, $\A_2\x_2$, and $\A_3\x_3$. On the right, three k-space MRI measurements are shown, each with a different part of the Fourier space being masked out.}]{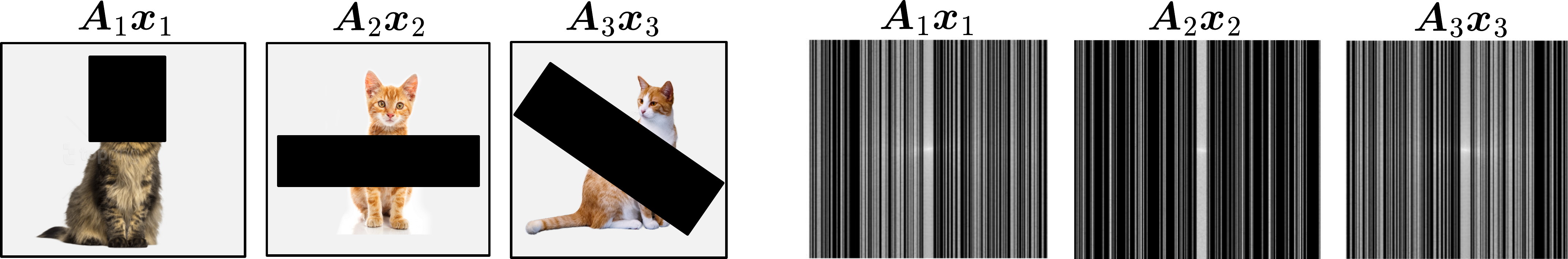}
\caption{\JT{\textbf{Learning with multiple forward operators.} In some imaging settings, such as image inpainting (left) and accelerated MRI (right), the masking operator changes across samples in the dataset, offering different views of the signal distribution.} }
\label{fig: multiple operators}
\end{figure}

Information from multiple operators can help us overcome the limitation of \Cref{prop:simple}. Imagine that there are a maximum set of $G$ forward operators to draw from, $p(\A)=\frac{1}{G}\sum_{g=1}^G \delta_{\A_g}$, and the signal distribution is composed of a single signal, $p(\x)=\delta_{\x_0}$. In this trivial case, we would have access to $G$ different measurement vectors, $\y_g = \A_g \x_0$ for $g=1,\dots,G$, all measuring the same underlying signal $\x_0$, each via one out of the $G$ different forward operators $\A_g \in \R{m \times n}$. We could then try to recover \JT{$\x_0$} by solving the following system of equations: 
\begin{equation} \label{eq: oracle multi-measurement}
   \begin{bmatrix}
    \y_{1} \\
    \vdots \\
    \y_{G}
    \end{bmatrix} =  \begin{bmatrix}
    \A_{1} \\
    \vdots \\
    \A_{G}
    \end{bmatrix} \x_0
\end{equation}
which has $mG$ equations and $n$ unknowns. As this is the maximum number of measurements that we can obtain, to be able to identify $\x_0$ from the measurement data, it is necessary that the system in \Cref{eq: oracle multi-measurement} has maximal rank, or equivalently that $\E{\A}{\A^{\top}\A}=\frac{1}{G}\sum_{g=1}\A_g^{\top}\A_g$ is an invertible matrix.

This argument gives us a necessary condition on the number and diversity of operators required to learn from incomplete measurements in the general case of non trivial signal distribution. It tells us that we need at least $G\geq n/m$ different forward operators in order to learn from incomplete measurements.
However, it does not constitute a practical algorithm nor provides a sufficient condition in the general case of a non-trivial signal distribution, since there will typically be zero probability of observing the same image more than once.


Assuming the operators are known, we can make explicit the dependency of the reconstruction network on the forward operator as $f(\y,\A)$. As discussed in the first chapter, there are various ways to incorporate knowledge of the forward operator into the architecture, with the most common being unrolled optimization algorithms, e.g., \cite{adler_PDnet_2018,Hammernik_varnet_2018}.
A naive approach for handling multiple operators is to minimize measurement consistency:
\begin{equation} 
 \argmin_{f} \E{\y,\A}{\mathcal{L}_{\text{MC}}(\y, \A \circ f)}
\end{equation}
with
\begin{equation} \label{eq: mcloss}\tag{MC}
    \mathcal{L}_{\text{MC}}(\y, \A \circ f) = \frac{1}{m}\|\A f(\y,\A) - \y\|^2.
\end{equation}
Unfortunately, this idea might fail, 
 as $f(\y,\A)=\A^{\dagger}\y$ is a global minimizer of the loss.
This trivial solution is due to the fact that $f$ can achieve zero training error without learning anything about the signal distribution. We can avoid this solution in two ways, using a splitting loss or enforcing consistency across the operators.

\paragraph{Splitting with noiseless measurements}


We can avoid the trivial solution of the measurement consistency loss by removing some of the measurements from 
the input, such that $f$ needs to predict the unobserved part.
Dividing our measurements into two
non-overlapping parts, i.e., $\y = [\y_1^{\top},\y_2^{\top}]^{\top}$ and $\A^{\top} = [\A_1^{\top}, \A_2^{\top}]^{\top}$, we can build a self-supervised loss asking the network to predict $\y$ given $\y_1$:
\begin{equation}\label{eq: multisplit} \tag{MSPLIT}
    \lossarg{MSPLIT}{\y,\A} = \E{\y_1, \A_1|\y,\A}{\frac{1}{m}\| \A f(\y_1,\A_1) - \y\|^2}
\end{equation}
where the expectation averages over some distribution $p(\y_1,\A_1|\y,\A)$ of random splittings. In practice, a single splitting is sampled per gradient step when training a network. There is an extensive literature on how to choose the splitting distribution, which is generally problem-dependent. For example, in accelerated MRI, a popular strategy~\cite{yaman_self-supervised_2020} is to split acceleration masks into non-overlapping sub masks, generally leaving most of the low-frequency information in $\A_1$ to avoid losing too much information.


It is easy to verify that the trivial solution $f(\y_1,\A_1)=\A_1^{\dagger}\y_1$ is not a global minimizer of the expected \Cref{eq: multisplit} loss. An important question is then, does the loss approximate the supervised loss? Assuming the observation model in~\Cref{eq: multop model} with noiseless measurements $\y=\A\x$, 
we have the following result:

\begin{proposition}[Adapted from Daras et al.~\cite{daras_ambient_2024}] \label{prop:split}
Assume the observation model given by \Cref{eq: multop model} with 
noiseless measurements, i.e., $\y=\A \x$.
The multioperator splitting loss is an unbiased estimator of a weighted supervised loss, that is
\begin{align*}
&\E{\y,\A}{\lossarg{MSPLIT}{\y,\A}} \\ 
&\quad =\E{\y_1,\A_1,\x}{\left(f(\y_1,\A_1) - \x\right)^{\top}\Q_{\A_1}\left(f(\y_1,\A_1) - \x\right)}
\end{align*}
where 
\begin{equation} \label{eq: QA1}
    \Q_{\A_1}=\E{\A|\A_1}{\A^{\top}\A} 
\end{equation}
and the global minimizers of the expected loss are given by  
\begin{equation}\label{eq: minimizer split}
f^{*}(\y_1,\A_1)= \Q_{\A_1}^{\dagger}\Q_{\A_1}\E{\x|\y_1,\A_1}{\x} + (\Id - \Q_{\A_1}^{\dagger}\Q_{\A_1})v(\y_1)
\end{equation}
where $v:\R{n}\to\R{n}$ is any function. 
\end{proposition}
\begin{proof}
\begin{align}
&\E{\y,\A}{\lossarg{MSPLIT}{\y,\A}} \\
    & = \E{\y_1, \A_1, \y, \A}{\| \A f(\y_1,\A_1) - \y\|^2}\\
    & =\E{\y_1, \A_1,\x}{\E{\A|\A_1}{\frac{1}{m}\| \A f(\y_1,\A_1) - \A\x\|^2}} \\
    & = \E{\y_1,\A_1,\x}{\frac{1}{m}\left(f(\y_1,\A_1) - \x\right)^{\top}\Q_{\A_1}\left(f(\y_1,\A_1) - \x\right)}
\end{align}
where the second line uses the definition of the observation model, the third line relies on the noiseless measurements assumption and the last line groups uses the definition of $\Q_{\A_1}$. Since this is a weighted $\ell_2$ loss, we can apply \Cref{prop: minimizer weighted l2 loss} to conclude that any minimizer of this loss is given by~\Cref{eq: minimizer split}.
\end{proof}


\begin{figure}[t]
\centering
\includegraphics[width=.8\textwidth,alt={An illustration of the splitting of two different inpainting masks. The mask on the left contains pixels in the lower region, and is split such that $\A_1$ is assigned to the upper most part of the mask. The mask on the right contains pixels on the upper part of the image, but shares the same central pixels as the mask on the left, and a similar $\A_1$ split is selected.}]{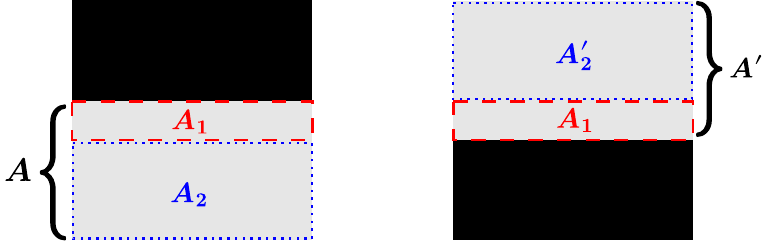}
\caption{\textbf{Illustration of \Cref{prop:split}} Consider a simple image inpainting problem $\A=\diag{\m}$ for some mask $\m$, where images are measured by one of the two masks in gray at random, $\m\sim p(\m)$. The $\A_1$ split denoted by red dashed lines is present in both masks, and thus we have that $\Q_{\A_1}=\A^{\dagger}\A + (\A')^{\dagger}\A'$ is the identity matrix. Hence, the global minimizer of the expected splitting loss is the conditional mean $\E{\x | \A_1\x}{\x}$.}

\label{fig: split intuition}
\end{figure}

If the matrix $\Q_{\A_1}$ is invertible for some split $\A_1$, then $\Q_{\A_1}^{\dagger}\Q_{\A_1}=\Id$, and the minimizer in expectation (c.f. \Cref{eq: same minimum}) is unique, and is given by the conditional mean 
\begin{equation} \label{eq: conditional mean single split}
    f^{*}(\y_1,\A_1)=\E{\x|\y_1,\A_1}{\x}.
\end{equation}

Having an invertible $\Q_{\A_1}$ is thus a sufficient condition to obtain a conditional mean estimator similar to the supervised case, but it is not necessary, since it relies on a single split $\A_1$ of the observed matrix $\A$, whereas we could average over all $J$ possible splittings $\{(\y_1^{(j)},\A_1^{(j)})\}_{j=1}^J$ of the observed $(\y,\A)$ at test time.
If the average
\begin{equation}
    \bar{\Q}_{\A}=\frac{1}{J}\sum_{j=1}^{J} \Q_{\A_1^{(j)}}=\E{\A_1|\A}{\Q_{\A_1}}
\end{equation}
is invertible, we can compute the average prediction at test time as 
\begin{equation} \label{eq: multi evals}
    f^{\text{test}}(\y,\A) = \frac{1}{J}\sum_{j=1}^J \bar{\Q}_{\A}^{-1} \Q_{\A_1^{(j)}} \; f(\y_1^{(j)},\A_1^{(j)})
\end{equation}
where $\bar{\Q}_{\A}^{-1}\Q_{\A_1^{(j)}}$ can be seen as  weighting terms that sum up to the identity, $\frac{1}{J}\sum_{j=1}^J \bar{\Q}_{\A}^{-1} \Q_{\A_1^{(j)}} = \Id$.
Applying~\Cref{prop:split}, we can show that the test time estimate approximates the following average of conditional means:
\begin{align*}
    &f^{\text{test}}(\y,\A) \\&\approx \frac{1}{J}\sum_{j=1}^J \bar{\Q}_{\A}^{-1} \Q_{\A_1^{(j)}}  \Big(\Q_{\A_1^{(j)}}^{\dagger}\Q_{\A_1^{(j)}}\E{\x|\y_1,\A_1^{(j)}}{\x} + (\Id - \Q_{\A_1^{(j)}}^{\dagger}\Q_{\A_1^{(j)}})v(\y_1)\Big) \\
    &= \frac{1}{J}\sum_{j=1}^J \bar{\Q}_{\A}^{-1} \Q_{\A_1^{(j)}} \; \E{\x|\y_1,\A_1^{(j)}}{\x} \\
    &= \E{\A_1|\A}{ \bar{\Q}_{\A}^{-1} \Q_{\A_1} \; \E{\x|\y_1,\A_1}{\x}}.
\end{align*}



In general, closed-form expressions of $\Q_{\A_1}$ and $\bar{\Q}_{\A}$ are not tractable, and practitioners often use a single split or multiple equally-weighted splits $f^{\text{test}}(\y,\A) = \frac{1}{J}\sum_{j=1}^J \; f(\y_1^{(j)},\A_1^{(j)})$ at test time, albeit without any theoretical guarantees.
However, in some specific cases, such as diagonal operators \cite{gan_self_sup_DEQ_2023}, it is possible to compute them explicitly, as illustrated in the example below.

\begin{example}
Consider an image inpainting problem $\A=\diag{\m}$ with random masks, where $b_i \sim \text{Ber}(p_i)$ taking values in $\{0,1\}$ for $i=1,\dots,n$. We can split measurements by an additional masking operation, such that $\A_1 = \diag{\m_1}$ and $\A_2 = \diag{\m_2}$ are also masking operators with $\m_1=\m \circ \bomega$ and $\m_2 = \m \circ (\boldsymbol{1}-\bomega)$ with splitting mask sampled as $\omega_i \sim \text{Ber}(q_i)$. In this case, we have that both $\Q_{\A_1}$ and $\bar{\Q}_{\A}$ are diagonal matrices, as they are given by averages over diagonal matrices.
Due to the separability across pixels
of the mask sampling distributions, we can focus the analysis on a single entry $i\in \{1,\dots,n\}$. 
Letting $[\A_2]_{i,i}=b_i(1-\omega_i)$ with fixed $[\A_1]_{i,i}=b_{1,i}$, the diagonal entries of $\Q_{\A_1}$ are given by
\begin{align*}
[\Q_{\A_1}]_{i,i} &= \E{b_{i}, \omega_{i}| b_{i} \omega_{i} = b_{1,i}}{b_{i}}
\end{align*}
Since
$
\E{b_i, \omega_i| b_i \omega_i = 0}{b_i}= \frac{p_i(1-q_i)}{1-p_iq_i}
$
and
$
\E{b_i, \omega_i| b_i \omega_i =1}{b_i}=1
$
we have that 
\begin{align*}
 [\Q_{\A_1}]_{i,i}  & = \begin{cases}
    1 & \text{if } b_{1,i}=1 \\
    \frac{p_i(1-q_i)}{1-p_iq_i}  & \text{otherwise}
 \end{cases}.
\end{align*}
Thus, if $p_i>0$ and $q_i<1$ for all $i=1,\dots,n$, then $\Q_{\A_1}$ is invertible for all splits $\A_1$.

We can now compute $\bar{\Q}_{\A}$. For a given $\A$ with $[\A]_{i,i}=b_i$, the diagonal entries of $\bar{\Q}_{\A}$ are given by
\begin{align*}
 [\bar{\Q}_{\A}]_{i,i} &= \begin{cases} \frac{p_i(1-q_i)}{1-p_iq_i} & \text{if } b_i=0 \\
    \frac{p_i(1-q_i)^2}{1-p_iq_i} & \text{if } b_i=1 
 \end{cases}.
\end{align*}
Again, if $p_i>0$ and $q_i<1$ for all $i=1,\dots,n$ to have an invertible $\bar{\Q}_{\A}$ for all possible inpainting masks\footnote{Although  presented here as a multiple operator inpainting problem, it was originally proposed in \cite{moran_noisier2noise_2020} as a multiplicative version of the Noisier2Noise self-supervised denoising technique.} $\A$~\cite{moran_noisier2noise_2020}.
\end{example} 

This example also holds for problems with $\A =  \diag{\m}\F$ where $\F$ is a fixed invertible matrix, such as the discrete Fourier transform in accelerated MRI~\cite{millard_theoretical_2023},  and diagonal values are sampled as $\m \sim p(\m)$. Moreover, in these problems, if $\Q_{\A_1}$ is invertible for all possible splits $\A_1$, it is possible to consider a weighted version of the splitting loss~\cite{millard_theoretical_2023}:
\begin{equation*}
    \lossarg{MSPLIT}{\y,\A} = \E{\y_1, \A_1|\y,\A}{\frac{1}{m}\| \Q_{\A_1}^{-\frac{1}{2}} \left(\A f(\y_1,\A_1) - \y \right)\|^2}
\end{equation*}
 to obtain an unbiased estimate of the supervised loss, that is $$\E{\y,\A}{\lossarg{MSPLIT}{\y,\A}}= \E{\y_1,\A_1}{\frac{1}{n}\|f(\y_1,\A_1)-\x\|^2}+ \const$$
 This additional weighting does not modify the global minimizer of the expected loss, but it can improve the performance of the learned $f$ in practice where the dataset is finite.

\paragraph{Splitting with noisy measurements}
What happens if the measurements have noise? We can analyze this case by decomposing the \Cref{eq: multisplit} loss as
\begin{equation*}
   \E{\y_1, \y_2, \A_1, \A_2|\y,\A}{\mathcal{L}_{\text{MC}}(\y_1, \A_1 \circ f) + \frac{1}{m}\| \A_2 f(\y_1,\A_1) - \y_2\|^2  } 
\end{equation*}
where the first term penalizes the measurement consistency with the input $\y_1$ as in~\Cref{eq: mcloss},  
and the second term is associated with the prediction of the unobserved part, $\y_2$. If we have noisy measurement data that is separable across measurements, i.e., $p(\y|\A\x) = \prod_{i=1}^{n} p(y_i|\boldsymbol{a}_i^{\top}\x)$ where $\boldsymbol{a}_i\in \R{n}$ denotes the $i$th row of $\A$, the second term is equivalent to the noiseless case, as $\y_1$ and $\y_2$ are independent given the underlying signal $\x$, that is
$$\E{\y_1,\y_2}{\| \A_2 f(\y_1,\A_1) - \y_2\|^2} = \E{\y_1,\x}{\| \A_2 f(\y_1,\A_1) - \A_2\x\|^2}$$
due to~\Cref{prop: noise2noise}.

However, the measurement consistency term is not equivalent to the noiseless case, as the noise is the same in both input and target:
$$\E{\y_1,\A_1}{\mathcal{L}_{\text{MC}}(\y_1, \A_1 \circ f) } \neq  \E{\y_1,\A_1,\x}{\frac{1}{m}\| \A_1 f(\y_1,\A_1) - \A_1\x\|^2}.$$
To account for noise in the measurement data, we need to replace the consistency loss $\mathcal{L}_{\text{MC}}$ by one of the self-supervised losses for noisy data presented in~\Cref{chap: noisy}, with the specific choice depending on the amount of knowledge we have about the noise distribution:

    \textbf{a)} If the measurements have noise with a known distribution (e.g., Poisson, Gaussian, etc.), we can use the \Cref{eq: GR2R} or \Cref{eq:sure} loss, i.e.,
    \begin{align*}
        &\lossarg{GR2R-MSPLIT}{\y,\A} = \\ &\E{\y_1, \y_2, \A_1, \A_2|\y,\A}{\mathcal{L}_{\text{GR2R}}\left(\y_1,\A_1 \circ f \right) + \frac{1}{m}\| \A_2 f(\y_1,\A_1) - \y_2\|^2  }
    \end{align*}
    As we saw in \Cref{sec: general probs}, the first term is an unbiased estimator of the noiseless measurement consistency, i.e., $$\E{\y_1,\A_1|\x}{\mathcal{L}_{\text{GR2R}}\left(\y_1,\A_1 \circ f \right)} = \E{\y_1,\A_1|\x}{\frac{1}{m}\| \A_1 f(\y_1,\A_1) - \A_1 \x\|^2}$$ and thus use a similar analysis to that in \Cref{prop:split}, to conclude that the minimizer of this loss approximates $\E{\x|\y_1,\A_1}{\x}$ if $\Q_{\A_1}$ is invertible and a single split is used at test time, or $\E{\A_1|\A}{\bar{\Q}_{\A}^{-1}\Q_{\A_1}\E{\x|\y_1,\A_1}{\x}}$ if $\bar{\Q}_{\A}$ is invertible and multiple splits are used at test time.
    
   \textbf{b)} If the measurements have an unknown noise distribution which can be assumed to be separable across measurements, we can simply remove the measurement consistency loss, such that the resulting loss, known by the name SSDU~\cite{yaman_self-supervised_2020}, can be seen as a multioperator extension of~\Cref{eq: noise2self}:
   \begin{equation*}
       \lossarg{SSDU}{\y,\A} = \E{\y_1, \y_2, \A_1, \A_2|\y,\A}{\frac{1}{m}\| \A_2 f(\y_1,\A_1) - \y_2\|^2 }
   \end{equation*}
    In this case, we can derive a similar result than that in~\Cref{prop:split}, but where $\Q_{\A_1} = \E{\A_2|\A_1}{\A_2^{\top}\A_2}$ instead of $\E{\A|\A_1}{\A^{\top}\A}$. As $\A_2$ and $\A_1$ do not overlap, $\Q_{\A_1}$ does not cover the range of $\A_1^{\top}$, and $\Q_{\A_1}$ will not be invertible for any split. Nonetheless, we can average over multiple splits at test time as in~\Cref{eq: multi evals}, such that $\bar{\Q}_{\A}=\frac{1}{J}\sum_{j=1}^J \Q_{\A_1^{(j)}}$ becomes invertible, and the test time estimator, $f^{\text{test}}$, approximates the conditional estimator $\E{\A_1|\A}{\bar{\Q}_{\A}^{-1}\Q_{\A_1}\E{\x|\y_1,\A_1}{\x}}$.
    

\paragraph{Consistency across operators}
The main drawback of splitting-based approaches is that they arguably do not input all the measurement information available to the network. 
The multi-operator imaging (MOI) loss~\cite{tachella_unsupervised_2022} addresses this issue by assuming that each of the imaging problems is invertible, i.e., there exists an $f^*$ such that $\x \approx f^*(\A \x,\A)$ for each $\A$, and then enforcing estimator consistency across the different operators. That is $f(\A \x,\A)\approx f(\A'\x,\A')$ for any pair of operators $\A\neq \A'$ belonging to $p(\A)$: 
\begin{equation}\label{eq: moi} \tag{MOI}
    \lossarg{MOI}{\y, \A} = \E{\A' \sim p(\A)}{\frac{1}{n}\|f\left(\A'f(\y,\A),\A'\right) - f(\y,\A)\|^2} 
\end{equation}
The loss is minimized together with measurement consistency, leading to the following total loss:
\begin{equation}
    \lossarg{}{\y, \A} = \mathcal{L}_{\text{MC}}(\y, \A \circ f) + \lambda \, \lossarg{MOI}{\y, \A}
\end{equation}
where $\lambda>0$ is a trade-off hyperparameter. It is easy to verify that the trivial reconstruction $f(\y,\A)=\A^{\dagger}\y$ is not a minimizer of this
loss, as long as $\A'\A^{\dagger} \neq \A{\A'}^{\dagger}$ for some $\A\neq \A'$. In the case of noisy measurements, the $\mathcal{L_{\text{MC}}}$ can be replaced 
by any of the losses introduced in~\Cref{chap: noisy} according to the amount of knowledge we have about the noise distribution, in a similar way as the robust extensions of~\Cref{eq: multisplit}. 

While the MOI loss does not have the nice equivalence to the supervised loss that we saw in with the splitting losses, it is able to leverage all the available measurements and can be theoretically motivated by the model identifiability theory discussed in \Cref{sec: model identification}.

\section{Leveraging invariance to transformations} \label{sec: equivariance}

In applications where we have a single non-invertible forward operator, Chen et al.~\cite{chen_equivariant_2021} show that it still possible learn in its nullspace if we assume that the distribution of signals is invariant to a group of  transformations $\T_g:\R{n}\to\R{n}$ for $g=1,\dots,G$, such as rotations or translations acting on a discretized image of $n$ pixels.
Mathematically speaking, if the support of the distribution is invariant, we have that for every transform $\T_g$, if $\x\in \text{supp}(p_{\x})$, then $\T_g(\x) \in \text{supp}(p_{\x})$. This condition is less strict than asking the distribution to be invariant, i.e., $p(\T_g(\x))=p(\x)$ for all transformations $\T_g$ and signals $\x$. Due to the invariance property, we have the following key observation:
\begin{equation}
    \y = \A \x = \A \circ \T_g \circ \T_g^{-1} (\x) = \A \circ \T_g (\x')
\end{equation}
where $\x'=\T_g^{-1}(\x)$ also belongs to the signal set. Thus, the invariance property \emph{implicitly} provides us with additional virtual observations through a family of different operators $\{\A_g=\A \circ \T_g\}_{g=1}^G$ and we are in a similar but slightly more constrained setup to that in~\Cref{sec: multioperators} (see~\Cref{fig: ei intuition} for an illustration of this idea). We will see that this enables us to exploit a powerful property called equivariance.

\begin{figure}[t]
    \centering
    \includegraphics[width=.8\textwidth,alt={Two rows of images of cats with missing pixels in the center of the image. In the first row, all cats are centered in the image, and their faces are blocked by the mask, whereas in the second row, the cats appear at random positions in the image, and their face is not fully blocked by the mask.}]{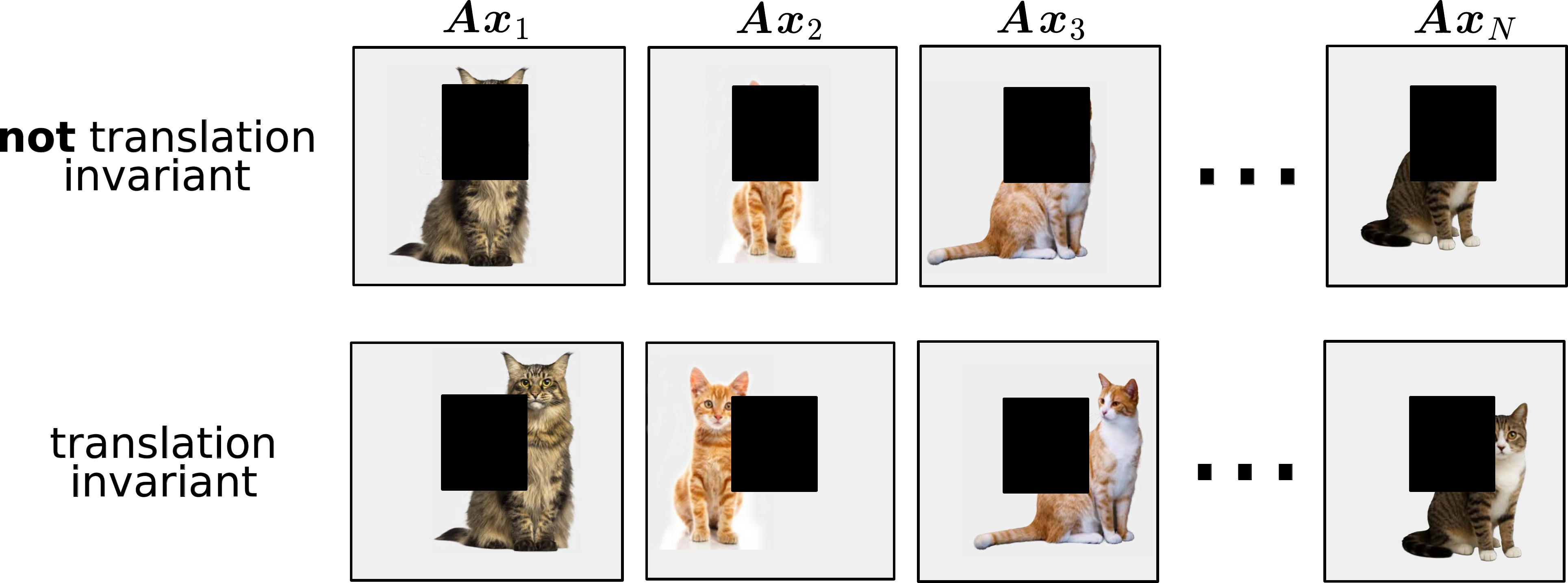}
    \caption{\JT{\textbf{Example of the equivariant imaging principle.} Consider a toy setting where we image cats through a fixed masking operator. If the cat distribution is not invariant to translations (top row), the cats have a canonical position in the image, and we never observe a part of the distribution (the heads). If the distribution is invariant (bottom row), the cats are not always centered, and we can learn the cat distribution (including their heads) despite never observing the masked pixels.} }
    \label{fig: ei intuition}
\end{figure}

\paragraph{Equivariance}
\rev{The concept of equivariance has been widely studied in machine learning, especially in the context of incorporating symmetries into neural network architectures~\cite{bronstein_geometric_2021}. Informally, a function is said to be equivariant to a group of transformations if applying a transformation to the input results in a corresponding transformation of the output. More formally, we have the following definition:}
\begin{definition}[Equivariance~\cite{cohen_steerable_2016}]
\label{def_equivariant_denoiser}
A function $\phi : {\R{n}} \rightarrow \R{m}$ is equivariant \JT{to a group action} if:
$\tilde{\T}_g (\phi(\x)) = \phi(\T_g (\x))$ for all \JT{$g=1,\dots,G$}, where $\T_g: {\R{n}} \rightarrow \R{n}$ and $\tilde{\T}_g: {\R{m}} \rightarrow \R{m}$ are (possibly the same if $n=m$) transformations satisfying the group properties.
\end{definition}
\rev{In the context of inverse problems, we cannot directly apply this definition to the reconstruction function $f$ due to its delicate interplay with the forward operator. We can use instead a more specific definition that takes into account this interplay:}
\begin{definition}[Equivariant reconstructor~\cite{sechaud2025equivariant}]\label{def: equi}
    \rev{We say that the $f(\y, \A)$ is an equivariant reconstructor if}
    \begin{equation} \label{eq:def_equiv_recon}
        f(\y,\A \T_g) = \T_g^{-1} f(\y,\A),\ \forall \y \in \R{m}, \forall g, \forall \A \in \R{m \times n}.
    \end{equation}
\end{definition}
\rev{Below we illustrate how popular reconstructors are equivariant as long as their building blocks are equivariant in the sense of~\Cref{def_equivariant_denoiser} or if the signal distribution $p_{\x}$ is invariant. Proofs can be found in~\cite{sechaud2025equivariant}.
\begin{enumerate}
    \item \textbf{Back-projection networks.} The backprojection network $f(\y,\A) = \phi(\A^{\dagger}\y)$ is an equivariant reconstructor if the image-to-image mapping $\phi$ is equivariant.
    \item \textbf{Unrolled network.} 
     The unrolled network defined in~\Cref{eq: unrolled} is an equivariant reconstructor for any stepsize $\tau \in \mathbb{R}$ if the proximal operators $\phi_1,\dots,\phi_k$ are equivariant.
    \item \textbf{Reynolds averaging.} Let $\tilde{f}(\y, \A)$, be any (non-equivariant) reconstructor, then
    \begin{equation} \label{eq:def_equiv_recon_reynolds}
        f(\y, \A) = \frac{1}{G}\sum_{g=1}^G \T_g \tilde{f}(\y, \A \T_g)
    \end{equation}
    is an equivariant reconstructor.
    \item
    \textbf{Variational methods.} The variational estimate
    \begin{equation} \label{eq:def_equiv_recon_map}
        f(\y, \A) = \argmin_{\x} \|\y - \A(\x)\|^2 + \rho(\x).
    \end{equation}
    is an equivariant reconstructor if the regularization term $\rho$ is an invariant distribution.
    In particular, for observations with isotropic Gaussian noise and $\rho(\x) = -\log p_{\x}(\x)$, then $f$ corresponds to the maximum-a-posteriori estimate, and is equivariant if $p_{\x}$ is an invariant distribution.
    \item \textbf{MMSE.} The MMSE estimate
    \begin{equation} \label{eq:def_equiv_recon_mmse}
        f(\y, \A) = \E{\x|\y,\A}{\x}
    \end{equation}
    is an equivariant reconstructor if $p_{\x}$ is an invariant distribution.
\end{enumerate}
The first three examples show how to build practical equivariant reconstructors, while the last two examples show that optimal Bayesian reconstructors are equivariant when the signal distribution is invariant. 
}

\rev{At this point, one might ask whether the equivariant reconstructor property in~\Cref{def: equi}, together with a simple measurement consistency loss~\Cref{eq: mcloss} is sufficient for learning with a single non-invertible operator $\A$. The answer is negative, since the simple linear pseudo-inverse $f(\y,\A)=\A^{\dagger}\y$ is an equivariant reconstructor that minimizes the measurement consistency loss, but does not recover any information in the nullspace of $\A$. As with the multi-operator case, we can use a splitting loss or enforce consistency across transforms to further constrain the solution space.}

\paragraph{\rev{Equivariant splitting loss}} \rev{
We can follow a similar approach to the \Cref{eq: multisplit} loss in the previous section, but this time using operators related by a transformation, resulting in the equivariant splitting loss~\cite{sechaud2025equivariant}:
\begin{align} \label{eq: loss ES} \tag{ESPLIT}
    \lossarg{ESPLIT}{\y} &= \frac{1}{G}\sum_{g=1}^G \lossarg{MSPLIT}{\y,\A \T_g} \\ 
    & =\frac{1}{G}\sum_{g=1}^G \E{\y_1, \A_1|\y, \A}{\| \A\T_g f(\y_1, \A_1\T_g) - \y\|^2} \nonumber 
\end{align}
where $\A_1 \sim p(\A_1|\A)$ is a random split of $\A$. As with \Cref{prop:split} in the multi-operator splitting case, minimizing this loss under the assumption of an invariant signal distribution $p_{\x}$ can recover the MMSE estimator, as long as $\Q_{\A_1}=\sum_{g\in \mathcal{S}_{\A_1}}\T_g^{\top}\A^{\top}\A\T_g$, where $\mathcal{S}_{\A_1}$ is comprised by all transforms for which $\A_1$ may arise as a random split of $\A \T_g$, or $\bar{\Q}_{\A} = \E{\A_1|\A}{\Q_{\A_1}}$ are full rank. Since summing over the whole group of transformations can be expensive, the loss can be evaluated by randomly sampling a transformation at each training iteration. As with the multi-operator splitting loss, the test time estimator can be computed by averaging over multiple splits and transformations as in~\Cref{eq: multi evals}, and a noise-robust   extension of the loss can be derived in a similar way by replacing the term enforcing consistency with the input $\y_1$ by \Cref{eq: GR2R} or \Cref{eq:sure}.
}

\rev{
If we choose $f$ to be an equivariant reconstructor by design, for example using a back-projection network with an equivariant denoiser architecture, the equivariant splitting loss in~\Cref{eq: loss ES} reduces to the simpler splitting loss in~\Cref{eq: multisplit} with a single operator, as stated in the following proposition:}
\begin{proposition}
\label{prop: es to split}
\rev{If $f$ is an equivariant reconstructor, then \Cref{eq: loss ES} is equivalent to the splitting loss}
\begin{equation} \label{eq:loss_reduces_to}
    \mathcal L_{\mathrm{ESPLIT}}(\y, f) = \mathcal L_{\mathrm{MSPLIT}}(\y, \A, f).
\end{equation}
\end{proposition}
\rev{
The proof follows directly from the definition of equivariant reconstructor in~\Cref{def: equi}.
This result shows how training on a simple splitting loss with a single operator can be effective even in the case of incomplete measurements, as long as the network is an equivariant reconstructor by construction.}
    
\paragraph{Consistency across transforms}
Following the same logic behind~\Cref{eq: moi}, if we assume that the imaging problem is approximately invertible, i.e.,  there exists a function, $f^*$, such that $\x \approx f^*(\A \x, \A)$ for all $\x \in \text{supp}(p_{\x})$, then, as illustrated in~\Cref{fig: equivariant system}, a good reconstruction function $f$ should be such that the \emph{full imaging/sensing system}, $f \circ \A$ is approximately equivariant
\begin{equation}
 f(\A\T_g(\x), \A)\approx \T_g f(\A\x, \A).
\end{equation}
The equivariant imaging (EI) loss aims at enforcing the system equivariance via training~\cite{chen_equivariant_2021}:
\begin{equation}\label{eq: ei} \tag{EI}
    \lossarg{EI}{\y} = \frac{1}{G}\sum_{g=1}^G \frac{1}{n}\|f(\A \T_gf(\y,\A), \A) - \T_g f(\y,\A)\|^2.
\end{equation}
\rev{The equivariant system condition is stronger than asking $f$ to be an equivariant reconstructor as in~\Cref{def: equi}, and in some cases will only hold exactly if the imaging problem is invertible\footnote{\rev{For example, the MMSE estimate under an invariant distribution $p_{\x}$ is always an equivariant reconstructor, but it only satisfies this stronger condition if the imaging problem is invertible, i.e., $\E{\x|\y,\A}{\x}=\x$.}}. Thus, the EI loss is necessary even when using equivariant reconstructors, as it further constrains the solution space.}

As with the MOI loss, this loss should be minimized together with measurement consistency:
\begin{equation}
    \lossarg{}{\y} = \lossarg{MC}{\y} + \lambda \, \lossarg{EI}{\y}
\end{equation}
where $\lambda>0$ is a hyperparameter and the $\mathcal{L_{\text{MC}}}$ can be replaced in the case of noisy measurements by any of the losses introduced in~\Cref{chap: noisy}, 
such as \Cref{eq:sure}~\cite{chen_robust_2022}.

\begin{figure}[t]
    \centering
    \includegraphics[width=.9\textwidth,alt={A diagram illustrating the lack of equivariance of an imaging system with a non-equivariant operator (an inpainting mask). The top workflow shows an image of a cat, which is masked by the operator and then reconstructed by a neural network. The bottom workflow shows the same steps on a translated version of the cat image. The end reconstructions of both workflows are not the same up to a translation.}]{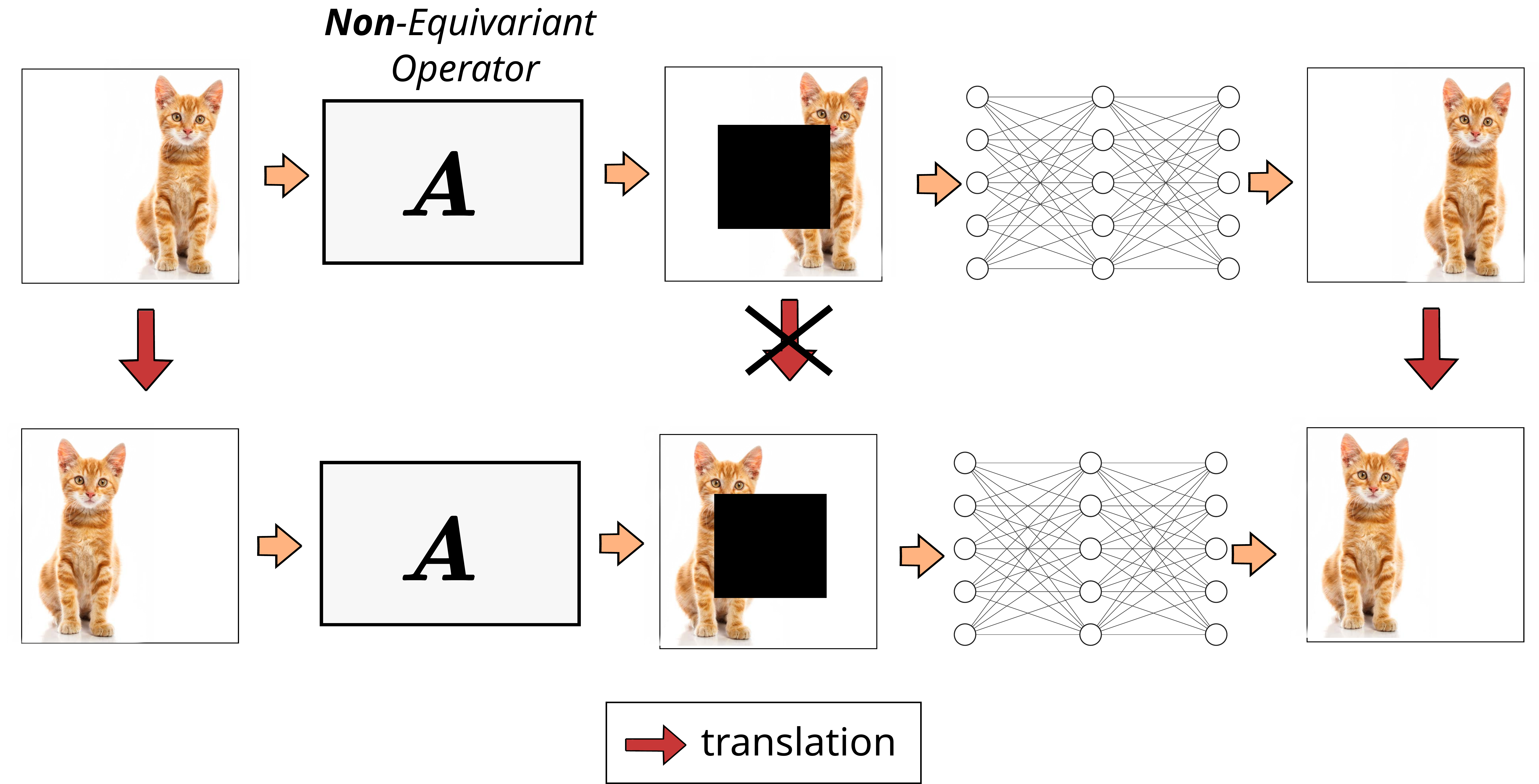}
    \caption{The equivariant imaging loss attempts to enforce equivariance of the imaging system to the a group of transformations (here translations). This loss can enable learning in the nullspace of the operator (here missing pixels), as the transformations implicitly \emph{translate} the nullspace.}
    \label{fig: equivariant system}
\end{figure}

\paragraph{Conditions on $\A$ with linear transformations} \JT{Typical transforms such as rotations or translations are linear operators, and we can think of them as a collection of invertible matrices $\{\T_g\}_{g=1}^G$. In this case, we can follow a similar analysis to the multiple operator setting in the previous section with operators defined as $\A_g=\A \T_g$.} As we saw in \Cref{sec: multioperators}, we need that $\A\T_g$ have different nullspaces, or in other words, that $\sum_{g} \T_g^{\top}\A^{\top}\A\T_g$ is an invertible matrix. An important consequence 
is that the forward operator should \emph{not} be equivariant to the group of transformations:
\begin{proposition}[Tachella et al.~\cite{tachella_sensing_2023}]
If $\A^{\top}\A$ is equivariant to the group of transformations $\{\T_g\}_{g=1}^G$, then all operators $\A \T_g$ share the same nullspace. 
\end{proposition}
Equivariance means that $\A^{\top}\A\T_g= \T_g \A^{\top}\A$ for all $g=1,\dots,G$, 
and therefore that $\A$ shares the same nullspace with $\A \T_g$ for all $g$.

As illustrated in the following examples, the requirement for $\A$ not being equivariant cannot be taken for granted, and in any given scenario the choice of group actions will depend on specific properties of the forward operator:
\begin{itemize}
    \item Compressive random operators $\A\sim\G{\boldsymbol{0}}{\frac{1}{m}\Id}$ with $m<n$ are not equivariant to any (non-trivial) set of transformations $\{\T_g\}_{g=1}^{G}$, except for the amplitude scalings $\T_g=g\Id$, with probability 1.
    \item Image inpainting $\A=\diag{\m}$ is a generally not equivariant to pixel shifts, as missing pixels have fixed locations in the image.
    \item Operators that admit a diagonalization $\A= \Q \,\diag{\m}\F$ where $\F$ is the Fourier transform and $\Q$ is any unitary basis, such as any blurring operation or MRI, are equivariant to pixel shifts or translations. However, if the blurs have a specific orientation or the MRI masks are accelerated using a Cartesian subsampling pattern, these operators are not equivariant to rotations.
    \item Isotropic blurs and downsampling with antialiasing filters are equivariant to both rotations and translations. However, they are not equivariant to scaling transformations which can be used to learn to reconstruct the missing high-frequencies~\cite{scanvic_self-supervised_2023}.  
\end{itemize}


While the EI loss does not carry with it a strong equivalence with a supervised loss, it shares the same motivation as MOI from the model identifiability theory discussed in \Cref{sec: model identification}.

\section{Learning generative models from incomplete measurements}

So far, we have presented multiple losses than can approximate the posterior mean of the problem. Here we go further and ask whether we can identify the signal distribution, $p_{\x}$, from incomplete measurement data alone.
In this section, we present the main approaches that have been explored so far, including generative adversarial networks (GANs) and diffusion models.

\paragraph{Generative adversarial networks}
In the unconditional generation setting, we aim to train a generator $f:\R{k}\mapsto \R{n}$ mapping latents $\z\in\R{k}$ following a simple distribution such as an isotropic Gaussian, to samples $\x\in\R{n}$ of the image distribution $p_{\x}$. We can aim to match distribution of measurements associated with the generated distribution $p_{\hat{\y}}$ defined as
\begin{equation} 
\left\{
\begin{array}{l}
    \z \sim \mathcal{N}(\boldsymbol{0},\Id ), \;
    \A \sim p(\A) \\ 
    \hat{\x} = f(\z)\\
    \hat{\y} \sim p(\y|\A\hat{\x}) \\ 
\end{array}
\right.
\end{equation}
to the distribution of observed measurements $p_{\y}$. AmbientGAN~\cite{bora_ambientgan_2018} proposes to train a Wasserstein GAN~\cite{gulrajani_improved_2017} to achieve this goal:
\begin{equation} \label{eq: ambientGAN}
    \min_{f} \max_{d} \; \E{\z,\A}{d(\A f(\z), \A)} - \E{\y,\A}{d(\y,\A)}
\end{equation}
where $d:\R{n} \times \R{m\times n} \mapsto \R{}$ is a 1-Lipschitz discriminator network trained jointly with the generator, which can incorporate information about the forward operator. The first term in~\Cref{eq: ambientGAN} pushes $f$ to generate realistic measurements by fooling the discriminator, whereas the second term trains $d$ to discriminate real measurements from generated ones.


This idea can be also extended to a conditional GAN model~\cite{cole_fast_2021,bendel_regularized_2023}, where the generator, $f(\z, \y, \A)$, is conditioned on a measurement and operator, and aims to generate posterior samples with varying latents $\z$.

\paragraph{Diffusion models}
AmbientDiffusion~\cite{daras_ambient_2024} extends the AmbientGAN idea to diffusion models, in the case of  \emph{noiseless} but incomplete measurements from multiple forward operators. In this setting, a reconstruction network is trained using the~\Cref{eq: multisplit} loss at different noise levels by adding synthetic Gaussian noise with standard deviation, $\sigma$, to the input measurements $\y_1$, thus approximating the conditional estimator $f(\y_1,\A_1, \sigma)=\E{\x|\y_1 + \sigma\epsilon, \A_1}{\x}$. 
Once the network is trained, samples of the signal distribution~\cite{daras_ambient_2024} or posterior~\cite{aali2025ambient} are obtained by fixing a random split $\A_1$, and using $f$ as a proxy for the Gaussian denoiser $\E{\x|\x + \sigma\epsilon}{\x}$ required to run the diffusion.

In a similar multioperator setting, Rozet el al.~\cite{rozet_learning_2024} propose a different approach based on expectation-maximization, which consists of iterating between i) generating posterior samples using a diffusion approach with a fixed denoiser network, and then ii) updating the denoiser on the generated samples. At initialization, the denoiser is initialized to sample from a Gaussian distribution.


\section{Model identification theory} \label{sec: model identification}




While we have seen that splitting losses can approximate the supervised $\ell_2$ loss under certain assumptions, it is important to ask if the harder task of learning a generative model is even well defined, or in other words, if we can uniquely identify $p_{\x}$ from incomplete measurements\footnote{We do not consider noise in this part, since we can take noise into account by reasoning in two steps: first we apply the identification results from noisy data in \Cref{prop:noise} to first identify the clean measurement distribution from the noisy one for each forward operator, and second, identify the signal distribution from the set of clean measurement distributions~\cite{tachella_sensing_2023}.}, without any additional assumptions on $p_{\x}$? The question can be answered by analyzing the available information about the characteristic function of the signal distribution, $\phi_{\x}(\bomega)$. For each operator $\A \sim p(\A)$, we have
\begin{align} \label{eq: char func arg}
    \phi_{\y|\A}(\tilde{\bomega}) &= \E{\y|\A}{e^{\mathrm{i}\tilde{\bomega}^{\top} \y }}
 = \E{\x}{e^{\mathrm{i}\tilde{\bomega}^{\top} \A \x }}
    = \E{\x}{e^{\mathrm{i}(\A^{\top}\tilde{\bomega})^{\top} \x }} \\
    &= \phi_{\x}(\bomega=\A^{\top}\tilde{\bomega}) \text { with } \tilde{\bomega}\in \R{m}
\end{align}
Thus, for each operator we observe the characteristic function of $\x$ in the range of $\A^{\top}$, which is an $m$-dimensional linear subspace of $\R{n}$ as $m<n$. Since a distribution is uniquely determined by its characteristic function, we need to observe infinitely-many operators in order to fully cover the characteristic function of $\x$! 
This observation dates back to the work by Cramer and Wold, which focuses in the case $m=1$:
\begin{theorem}[Cramér and Wold~\cite{cramer_theorems_1936}]
A probability distribution $p(\x)$ is uniquely determined by \emph{the totality} of its one-dimensional projections.
\end{theorem}
The theorem says that if we observe (unpaired) scalar measurements $y=\boldsymbol{a}^{\top}\x \in \R{}$ with $\x\sim p_{\x}(\x)$ and a measurement distribution $\boldsymbol{a}\sim p(\boldsymbol{a})$ that covers the whole space of possible projections (i.e., is dense in $\R{n}$), we can uniquely identify the $p_{\x}$.  Unfortunately, this result is not very practical, since it only holds in the limit of observing \emph{all possible} projections in $\R{n}$, whereas we generally only obtain observations via a finite number of operators, and thus the distribution is not dense in $\R{n}$ (e.g., is limited to varying masks in image inpainting or accelerated MRI).

\subsection{Identification of low-dimensional distributions}

In order to obtain sufficient conditions in the more realistic setting of a finite number of operators, we need to consider some additional assumptions on the signal distribution, $p_{\x}$. In the following, we will see that assuming that the signal distribution is low-dimensional, or in other words, that the support of $p_{\x}$ is a low-dimensional subset of $\R{n}$, is sufficient for obtaining model identification guarantees. Low-dimensionality is a common assumption in imaging and data science, and it is often referred to as the \emph{manifold hypothesis}~\cite{fefferman_reconstructing_1993}.

\begin{example}
We can illustrate why and how low-dimensionality can help by considering a simple example where $p_{\x}$ is supported on a $k$-dimensional subspace of dimension $k$, such that we can write any signal as $\x = \Q \z$ for some low-dimensional latent vector $\z\sim p_{\z}$ and a fixed linear decoder $\Q \in \R{n \times k}$.
Assuming that the linear decoder $\Q$ is known, we can observe the characteristic function of the latent variable as $\phi_{\y|\A}(\tilde{\bomega})=\phi_{\z}(\bomega=(\A\Q)^{\top}\tilde{\bomega})$ for all $\tilde{\bomega}\in \R{m}$ and $g=1,\dots,G$.
If we further assume that $m\geq k$ and $\rank{\A\Q} = k$ for some $\A$, we can follow a similar argument to that in~\Cref{eq: char func arg} to conclude that we can uniquely identify the latent distribution $p_{\z}$, and thus also uniquely identify $p_{\x}$ as the linear decoder is known.
\end{example}

The intuition of the linear subspace can be generalized to more general $k$-dimensional sets.
In the example, the two key steps for model identification are to i) identify the low-dimensional support of the distribution, which we denote as $\text{supp}(p_{\x}):=\signalset$ (in the linear case, the support is given by the range of the linear decoder $\Q$), and ii) require that $\A$ is one-to-one when restricted to $\signalset$ (in the linear case, this is equivalent to asking $\rank{\A\Q}=k$).  

In the noiseless measurements case, if there is a one-to-one reconstruction map $f$ between the measurement set, $\text{supp}(p_{\y})=\yset$, and the signal set, $\signalset$, then we can identify the signal distribution $p_{\x}$ by simply applying $f$ to all measurements in $\yset$. We thus focus on the problem of identifying the support $\signalset$ from the measurement distribution one.

While multiple definitions of low-dimensionality of a set exist~\cite{falconer2013fractal}, here we focus on the upper box-counting dimension\footnote{The box-counting dimension~\cite[Chapter~2]{falconer2013fractal} is defined for a compact subset $\mathcal S\subset\R{n}$ as
\begin{equation}
   \bdim{\mathcal S} = \lim \sup_{\epsilon\to 0}  \frac{\log \mathds{N}(S,\epsilon)}{-\log \epsilon}
\end{equation}
where $\mathds{N}(\mathcal S,\epsilon)$ is the minimum number of closed balls of radius $\epsilon$ with respect to  the norm $\|\cdot\|$ that are required to cover $\mathcal S$.}, which is convenient for the theoretical results, and covers both well behaved models such as compact manifolds where the definition coincides with the more intuitive topological dimension, as well as more general sets. The following theorem provides a sufficient condition for uniquely identifying the signal set from measurements associated to multiple forward operators (c.f.,~\Cref{sec: multioperators}):

\begin{theorem}[Tachella et al.~\cite{tachella_sensing_2023}] \label{theo: multiple op}
Assume that the signal set $\signalset$ is a bounded set with box-counting dimension $k$. For almost every set of $G$ operators $\A_1,\dots,\A_{G}\in \R{m\times n}$, the signal model $\signalset$ can be uniquely identified from the measurement sets $\{\A_g\signalset\}_{g=1}^{G}$ if the number of measurements per operator  verifies $m> k + n/G$.
\end{theorem}

It is worth noting that this result does not directly apply for \emph{any} set of $G$ operators (e.g., specific MRI or inpainting operators), but rather requires $G$ \emph{generic} operators, which removes degenerate cases. Nonetheless, it provides us with fundamental bounds on the number of measurements and operators needed to uniquely identify the signal distribution.

We can further refine this result for the more constrained case where the operators are linked by a group of transformations (c.f.,~\Cref{sec: equivariance}), i.e., $\{\A_g = \A \T_g\}_{g=1}^G$. The following theorem relies on the dimension of the largest linear subspace of $\R{n}$ which is invariant to the group of transformations\footnote{The maximum multiplicity of the group action $c_{\max}$, is given by the largest dimension of a linear subspace $\mathcal{S}\subseteq \R{n}$ such that $\T_{g}\mathcal{S} \subseteq \mathcal{S}$ for all $g=1,\dots,G$. See~\cite{tachella_sensing_2023} for more details.}, which we denote by $c_{\max}$.

\begin{figure}[t]
\centering
\includegraphics[width=.9\textwidth,alt={Two straight lines showing the necessary and sufficient conditions on the number of measurements, operators and dimensionality of the model, of model identification (on the left) and signal recovery (on the right). Different measurements regimes are highlighted in color, according to whether identification or recovery are impossible (red), might be possible (yellow) or impossible (red).}]{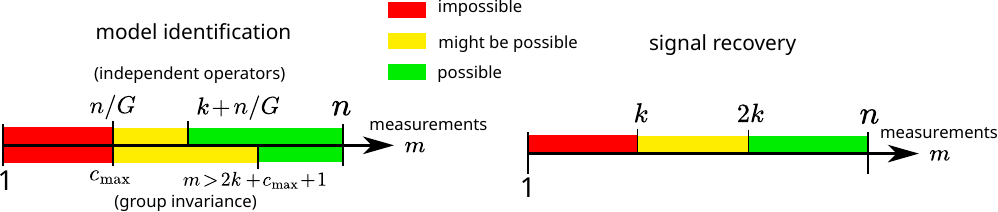}
\caption{Model identification and signal recovery regimes~\cite{tachella_sensing_2023} as a function of the number of partial observations $m$ per signal, model dimension $k$, ambient dimension $n$ and number of measurement operators $G$ (when it is possible to access multiple independent operators) or maximum multiplicity of an invariant subspace $c_{\max}$ (when the signal is group invariant or the operators are related via a group action).}
\label{fig:m regimes}
\end{figure}

\begin{theorem}[Tachella et al.~\cite{tachella_sensing_2023}] \label{theo: sufficient invariance}
Let $\{\T_g\}_{g=1}^{G}$ be a group of transformations associated with a compact cyclic group and assume that the signal set $\signalset$ is a bounded set with box-counting dimension $k$. For almost every $\A\in\R{m\times n}$, the signal set can be uniquely identified from the sets $\{\A \T_g\signalset\}_{g=1}^G$ if the number of measurements verifies $m > 2k + c_{\max} + 1$. 
\end{theorem}
Using the fact that for any compact cyclic group $c_{\max} \geq n/G$, when the equality holds we get the bound $m > 2k + n/G + 1$ which resembles that of \Cref{theo: multiple op}, although requiring $2k$ measurements instead of $k$. It turns out that these additional measurements are necessary in some cases, since it is possible to build a counter-example where identifying the support of $p_{\x}$ is impossible using any operator with $m \leq 2k + c_{\max} - 2$ measurements~\cite{tachella_sensing_2023}.

\paragraph{Comparison with signal recovery theory}
A well-known result from compressed sensing~\cite{puy_recipes_2017} and embedding theory~\cite{sauer_embedology_1991}, 
is that a sufficient condition for a generic $\A\in\R{m\times n}$ to be one-to-one on the support of the signal distribution $\signalset$ is to have a number of measurements larger than two times the dimension of the set, that is $m>2k$ where $k$ is the box-counting dimension of $\signalset$. These results are known as signal recovery theorems, since they specify the minimum number of measurements that guarantee the existence of a reconstruction function perfectly recovering all plausible signals. \Cref{fig:m regimes} compares the necessary and sufficient conditions for model identification presented in this chapter with those for unique signal recovery.

\paragraph{Relationship to splitting}
Building on~\Cref{prop:split}, we can also ask what are the minimum numbers of operators and measurements in order to uniquely identify the signal distribution via splitting losses and how these compare to the results in~\Cref{theo: multiple op}.

The best-case scenario of $G$ operators that verify the conditions in~\Cref{prop:split} can be constructed in the following way: consider operators given by $\A = [\A_1^{\top},\A_2^{\top}]^{\top} \in \R{m \times n}$ where $\A_1 \in \R{(2k+1) \times n}$ is fixed, and $\A_2\in\R{(m-2k+1)\times n}$ is sampled as one out of $G$ operators, i.e., $\A_2 \sim \frac{1}{G} \sum_{g=1}^G \delta_{\A_{2,g}}$, independently of $\A_1$. 
In order to learn the conditional estimator $\E{\x|\y_1,\A_1}{\x}$, \Cref{prop:split} requires that $\Q_{\A_1} = \E{\A|\A_1}{\A^{\top}\A} \in \R{n \times n}$ is invertible (and thus has full rank). Assuming that we are in the noiseless setting and that the signal distribution has dimension $k\leq n$, we can choose $2k+1$ measurements in $\A_1$, such that we have unique signal recovery~\cite{sauer_embedology_1991}. Thus, the conditional estimator obtains a perfect reconstruction, $\E{\x|\y_1,\A_2}{\x}=\x$. We can then uniquely identify $p_{\x}$ by simply reconstructing measurements from $p_{\y}$ once we have learned the conditional mean estimator.
Due to the simplified form of $\A$, we can compute $\Q_{\A_1}$ of~\Cref{prop:split} in closed form as 
\begin{align*}
    \Q_{\A_1} &= \E{\A|\A_1}{\A^{\top}\A} \\&=\E{\A_2}{\A_2^{\top}\A_2} + \E{\A|\A_1}{\A_1^{\top}\A_1}
    \\ &= \frac{1}{G}\sum_{g=1}^G \A_{2,g}^{\top}\A_{2,g} + \A_1^{\top}\A_1 .
\end{align*}
Since $\Q_{\A_1}$ is composed of the sum of $G$ matrices of rank at most $m-(2k+1)$ and one matrix of rank at most $2k+1$, we have that $\rank{\Q_{\A_1}} \leq \min \{G(m-2k-1) + 2k+1,n\}$ with equality for a generic choice of $\A_{2,g}$.

As we need $\rank{\Q_{\A_1}}=n$, it is necessary that $G(m-2k-1) + 2k+1 \geq n $. Thus we obtain the condition 
$$m \geq \frac{n}{G} + \left(1-\frac{1}{G}\right)(2k+1)$$
which is a similar, albeit less tight, version of the sufficient condition in~\Cref{theo: multiple op}.

\section{Summary}

Self-supervised learning in inverse problems where the forward operator is many-to-one is possible as long as the operators change across samples, or if we can assume that the signal distribution is invariant to a group of transformations, such as translations, rotations or scalings.

\Cref{tab: summary chap3} summarizes the assumptions behind the different families of self-supervised losses introduced in this chapter. The assumptions focus on the conditions on the forward operators rather than the noise, as all losses can be adapted to handle noise following the principles introduced in~\Cref{chap: noisy}. 

\begin{table}[h]
\centering
\begin{tabular}{|l|l|l|}
\hline
\textbf{Family} & \textbf{Assumptions} & \textbf{Refs.} \\ \hline
\Cref{eq: multisplit} & \begin{tabular}[c]{@{}l@{}}
\textbf{Necessary}\\ Multiple forward operators \\ 
$\E{\A}{\A^{\top}\A}$ is invertible \\ \hline
\textbf{Sufficient} (cf.~\Cref{prop:split})\\
$\E{\A_1|\A}{\E{\A|\A_1}{\A^{\top}\A}}$ is invertible \end{tabular}  & \cite{yaman_self-supervised_2020,daras_ambient_2024} \\ \hline
\Cref{eq: moi} & \begin{tabular}[c]{@{}l@{}}
\textbf{Necessary}\\
 Multiple forward operators \\ 
$\E{\A}{\A^{\top}\A}$ is invertible \\ 
$\exists f^*$ such that $\x \approx f^*(\A \x,\A)$ 
 \\  \hline  \textbf{Sufficient} (cf.~\Cref{theo: multiple op}) \\
$G$ generic  $\A$s  with $m > k + n/G$ \\
  \end{tabular}  & \cite{tachella_unsupervised_2022} \\ \hline
\Cref{eq: loss ES} & \begin{tabular}[c]{@{}l@{}}
\rev{\textbf{Necessary}}\\
\rev{Single forward operator $\A$} \\ \rev{$p(\x)$ invariant to transforms $\{\T_g\}_{g=1}^G$}  \\ 
\rev{$\sum_{g=1}^G \T_g\A^{\top}\A\T_g $ is invertible} \\ \hline
\rev{\textbf{Sufficient}}\\
\rev{$\E{\A_1|\A}{\sum_{g\in\mathcal{S}_{\A_1}}{\T_g^{\top}\A^{\top}\A\T_g}}$ is invertible} \\ \rev{where $\mathcal{S}_{\A_1} = \{ g : \A_1 \text{ is a split of }\A \T_g \}$} \end{tabular}  & \cite{sechaud2025equivariant} \\ \hline 
\Cref{eq: ei} & \begin{tabular}[c]{@{}l@{}} 
\textbf{Necessary}\\
    Single forward operator $\A$ \\ $\text{supp}(p_{\x})$ invariant to transforms $\{\T_g\}_{g=1}^G$  \\ 
$\sum_{g=1}^G \T_g\A^{\top}\A\T_g $ is invertible \\
$\exists f^*$ such that $\x \approx f^*(\A\x)$ \\ \hline
\textbf{Sufficient} (cf.~\Cref{theo: sufficient invariance})\\
Generic $\A$ with $m > 2k + n/G + 1$ \\ 
\end{tabular}  &  \begin{tabular}[c]{@{}l@{}} \cite{chen_equivariant_2021,chen_robust_2022} \\ \cite{scanvic_self-supervised_2023,wang_perspective-equivariant_2024} \end{tabular} \\ \hline
\end{tabular}
\caption{\textbf{Summary of losses for learning from incomplete measurements.} The first two losses rely on having measurements with multiple operators, whereas the last \rev{two assume} that the signal distribution is invariant to a group of transformations to obtain a set of virtual operators. 
} \label{tab: summary chap3}
\end{table}
\chapter{Finite dataset effects}
\label{chap: sample complexity}


So far we have seen how to build self-supervised losses that are unbiased estimators
of the (constrained or unconstrained) supervised loss in expectation. However, an important question remains, how good are the approximations with a finite number of samples?
In this chapter, we discuss existing answers, while noting that a full theoretical characterization of the sample complexity of self-supervised methods is not yet fully understood. We illustrate
the dependency of self-supervised methods on the dataset size in some practical scenarios, showing
that it typically scales similarly to the supervised setting. We also introduce some practical tools for dealing with finite datasets: applying the hold-out method to avoid under or overfitting, and starting from pretrained models to reduce the number of measurements required to obtain good performances.

\section{Hold-out method with self-supervised losses}\label{sec: hold out}

A standard practice in machine learning is to divide the dataset into non-overlapping training, validation and testing sets~\cite{hardt2022patterns}. The validation set plays a crucial role to avoid under or overfitting: if the validation loss remains always close to the training one, the model might not be expressive enough to fit all the data, and on the contrary, if the validation loss is bigger than the training one, the model is probably overfitting the training set. 

A similar practice can be done in the self-supervised setting, even if we do not have ground truth validation samples~\cite{batson_noise2self_2019,klug_analyzing_2023}. Since self-supervised losses serve as a proxy for the supervised loss, they can also be used on a validation set without ground truth to verify if the model is under or overfitting the data.

\begin{figure}[t]
\centering
\includegraphics[width=.8\textwidth,alt={Line plot showing self-supervised training and validation, and supervised test loss versus epochs (0 to 350) on a logarithmic scale. The y-axis represents loss values ($10^{0}$ to $10^{4}$). The blue solid line represents self-supervised training loss, the orange dashed line represents self-supervised validation loss, and the green line represents supervised test loss. Initially, all losses decrease and overlap. Around epoch 200, validation and test losses flatten while training loss continues to decrease, indicating overfitting. The region after epoch 200 is shaded red with the label  ‘‘overfitting’’.}]{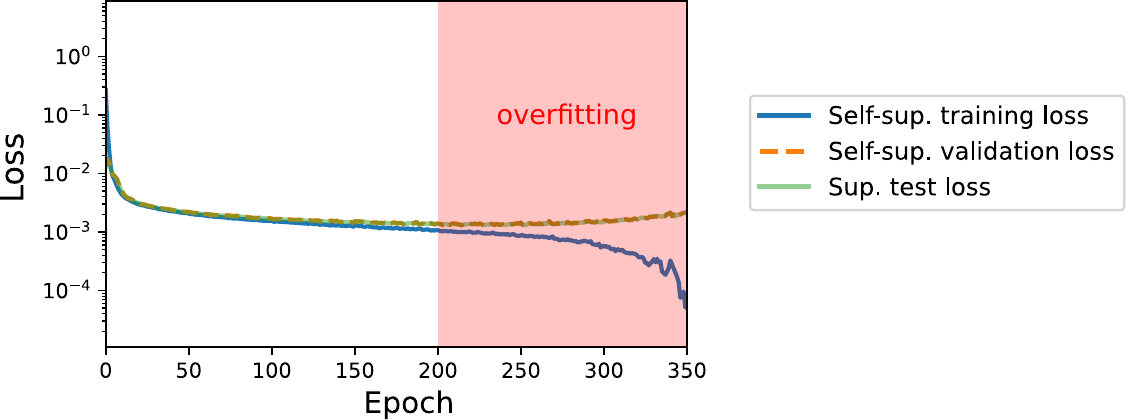}
\caption{Training, validation and test losses using a UNet network on a Gaussian denoising problem with $\sigma=0.1$. Training and validation are computed using the \Cref{eq:sure} loss presented in \Cref{chap: noisy}. Measurements are generated from the MNIST dataset, with 768 noisy images on the training set, 256 noisy images on the validation set, and 8192 (supervised pairs) of noisy and clean test images (not available in real-world settings). The self-supervised validation loss serves as a very accurate proxy for the supervised test error, and can be used to stop the training if the model starts overfitting the training data.} 
\label{fig: cross-validation}
\end{figure}

\Cref{fig: cross-validation} shows a self-supervised loss on the training and validation sets, and the supervised loss on the test set. The self-supervised validation loss tracks very well the performance on the test set, and can be used to stop the training when the model starts overfitting, i.e., when the gap between validation and training increases.

\section{Variance of the loss and its gradients}

A first step towards understanding how well self-supervised losses approximate the supervised counterparts is to study the variance of the losses and the variance of their gradients. A larger variance means that we will have a worse estimation of the supervised loss, leading to a decrease in performance in comparison with the supervised case.
\Cref{fig: variance grads} shows the average normalized mean squared error for loss and gradient estimates for \Cref{eq: noise2noise}
and \Cref{eq:sure} using a DRUNet denoiser~\cite{zhang2021plug} on $512\times 512$ patches corrupted by isotropic Gaussian noise using the Urban100 dataset. The experiment is repeated for a network with trained weights,
and one with randomly initialized weights. In both trained and untrained cases, the error with respect to the supervised loss is below $30\%$ for all noise levels. 
The gradient estimates are more accurate 
in the untrained model than the trained one. In the untrained case, all losses give errors of around $10\%$, as the 
loss is large, and the excess variance of self-supervised losses is negligible.
In the trained case, the loss is small, and the additional variance of self-supervised losses starts to play a role. 
Noise2Noise gives better estimates than SURE, as it relies on more information, i.e., two independent noisy copies of each image.

When averaging over $N$ noisy images, we should expect that the variances of all losses to decay as $1/N$ if the image samples are independent.
This can give us an idea of \emph{effective} sample complexity of a self-supervised method, by understanding how much larger $N$ needs to be to match the variance of the supervised dataset. In the case of $\sigma=.5$ in~\Cref{fig: variance grads}, we need approximately $\sqrt{10}\approx 3$ \JT{times more} noisy samples to obtain the variance with SURE compared to the supervised case.

\begin{figure}[h]
    \centering
    \includegraphics[width=.8\textwidth,alt={Four plots showing NMSE gradient and NMSE loss versus noise level sigma for trained and random models. For the trained model, SURE loss is larger than supervised and Noise2Noise losses, especially in gradients, while all losses increase with noise then plateau. For the random model, all losses overlap and change little with noise.}]{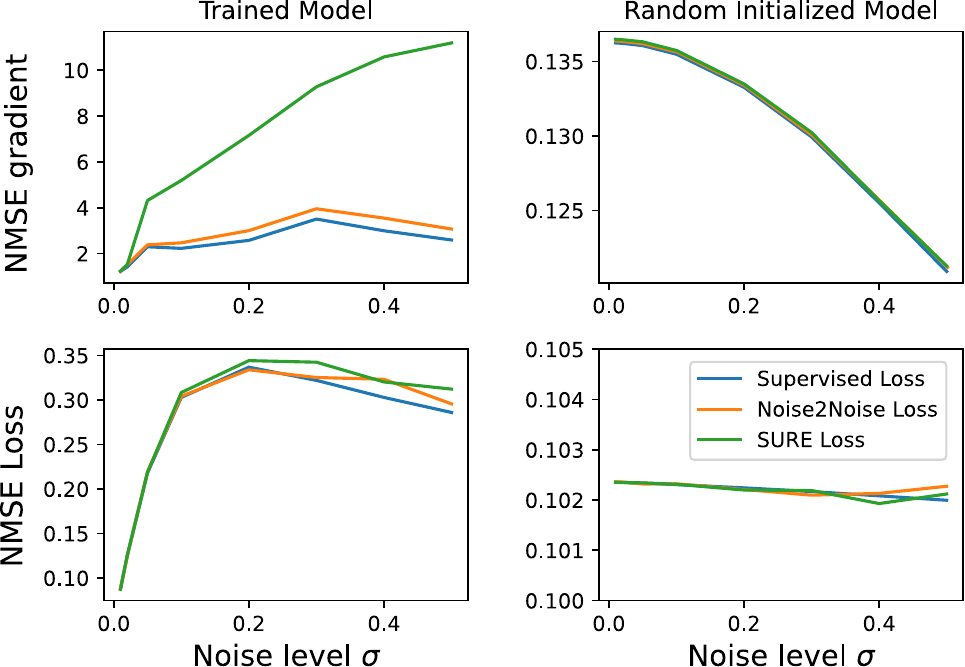}
    \caption{\textbf{Gradient approximation error of supervised and self-supervised losses.} Normalized mean squared error (NMSE) of the supervised, Noise2Noise 
    and Monte Carlo SURE losses, and of its gradients with respect to the network weights. 
    The experiments uses a DRUNet denoiser architecture evaluated on $512\times 512$ patches of the Urban100 dataset.}
    \label{fig: variance grads}
\end{figure}


\paragraph{Variance of the loss}

In the following analysis, we focus on the~\Cref{eq: noise2noise} loss for simplicity, but the intuition carries over to the other losses presented in the previous chapters. 
An in-depth analysis of the variance of the \Cref{eq:sure} loss can be found in~\cite{bellec_second-order_2021}.
Due to the independence of $\y_1$ and $\y_2$ conditional on $\x$, the variance of the Noise2Noise loss admits the following decomposition:
\begin{restatable}{proposition}{excessprop} \label{prop:excess variance}
Let $\y_1$ and $\y_2$ be two random independent random variables conditional on $\x$, such that $\E{\y_2|\x}{\y_2}=\x$. Then 
\begin{align*}
\V{\y_1,\y_2}{\lossarg{N2N}{\y_1,\y_2}} = \V{\x,\y_1}{\frac{1}{n}\| f(\y_1) - \x\|^2}  + \Delta 
\end{align*}
where the first term is the variance of the supervised loss and $\Delta$ is the additional variance with respect to the supervised case, which is given by 
\begin{align*}
  \Delta &=  \V{\x,\y_2}{ \|\y_2-\x\|^4}+ \frac{4}{n^2} \,  \E{\x}{\trace{\bPsi_{\x}\bSigma_{\x} }}
  \\ & -\frac{2}{n^2}\E{\x}{\E{\y_2|\x}{\|\y_2- \x\|^2(\y_2-\x)^{\top}}\Big(\E{\y_1|\x}{f(\y_1)}-\x\Big)}
\end{align*}
where $\bSigma_{\x} = \E{\y_2|\x}{(\y_2-\x)(\y_2-\x)^{\top}}$ is the covariance of the noisy target $\y_2$,
and $\bPsi_{\x} = \E{\y_1|\x}{(f(\y_1) - \x)(f(\y_1) - \x)^{\top}}$ is the error covariance for an image $\x$.
\end{restatable}
The proof is included in Appendix~\ref{app: proofs}. The last term in $\Delta$ is zero if the noise distribution is symmetric $\E{\y_2|\x}{\|\y_2- \x\|^2(\y_2-\x)}=\boldsymbol{0}$, or if the estimator $f$ is unbiased $\E{\y_1|\x}{f(\y_1)}=\x$ for all $\x$.
For example, in the simple case of targets with isotropic Gaussian noise, $\y_2 = \x + \sigma \bepsilon$ with $\bepsilon \sim \G{\boldsymbol{0}}{\Id}$
we have
$$
\Delta  = \frac{3}{n} \sigma^4  + \frac{4\sigma^2 }{n} \E{\x,\y_1}{\frac{1}{n}\|f(\y_1)-\x\|^2}
$$
which goes to zero as $n$ grows, as long as the mean squared error, $\E{\x,\y_1}{\frac{1}{n}\|f(\y_1)-\x\|^2}$, is approximately independent of $n$.

\paragraph{Variance of the gradients}
In order to compute the variance of the gradients, we need to consider a parameterization of the denoiser, $f_{\btheta}$, where $\btheta \in \R{p}$ are the trainable parameters of the denoiser (e.g., the network weights). 
The gradients could vary significantly even when the self-supervised loss has very small variance, if the denoiser is highly sensitive to changes in the parameters.
The gradients of the \Cref{eq: noise2noise} loss can be decomposed into two independent quantities as
\begin{align} \label{eq: additional var}
    \der{\mathcal{L}_{\text{N2N}}}{\btheta} (\y_1,\y_2,f_{\btheta}) &\propto \frac{1}{n}\der{f_{\btheta}}{\btheta} \left( f_{\btheta}(\y_1) - \y_2 \right) \\
    &\propto \underbrace{\frac{1}{n}\der{f_{\btheta}}{\btheta} \left( f_{\btheta}(\y_1) - \x \right)}_{\text{Supervised gradient}} - \underbrace{\frac{1}{n}\der{f_{\btheta}}{\btheta} \left(\y_2 - \x \right)}_{\text{Additional noise}}\label{eq: additional var2}
\end{align}
where $\der{f_{\btheta}}{\btheta}\in \R{p \times n}$ is the Jacobian of the denoiser evaluated at $\y_1$.
The first term corresponds to the gradient of the supervised loss and the second term comprises the additional randomness due to the use of a noisy target.
Since $\y_1$ and $\y_2$ are independent conditional on $\x$, the variance of the loss gradient is given by
\begin{align*} 
 &\V{\y_1,\y_2}{\| \der{\mathcal{L}_{\text{N2N}}}{\btheta} (\y_1,\y_2,f_{\btheta}) \|^2} \\
 &=\V{\y_1,\x}{\|\frac{1}{n}\der{f_{\btheta}}{\btheta} \left( f_{\btheta}(\y_1) - \x \right)\|^2} + \E{\y_1,\y_2,\x}{\|\frac{1}{n}\der{f_{\btheta}}{\btheta} \left(\y_2 - \x \right)\|^2} \\
 &= \underbrace{\V{\y_1,\x}{\|\frac{1}{n}\der{f_{\btheta}}{\btheta} \left( f_{\btheta}(\y_1) - \x \right)\|^2}}_{\text{Variance of supervised loss}} + \underbrace{\frac{1}{n^2}\E{\x}{\trace{\bSigma_{\x}\, \E{\y_1|\x}{\der{f_{\btheta}}{\btheta}^{\top}\der{f_{\btheta}}{\btheta}}}}}_{\text{Additional variance}}
\end{align*}
where the second line uses the decomposition in~\Cref{eq: additional var2} and that the additional noise has zero mean, and the third line defines $\bSigma_{\x} := \E{\y_2|\x}{(\y_2-\x)(\y_2-\x)^{\top}}$ as the covariance of the noisy target $\y_2$.
For example, in the case of targets with isotropic Gaussian noise, $\y_2 = \x + \sigma \bepsilon$ with $\bepsilon \sim \G{\boldsymbol{0}}{\Id}$
we have
$$
\text{Additional variance} = \sigma^2 \, \E{\y_1}{\|\frac{1}{n}\der{f_{\btheta}}{\btheta}\|_F^2}
$$
where the second term is the Frobenius norm of the Jacobian of the network with respect
to its parameters.
The additional variance introduced by the noisy target $\y_2$ depends on the noise level $\sigma^2$, and the average sensitivity of the network's output to changes in the input weights.

\section{Gap with supervised learning}
A key question when comparing supervised and self-supervised methods, is how self-supervised methods compare with supervised counterparts as a function of the amount of data we have for training, which amounts to computing the following gap:
\begin{equation*}  \label{eq: gap}
   \text{gap}(N) =  \min_{\btheta} \frac{1}{N}\sum_{i=1}^N \frac{1}{n}\|f_{\btheta}(\y_{1,i})- \y_{2,i}\|^2 -  \underbrace{\min_{\btheta} \E{\x,\y_1}{ \frac{1}{n} \|f_{\btheta}(\y_{1})- \x\|^2}}_{\approx\text{MMSE}}
\end{equation*}
Quantifying this gap is generally a difficult problem, since the results can be highly dependent on the data distribution, the parameterization of the estimator, and the specific learning algorithm used for estimating the parameters. Understanding the sample complexity of learning methods is an active area of research~\cite{hardt2022patterns} with many open questions~\cite{zhang2021understanding}. In particular, a challenging problem is to obtain meaningful bounds that are not highly dependent on the number of parameters of the model, which is typically very large in deep networks. For example, a line of work~\cite{hardt2016train,klug_analyzing_2023} studies the generalization error of stochastic gradient descent methods, which are the most popular optimization methods for training deep networks, obtaining bounds that are approximately independent of the parameter count.

\begin{figure}[t]
    \centering
    \includegraphics[width=.9\textwidth,alt={Four plots showing PSNR and MSE-MMSE gap versus dataset size N,  for noise levels $\sigma=0.1$ and $\sigma=0.2$. For both noise levels, PSNR improves with larger datasets, with supervised (blue) slightly outperforming Noise2Noise (orange) and SURE (green). MSE–MMSE gap decreases as dataset size grows, with supervised consistently achieving the lowest error but showing a similar dependency on $N$ than the self-supervised methods. A red dashed line indicates the theoretical $\sigma^2/\sqrt{N}$ trend.}]{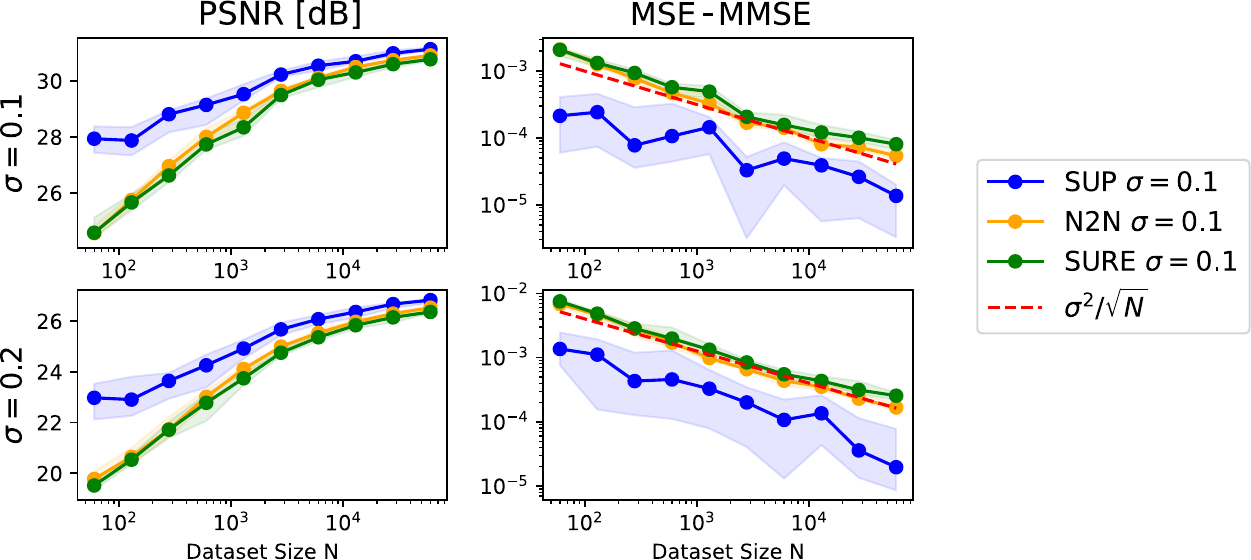}
    \caption{\textbf{Gap of Noise2Noise and SURE w.r.t. supervised learning as a function of dataset size for MNIST Gaussian denoising using a U-Net denoiser.} 
    On the left, we show the PSNR obtained by each learning method, and on the right we show the test mean squared error gap compared to the supervised baseline using the full dataset. The experiment is repeated 15 times for each dataset size $N$ and each of the two different noise levels $\sigma_1=.1$ and $\sigma_2=.2$, with shaded areas denoting the $90\%$ intervals across repetitions.  The optimality gap follows approximately $\sigma^2/\sqrt{N}$.}
    \label{fig: sample complexity}
\end{figure}

Here, we provide an empirical evaluation of the sample complexity, using the (self-supervised) hold-out method described in \Cref{sec: hold out} to obtain the best model for each dataset size.
\Cref{fig: sample complexity} shows the empirical gap for the networks trained with Noise2Noise or SURE on an MNIST Gaussian denoising problem for different dataset sizes. The observed gap follows the asymptotic behaviour $\text{gap}(N) \propto \sigma^2/\sqrt{N}$. How good is this rate? We can gain some intuition by comparing it with the setting of a trivial signal distribution consisting of a single signal: in this case, we could estimate the signal by simply averaging $N$ noisy realizations, and we would get the standard rate for estimating the mean of a Gaussian distribution of variance $\sigma^2$, that is $\sigma^2/N$. Thus, the cost of dealing with non-trivial signal distributions is a factor of $\sqrt{N}$.

\section{Fine-tuning and test-time adaptation}
While most self-supervised losses presented in this manuscript  approximate the supervised loss and can be used to train a reconstruction network from a random initialization without any ground truth references, they can also be used for fine-tuning a pretrained network on new measurement data, which might differ from the supervised dataset used for pretraining. This procedure is often referred to as test-time adaptation~\cite{mohan2021adaptive}. 

Starting from a pretrained model can significantly reduce the number of samples and training time needed to obtain good results compared with a randomly initialized model (e.g., often a couple of samples suffices). Moreover, fine-tuning can significantly improve the performance of the pretrained model on out-of-domain measurement data~\cite{terris_reconstruct_2025}.

\chapter{Extensions and open problems}\label{chap: perspectives}
The ideas of self-supervised learning are already making a considerable impact in imaging and sensing, with applications  emerging in MRI \cite{millard_theoretical_2023}, microscopy~\cite{bepler_topaz-denoise_2020}, and remote sensing~\cite{dalsasso_as_2022} to name but a few. 
In this final chapter, we review some ongoing work that is exploring extensions of the self-supervised learning framework for inverse problems and present some of the open problems in the field.

\section{Non-linear inverse problems}

Many real-world inverse problems are non-linear, such as quantized sensing~\cite{jacques_robust_2015}, phase retrieval~\cite{shechtman_phase_2015}, and non-linear problems associated to partial differential equations, such as the inverse scattering problem~\cite{soubies_efficient_2017}.
While, in principle, most of the self-supervised losses presented in~\Cref{chap: multioperators} can be applied with non-linear forward models, most of the theoretical analyses associated with these losses are restricted to the linear case and the development of a general theoretical framework for nonlinear operators is an open problem. 

Another important challenge in non-linear settings is that self-supervised losses involve the evaluation of the non-linear operator $\A$. This can both be computationally expensive, and lead to more complex loss landscapes including many local minima, compared to the supervised loss that does not involve the operator $\A$. We believe that these problems introduce new challenges, such as exploring relaxations of $\A$ in the self-supervised loss~\cite{tachella_learning_2023,sechaud_equivariance-based_2024} and problem-specific optimization algorithms.

Hu et al.~\cite{hu_spice_2022} show that the (multi-operator) splitting loss can be used in the context of accelerated MRI with unknown coil maps, which can be seen as a bilinear inverse problem.

The equivariant imaging approach has also been extended to declipping problems given by $\A(\x)=\eta(\x)$ where $\eta:\R{} \to [-1,1]$ is an elementwise clipping operator. In this case, learning beyond the clipping threshold is possible if we can assume that the signal model is invariance to amplitude scaling, $\T_g = g \Id$ with $g>0$, as the forward operator is not equivariant to these transformations~\cite{sechaud_equivariance-based_2024}.

The model identification theory in~\Cref{sec: model identification} has been extended to  quantized inverse problems~\cite{tachella_learning_2023}, in the extreme case where every measurement is quantized to a single bit, a problem that can be written as $\A_g(\x)=\text{sign}(\Q\x)$ where $\{\Q_g\in\R{m\times n}\}_{g=1}^G$ are linear operators. In this case, exact identification of the support of the signal distribution is impossible even if the set has low dimension $k\ll n$, however it is still possible to learn an approximation up to a global error of order $\mathcal{O}(\frac{k+n/G}{m}\log\frac{nm}{k+n/G})$.

\section{Towards large scale self-supervised imaging}

Deep learning imaging solutions rely on the substantial computing power offered by modern GPUs. Although this technology is advancing at a rapid pace, the current memory capacities of consumer GPUs limit the size of imaging inverse problems that can be handled. For example, this is the case for challenging high-dimensional medical imaging such as, extreme scale 3D or 4D (3D $+$ time) CT and MRI imaging~\cite{kellman_memory_efficient_imaging_2020,ong_extreme_2020,rudzusika_3d_2024}, as well as applications like ptychography in electron microscopy~\cite{ophus_4D-STEM_2019}. Such problems can generate as much as 10s of gigabytes of measurements and/or image data \emph{per reconstruction}, with the size of both image and measurements easily surpassing the memory capabilities of most GPUs.

Training deep learning solutions for such problems therefore faces additional computational complications. This is made all the more challenging when looking to use self-supervised learning techniques as the associated loss functions require the calculation of (and backpropagation through) the forward operator, $\A$, resulting in both large computation and memory usage in training and possible also in evaluation. The same issue occurs when training unrolled deep learning solutions in the supervised scenario, e.g.,~\cite{kellman_memory_efficient_imaging_2020,rudzusika_3d_2024}, where various approaches have been considered to mitigate these computational issues, e.g., forward or reverse recalculation, and gradient checkpointing for reduced storage during backpropagation, and data splitting methods which have traditionally been used in model-based image reconstruction~\cite{erdogan_ordered_subsets_1999,tang_practicality_of_stochastic_optimization_2020}.

Understanding the best approaches for tackling these issues in self-supervised learning has received much less attention and is a fruitful area for future research. One exception is~\cite{kosomaa_simulator-based_2023}, where the authors train an image reconstruction network for low-dose 3D helical CT in a self-supervised manner. Their core approach was to use a measurement splitting technique similar in spirit to the Noise2Inverse method~\cite{hendriksen_noise2inverse_2020} that only involved partial calculation of the forward operator at each iteration (they also implemented a range of other computational tricks such as gradient checkpointing, customised CUDA modules, etc.). 

While Kosoma et al.~\cite{kosomaa_simulator-based_2023} focused on an invertible imaging operator, it should be straightforward to extend these ideas to non-invertible mutliple operators using the related splitting ideas presented in \Cref{chap: multioperators}. Another interesting direction is exploring self-supervised learning solutions that explicitly use reconstruction networks based on stochastic optimization~\cite{tang_stochastic_primal_dual_deep_unrolling_2022}.

\section{Robust solutions and partially defined models}

Another practical issue that is important to address in real world inverse problems is the accuracy with which we can define the observation model. As George Box said “all models are wrong, but some are useful.” In any statistical learning or inference scenario it is important to capture as accurately as possible the underlying relationships between the variables, but also to ensure that the solutions take account of any unknown components of the forward model and are robust to any approximations/imperfections. This is particularly important in self-supervised learning, where the surrogates for the supervised loss rely heavily on the additional information from the inverse problem.

In \Cref{chap: noisy}, we have already seen instances of this that took account of partially defined noise models - constraining the class of estimators to be learned, e.g., \cite{krull_noise2void_2019,batson_noise2self_2019,tachella_unsure_2025}. This typically leads to sub-optimal solutions that nevertheless can outperform "optimal” solutions with misspecified assumptions. 

Extending these ideas to partially defined or misspecified forward models, either theoretically and algorithmically, is much more challenging and is an interesting direction for future research. For example, partially defined models include the case when there are unknown calibration parameters that must be estimated. One approach to this is to treat the unknown parameters as additional unknowns within the imaging problem that can either be estimated or marginalised out as part of the reconstruction process, e.g. self-supervised estimation of coil sensitivity maps in MRI~\cite{hu_spice_2022}, or self-supervised blind deblurring solutions~\cite{xia2019training}.

A common source of approximation within the forward model is the typical digital representation of the continuous image as a finite number of pixels/voxels. Acquisition systems are often designed to avoid introducing aliasing within such digital representations which in turn can induce correlation within the image noise which in the self-supervised learning setting must be treated with care, e.g., in SAR processing~\cite{dalsasso_as_2022}.

\section{Beyond the $\ell_2$ loss} 
Most of the losses presented in this manuscript aim at approximating the supervised $\ell_2$ loss. This loss has the benefit of having a simple decomposition 
\begin{equation*}
    \|f(\y) - \x\|^2 = \|f(\y) - \y\|^2 - 2 \, f(\y)^{\top}(\y - \x) + \const 
\end{equation*}
 where the second term captures the difference between the supervised case with simple measurement consistency. As we have seen in~\Cref{chap: noisy}, this term can be handled using independent noise realizations as in~\Cref{eq: noise2noise} and~\Cref{eq: R2R}, Stein's lemma as in~\Cref{eq:sure}, or blind-spot networks as in~\Cref{eq:cross}. While similar expressions also hold for Bregman divergences~\cite{efron_estimation_2004,oliveira_unbiased_2023}, this decomposition does not apply on other popular losses, such as the $\ell_1$ or $\ell_0$ loss.

In Noise2Noise~\cite{lehtinen_noise2noise_2018}, general $\ell_p$ losses are proposed for handling noise distributions with non zero-mean noise, such as the $\ell_1$ loss for random text removal or the $\ell_0$ loss for salt-and-pepper noise. Recalling that Noise2Noise relies on independent noisy pairs $(\y_1,\y_2)$, training on an $\ell_0$ loss leads (in expectation) to the mode estimator 
$$\text{Mode}\{\y_2|\y_1\} = \argmax_{\y_2} p(\y_2|\y_1).$$
In general, $\text{Mode}\{\y_2|\y_1\} \neq \text{Mode}\{\y_2|\x\}$, even if $\y_2$ and $\y_1$ are independent given $\x$. However, under the assumption that the (posterior) distribution of $\x$ given $\y_1$ is heavily concentrated, we have $p(\y_2|\y_1) \approx p(\y_2|\x)$, and thus $\text{Mode}\{\y_2|\y_1\} \approx \text{Mode}\{\y_2|\x\} = \x $. A similar argument holds for the $\ell_1$ loss and noise distributions whose median is given by $\x$.
While this approximation provides good empirical results~\cite{lehtinen_noise2noise_2018}, a better understanding under which conditions such approximations are reasonable, and determining how to extend \Cref{eq:sure} and other noise distribution aware self-supervised losses to general $\ell_p$ losses, are interesting directions of future research.

\section{Uncertainty quantification and generative modelling}
Most of the focus in this review has been on learning a good reconstruction mapping or denoiser, often targeting the conditional mean of the inverse problem, $\E{\x|\y}{\x}$. However, in many scenarios it is also important to quantify the uncertainty for a given estimator, so that the image estimate can be used with confidence in downstream analysis. 

Most self-supervised losses presented in this manuscript serve as estimators of the supervised loss, and can thus be used to estimate reconstruction errors at test time. For example, \Cref{eq:sure} or \Cref{eq: R2R} can be used to estimate the mean squared error, and higher-order extensions such as SURE for SURE~\cite{bellec_second-order_2021} can be used to quantify the uncertainty of this error estimate. It is also possible to use extensions of \Cref{eq: tweedie} to high order moments of the posterior distribution. For example, in the Gaussian denoising scenario, Manor and Michaeli~\cite{manor_posterior_2024} use these extensions to estimate the principal components for the covariance of the posterior distribution, $p(\x|\y)$, directly from the MMSE estimator (which itself can be estimated in a self-supervised manner). This nicely augments the estimated denoised images with the principal directions of uncertainty.
When dealing with incomplete data, \Cref{eq: ei} can be used to quantify the reconstruction error in the nullspace of the forward operator~\cite{tachella_equivariant_2024}.

The extent to which these ideas could be applied to other self-supervised learning solutions in this review is an interesting open problem.


\paragraph{Generative models}
While predicted error covariance of point estimates provide an efficient and compact characterization of the uncertainty, in certain areas of imaging science it is desirable to be able to explore the full posterior distribution of the imaging problem for downstream analysis, e.g., to characterize plausible solutions when ambiguities exist, or to test statistical hypotheses. 

A popular machine learning solution in such circumstances is to learn a generative model, such as VAEs~\cite{kingma_vae_2014}, GANs~\cite{goodfellow_gan_2014} or diffusion models~\cite{daras_survey_2024} that can act as a stochastic simulator and provide samples from the posterior distribution, $p(\x|\y)$. As such, generative models have become a popular approach for solving imaging inverse problems. However, in general such solutions currently rely on a \emph{pre-trained} generative models that have been trained on existing ground truth data.  An interesting research direction is therefore to what extent generative models can be learned in a purely self-supervised manner. Some progress has already been made on this. For example, as discussed in~\Cref{chap: noisy,chap: multioperators}, Prakash et al.~\cite{prakash_fully_2020,prakash_interpretable_2021} have developed VAEs for the Gaussian denoising problem, while GANs~\cite{bora_ambientgan_2018} and diffusion models~\cite{daras_ambient_2024,rozet_learning_2024} have been proposed that can be trained from noiseless but incomplete measurements. However, these methods generally rely on low-noise measurements from multiple operators, and it remains an open question as to whether generative models could be trained with noisy measurements taken from a single ill-posed measurement operator in a similar manner to~\Cref{eq: ei}. 




\section{Sample complexity}

Most of the theoretical analysis of self-supervised learning imaging solutions are either geometric~\cite{tachella_learning_2023} or focus on the asymptotic properties of the learning problem considering the behaviour of expected values. However, these do not indicate how hard the problem is statistically in terms of the number of measurement training samples required to achieve a good solution. This essentially comes down to how accurately we can approximate the expected risk from the empirical risk. 

We discussed this briefly in \Cref{chap: sample complexity} with respect to Noise2Noise and showed empirically that the gap between supervised learning and the equivalent (in expectation) self-supervised learning strategy scales approximately as $\mbox{gap}(N) \propto \sigma^2/\sqrt{N}$. However, analyzing sample complexity, even in the supervised case, is a challenging problem~\cite{hardt2022patterns}. Most existing results typically make significant simplifying assumptions, such as that the learning problem is convex~\cite{hardt2022patterns}, or that the reconstruction estimator is linear~\cite{klug_analyzing_2023}. There is also the question of how the sample complexity behaves as a function of the image size. Here, we have reason to be optimistic that we may be able to benefit from a blessing of dimensionality~\cite{donoho_curses_and_blessings_2000} associated with the typical high dimensionality of images, in the similar manner to how we can get accurate estimates of the SURE loss using only a single Monte Carlo sample~\cite{ramani_monte-carlo_2008}.

Understanding the nature of this supervised-self-supervised gap would help imaging practitioners to understand whether it is better to try to collect a small amount of supervised training data or whether the same result can be achieved through using a larger collection of measurement data that is usually much easier to acquire.



\section{Choosing the right self-supervised method}

In this article we have covered various self-supervised techniques offering capabilities ranging from denoising to solving ill-posed imaging inverse problems. We have also seen that the same problem can sometimes be solved through judicious choice of network architecture (but avoiding appealing to more nebulous architectural inductive biases) or through a cleverly designed loss function. However, we have refrained from explicitly promoting one solution over another. This may well leave the practitioner somewhat frustrated with a lack of guidance as to what is the right solution for a given task. 

Depending on the scenario various forms of information may be available to the practitioner, e.g., in terms of the nature of and information about the measurement noise, or the number and type of measurement operators available for creating training data. In some instances this naturally selects subsets of relevant algorithms and techniques and these have been highlighted in the tables at the end of \Cref{chap: noisy,chap: multioperators}. However, it is important to also note that not all relevant algorithms make the same use of the available information. For example, if presented with a problem where pairs of noisy realizations of the same signal are available for training and testing, one might be tempted to naturally reach for the Noise2Noise algorithms. However, if one knows something about the statistical noise model, even if this is only partial, it may be better to simply average the multiple realizations (thereby gaining 3~dB of SNR) and applying a different technique that incorporates additional statistical information\footnote{For example, Tachella et al.~\cite{tachella_unsure_2025} show that one can obtain better performance using UNSURE than Noise2Noise in a cryogenic electron microscopy denoising, where the noise model is approximately known.}.

The fact that many of the reviewed techniques exploit different information also opens up the opportunity to explore hybrid approaches. For instance, in the example above, rather than simply averaging the noise image pairs and then applying a single image technique that incorporated further knowledge of the noise model, one could explicitly construct a version that can exploit pairs of images\footnote{Note having access to pairs of noisy images immediately provides an estimate for the SNR of the signal.}. Similarly, one is not restricted to incorporating invariance properties into the inverse problem only when there is a single measurement operator. Indeed, empirical evidence suggests that exploiting multiple sources of information within the measurements tends to only make things better~\cite{wang_benchmarking_ssl_2025}.

In many real world settings, the most challenging aspect for the practitioner is to know the accuracy of the underlying forward model that is being exploited to enable self-supervised learning, as we have discussed above. This is extremely important and should probably be the first thing for the practitioner to consider when selecting the right algorithm as better defined forward models typically offer better performance but at a price of being more sensitive to any misspecifications.

\section*{Acknowledgements}
The authors are grateful for discussions with Dongdong Chen, Laurent Jacques, Marcelo Pereyra, Lo\"{i}c Denis, Andrew Wang, Victor Sechaud, Jérémy Scanvic, César Caiafa, Thomas Davies, and Brett Levac. J. Tachella acknowledges support by the French National Research Agency (ANR-23-CE23-0013). The authors thank ENS Lyon for supporting the research visit of Mike Davies as invited professor.

\appendix
\chapter{Noisier2Noise and R2R equivalence} \label{app: noisier2noise}

In this appendix, we show the asymptotic equivalence between the Noisier2Noise~\cite{moran_noisier2noise_2020} and Recorrupted2Recorrupted~\cite{pang_recorrupted--recorrupted_2021}  losses for the Gaussian denoising case, that is $\y=\x+\epsilon$ with $\bepsilon \sim \G{\boldsymbol{0}}{\bSigma}$.
Noisier2Noise proposes to train a network 
\begin{equation}
    \loss{Noisier2Noise} = \E{\y_1|\y}{\|f(\y_1)-\y\|^2}
\end{equation}
where $\y_1 = \y + \tau \bomega $ with $\bomega \sim \G{\boldsymbol{0}}{\bSigma}$ and $\tau>0$. The minimizer of this loss in expectation is 
\begin{align*}
    f^{*}(\y_1) &= \E{\y|\y_1}{\y}  \\ 
    &= \E{\x,\bepsilon|\y_1}{ \frac{\tau}{1+\tau} \x +\frac{1}{1+\tau}\x + \bepsilon) } \\
    &= \frac{\tau}{1+\tau} \E{\x|\y_1}{\x} + \frac{1}{1+\tau}\Big(\E{\x|\y_1}{\x} + \E{\bepsilon|\y_1}{\bepsilon} + \frac{1}{1+\tau} \E{\bepsilon|\y_1}{\bepsilon}\Big) \\
    &= \frac{\tau}{1+\tau}\E{\x|\y_1}{\x} + \frac{1}{1+\tau}\Big(\E{\x|\y_1}{\x} + \E{\bepsilon|\y_1}{\bepsilon} + \tau\E{\bomega|\y_1}{\bomega}\Big) \\
    & = \frac{\tau}{1+\tau}\E{\x|\y_1}{\x} + \frac{1}{1+\tau}\E{\x|\y_1}{\y_1} \\
    & = \frac{\tau}{1+\tau}\E{\x|\y_1}{\x} + \frac{1}{1+\tau}\y_1
\end{align*}
where the third line uses that $\E{\bepsilon|\y_1}{\x}= \E{\bomega|\y_1}{\bomega}$ since $\bepsilon$ and $\bomega$ are iid. This requires knowing the distribution of the noise for this result to hold.
\begin{align} 
    f^{\text{test}}(\y_1) &=  \frac{1+\tau}{\tau} f^{*}(\y_1) -  \frac{1}{\tau} \y_1 \\
    &= \E{\x|\y_1}{\x} \label{eq: noisier2noise min}
\end{align}
The  \Cref{eq: R2R} loss is defined as 
\begin{equation}
    \loss{R2R} = \E{\y_1|\y}{\|f(\y_1)-\y_2\|^2}
\end{equation}
where $\y_1 = \y + \tau \bomega $ and $\y_2 = \y - \frac{1}{\tau}\bomega $, with $\bomega \sim \G{\boldsymbol{0}}{\bSigma}$ and $\tau>0$. As we show in~\Cref{chap: noisy}, this loss is an unbiased estimator of the supervised $\ell_2$ loss with input $\y_1$, and thus its minimizer is $f^{*}(\y_1) = \E{\x|\y_1}{\x}$ which is the same as the Noisier2Noise test time function in \Cref{eq: noisier2noise min}.

\chapter{Identification of moments} \label{app: moments}
%

Recovering the signal distribution, $p_{\x}$, from the measurement distribution, $p_{\y}$, can be seen as an inverse problem in infinite dimensions defined by the forward problem
\begin{align} \label{eq: inf inv}
   p_{\y}(\y) = \int_{\x\in\signalset} p(\y|\x) p_{\x}(\x) d\x.
\end{align}
Identifying $p_{\x}$ from $p_{\y}$ is possible if the noise distribution has a nowhere zero characteristic function (see \Cref{sec: learning from noisy py}), and, in the case of incomplete observations from multiple operators, if the support of $p_{\x}$ is low-dimensional (see~\Cref{sec: model identification}).

However, in some cases, we might not be able to identify the full distribution $p_{\x}$, but we can still find some of its higher-order moments, or moments of the posterior distribution $p(\x|\y)$.

\begin{example}
Assume a Bernouilli noise model $\y|\x\sim \text{Ber}(\x)$ where
$p(\y|\x)=\prod_{i=1}^{n} x_i^{y_i}(1-x_i)^{1-y_i}$. Since measurements are binary, we identify at most $2^n-1$ different moments of $p_{\x}$ associated with all possible inputs of $p_{\y}$, which are given by
$\E{\x}{\prod_{i\in \mathcal{I}} x_i}$ where $\mathcal{I}$ is an arbitrary choice of indices in $\{1,\dots,n\}$.
\end{example}

Since we can compute any moment of $p_{\y}$, using the law of total expectation we have that
\begin{align} \label{eq: g(y) moments}
    \E{\y}{g(\y)} &= \E{\x}{\E{\y|\x}{g(\y)}} \\
    &= \E{\x}{r(\x)}
\end{align}
where we defined $r(\x):=\E{\y|\x}{g(\y)}$ for some function $g:\R{m}\mapsto \R{}$. 

In the case where $p_{\y}$ is continuous and differentiable, we can compute moments of $p_{\x}$ by differentiating \Cref{eq: inf inv} as
\begin{align} \label{eq: r(x) moments}
   \nder{p(\y)}{y_i}{k} &= \E{\x}{\nder{p(\y|\x)}{y_i}{k}} \\
   &= \E{\x}{\tilde{r}_{i,k}(\x,\y)}
\end{align}
for any $k\geq 0$ and $i=1,\dots,n$, where we defined $\tilde{r}_{i,k}(\x,\y):=\nder{p(\y|\x)}{y_i}{k}$.

\begin{example}
Consider a Gaussian noise model where we have
$p(\y|\x)=\frac{1}{(2\pi\sigma^2)^{n/2}}e^{-\frac{\|\x-\y\|^2}{2\sigma^2}}$, such that
$\tilde{r}_{i,1}(\x,\y)= \frac{y_i-x_i}{\sigma^2}p(\y|\x)$. 
Using this result, we obtain
\begin{align}
    \E{\x}{\frac{y_i-x_i}{\sigma^2}p(\y|\x)} &= \frac{y_i}{\sigma^2} p(\y)-\frac{1}{\sigma^2}\E{\x}{x_ip(\y|\x)}  \\
   \der{p(\y)}{y_i} &= \frac{y_i}{\sigma^2} p(\y)-\frac{1}{\sigma^2}\E{\x|\y}{x_i} p(\y) \\
      \sigma^2 \der{\log p(\y)}{y_i} &= y_i - \E{\x|\y}{x_i} 
\end{align}
which is the well-known \Cref{eq: tweedie} formula, $\mathbb{E}\{\x|\y\} = \y + \sigma^2 \score$, the same formula that we derived
in \Cref{subsec: sure} as the minimizer (in expectation) of the \Cref{eq:sure} loss. Higher order derivatives can be used
to estimate higher posterior moments~\cite{manor_posterior_2024}, such as the posterior variance which is
equivalent to the minimum mean square error:
$$\text{MMSE} = \sigma^2 \left( \frac{1}{n}\sum_{i=1}^n \nder{\log p_{\y}}{y_i}{2}(\y)\right).$$

\end{example}

    
\chapter{Additional proofs} \label{app: proofs}

\begin{proposition} \label{prop: minimizer l2 loss}
Let $\x \in \R{n}$ and $\y \in \R{n}$ be two random variables following the joint distribution $p(\x,\y)$. The minimizer of the following $\ell_2$ loss 
\begin{equation}
    f^{*}(\y) = \argmin_f \; \E{\x,\y}{\|f(\y) - \x\|^2} 
\end{equation}
is given by 
\begin{equation}
    f^{*}(\y) = \E{\x|\y}{\x}.
\end{equation}
\end{proposition}
\begin{proof}
Letting $\mathcal{L}(f)=\E{\x,\y}{\|f(\y) - \x\|^2}$, we have that
\begin{align*}
    &\mathcal{L}(f)= \; \E{\x,\y}{\|\left(f(\y) - f^{*}(\y)\right)- \left(\x-f^{*}(\y)\right)\|^2} \\
    & \propto \; \E{\x,\y}{\|f(\y) - f^{*}(\y)\|^2}  - 2 \, \E{\x,\y}{\left(f(\y) - f^{*}(\y)\right)^{\top}\left(\x-f^{*}(\y)\right)}\\
    & = \;  \E{\x,\y}{\|f(\y) - f^{*}(\y)\|^2}  - 2 \, \E{\y}{\E{\x|\y}{\left(f(\y) - f^{*}(\y)\right)^{\top}\left(\x-f^{*}(\y)\right)}}\\
    & = \; \E{\x,\y}{\|f(\y) - f^{*}(\y)\|^2}  - 2 \, \E{\y}{\left(f(\y) - f^{*}(\y)\right)^{\top}\E{\x|\y}{\x-f^{*}(\y)}} \\
     &= \; \E{\x,\y}{\|f(\y) - f^{*}(\y)\|^2} 
\end{align*}
where the fourth line uses the fact that $\E{\x|\y}{\x-f^{*}(\y)}=0$. Thus, the global minimizer of $\mathcal{L}(f)$ is $f(\y)=f^{*}(\y)=\E{\x|\y}{\x}$.
\end{proof}

\begin{proposition} \label{prop: minimizer weighted l2 loss}
Let $\x \in \R{n}$ and $\y \in \R{m}$ be two random variables following the joint distribution $p(\x,\y)$, and let $\A \in \R{m\times n}$ be a linear operator. The minimizer of the following weighted $\ell_2$ loss 
\begin{equation}
    f^{*}(\y) \in \argmin_f \E{\x,\y}{\|\A f(\y) - \A\x\|^2} 
\end{equation}
is given by 
\begin{equation} \label{eq: A_dagger decomp f}
    f^{*}(\y) =\A^{\dagger}\A  \; \E{\x|\y}{\x} + (\Id - \A^{\dagger}\A) v(\y)
\end{equation}
where $\A^{\dagger}$ is the linear pseudoinverse of $\A$, $\A^{\dagger}\A$ is the projection into the range space of $\A^{\top}$, and $v:\R{n}\mapsto\R{n}$ is any function. 
\end{proposition}
\begin{proof}
Defining $\tilde{\x}=\A \x $ and $\tilde{f} = \A \circ f$, we can apply~\Cref{prop: minimizer l2 loss} to conclude that $\tilde{f}^{*}(\y)=\E{\tilde{\x}|\y}{\tilde{\x}}$, or equivalently that $\A f^{*}(\y)=\E{\x|\y}{\A\x}$. Applying the linear pseudoinverse of $\A$ on both sides, we obtain the desired equality in~\Cref{eq: A_dagger decomp f}.
\end{proof}

\excessprop*

\begin{proof}
Defining $a_{\x,\y_1} = \frac{1}{n}\|f(\y_1)- \x\|^2$, $b_{\x,\y_2}= \frac{1}{n}\|\y_2- \x\|^2$ and $c_{\x,\y_1,\y_2} = -\frac{2}{n} (f(\y_1)-\x)^{\top}(\y_2-\x)$ we have
\begin{align*}
\V{\y_1,\y_2,\x}{\frac{1}{n}\| f(\y_1) - \y_2\|^2} &= \V{\y_1,\y_2,\x}{a_{\x,\y_1}+ b_{\x,\y_2} + c_{\x,\y_1,\y_2}}
\\ &=  \V{}{a_{\x,\y_1}} + \Delta
\end{align*}
with
\begin{align*}
\Delta &= \V{}{b_{\x,\y_2}} + \V{}{c_{\x,\y_1,\y_2}} 
\\ &  + 2\, \E{\x}{ \Cov{a_{\x,\y_1},b_{\x,\y_2}}  + \, \Cov{a_{\x,\y_1},c_{\x,\y_1,\y_2}} 
+ \, \Cov{b_{\x,\y_2},c_{\x,\y_1,\y_2}} }.
\end{align*}
where $\Cov{\cdot,\cdot}$ denotes the covariance between two one-dimensional random variables with respect to $p(\y_1,\y_2|\x)$.

The first variance term is simply $\V{}{b_{\x,\y_2}}=\V{\y_2,\x}{\frac{1}{n}\|\y_2- \x\|^2}$, and the second variance can be computed as
\begin{align*}
    &\V{\x,\y_1,\y_2}{c_{\x,\y_1,\y_2}} =\frac{4}{n^2} \E{\x,\y_1,\y_2}{ \Big((f(\y_1)-\x)^{\top}(\y_2-\x)\Big)^2}
    \\&= \frac{4}{n^2} \E{\x}{\E{\y_1|\x}{(f(\y_1)-\x)(f(\y_1)-\x)^{\top}} \E{\y_2|\x}{ (\y_2-\x)(\y_2-\x)^{\top}} }
    \\ &= \frac{4}{n^2} \E{\x}{\trace{\bPsi_{\x}\bSigma_{\x}}}
\end{align*}
where the first line uses the fact that $\E{\y_1,\y_2|\x}{c_{\x,\y_1,\y_2}} = 0$, and the third line uses the definitions $\bSigma_{\x} = \E{\y_2|\x}{(\y_2-\x)(\y_2-\x)^{\top}}$ and $\bPsi_{\x} = \E{\y_1|\x}{(f(\y_1) - \x)(f(\y_1) - \x)^{\top}}$.

We have that $\Cov{a_{\x,\y_1},b_{\x,\y_2}} = 0$ since $a_{\x,\y_1}$ and $b_{\x,\y_2}$ are independent conditioned on $\x$. We also have that
\begin{align*}
   \Cov{a_{\x,\y_1},c_{\x,\y_1,\y_2}} &= 
   \E{\y_1,\y_2|\x}{ (a_{\x,\y_1} - \E{\y_1|\x}{a_{\x,\y_1}}) c_{\x,\y_1,\y_2}}
\\   &=  \E{\y_1|\x}{ (a_{\x,\y_1} - \E{\y_1|\x}{a_{\x,\y_1}}) \E{\y_2|\x}{c_{\x,\y_1,\y_2}}}
\\   &=0
\end{align*}

The remaining covariance term can be computed as
\begin{align*}
   &\Cov{b_{\x,\y_2},c_{\x,\y_1,\y_2}} 
\\   &=  \E{\y_2|\x}{ \Big(b_{\x,\y_2} - \E{\y_2|\x}{b_{\x,\y_2}}\Big) \E{\y_1|\x}{c_{\x,\y_1,\y_2}}}
\\   &=  \E{\y_2|\x}{b_{\x,\y_2}\E{\y_1|\x}{c_{\x,\y_1,\y_2}}} - \E{\y_2|\x}{b_{\x,\y_2}}\E{\y_1,\y_2|\x}{c_{\x,\y_1,\y_2}}
\\   &=  \E{\y_2|\x}{b_{\x,\y_2}\E{\y_1|\x}{c_{\x,\y_1,\y_2}}}
\\ &= -\frac{2}{n^2}\E{\y_2|\x}{\|\y_2- \x\|^2(\y_2-\x)^{\top}} \Big(\E{\y_1|\x}{f(\y_1)}-\x\Big)
\end{align*}
where the second and third lines use the fact that $\E{\y_1,\y_2|\x}{c_{\x,\y_1,\y_2}} = 0$. 
\end{proof}

\bibliographystyle{IEEEtran}
\bibliography{references}

\end{document}